\documentclass[12pt]{article}
\pdfoutput=1
\usepackage{fullpage}

\usepackage{amsmath}
\usepackage{amsthm}
\usepackage{amsfonts}
\usepackage{amssymb}
\usepackage{graphicx}
\usepackage{hyperref}
\usepackage{natbib}
\usepackage{bbm}
\usepackage{bm}
\usepackage{booktabs}
\usepackage{mathtools}
\usepackage{pagecolor}
\usepackage[colorinlistoftodos]{todonotes}
\usepackage[nottoc]{tocbibind}

\newtheorem{theorem}{Theorem}[section]
\newtheorem{lemma}{Lemma}[section]
\newtheorem{prop}{Proposition}[section]
\newtheorem{defin}{Definition}[section]

\numberwithin{equation}{section}

\newcommand{\footremember}[2]{%
    \footnote{#2}
    \newcounter{#1}
    \setcounter{#1}{\value{footnote}}%
}

\begin{document}

\title{Universal approximations of invariant maps \\ by neural networks}

\author{Dmitry Yarotsky\footremember{skoltech}{Skolkovo Institute of Science and Technology,  Nobelya Ulitsa 3, Moscow  121205,
Russia}\footremember{iitp}{Institute for Information Transmission Problems, Bolshoy Karetny 19 build.1, Moscow 127051, Russia}\\
\texttt{d.yarotsky@skoltech.ru}
}
\date{}
\maketitle

\begin{abstract}
We describe generalizations of the universal approximation theorem for neural networks to maps invariant or equivariant with respect to linear representations of groups. Our goal is to establish network-like computational models that are both invariant/equivariant and provably complete in the sense of their ability to approximate any continuous invariant/equivariant map. Our contribution is three-fold. First, in the general case of compact groups we propose a construction of a complete invariant/equivariant network using an intermediate polynomial layer. We invoke classical theorems of Hilbert and Weyl to justify and simplify this construction; in particular, we describe an explicit complete ansatz for approximation of permutation-invariant maps. Second, we consider groups of translations and prove several versions of the universal approximation theorem for convolutional networks in the limit of continuous signals on euclidean spaces. Finally, we consider 2D signal transformations equivariant with respect to the group SE(2) of rigid euclidean motions. In this case we introduce the ``charge--conserving convnet'' -- a convnet-like computational model based on the decomposition of the feature space into isotypic representations of SO(2). We prove this model to be a universal approximator for continuous SE(2)--equivariant signal transformations.  

\end{abstract}

\tableofcontents

\section{Introduction}
\subsection{Motivation}
An important topic in learning theory is the design of predictive models properly reflecting symmetries naturally present in the data (see, e.g., \cite{burkhardt2001invariant, schulz1995invariant, Reisert:2008}). Most commonly, in the standard context of supervised learning, this means that our predictive model should be \emph{invariant} with respect to a suitable \emph{group of transformations}: given an input object, we often know that its class or some other property that we are predicting does not depend on the object representation (e.g., associated with a particular coordinate system), or for other reasons does not change under certain transformations. In this case we would naturally like the predictive model to reflect this independence. If $f$ is our predictive model and $\Gamma$ the group of transformations, we can express the property of invariance by the identity $f(\mathcal A_\gamma \mathbf x)= f(\mathbf x)$, where $\mathcal A_\gamma \mathbf x$ denotes the action of the transformation $\gamma\in\Gamma$ on the object $\mathbf x$. 

There is also a more general scenario where the output of $f$ is another complex object that is supposed to transform appropriately if the input object is transformed. This scenario is especially relevant in the setting of multi-layered (or stacked) predictive models, if we want to propagate the symmetry through the layers. In this case one speaks about \emph{equivariance}, and  mathematically it is described by the identity $f(\mathcal A_\gamma \mathbf x)= \mathcal A_\gamma f(\mathbf x)$, assuming that the transformation $\gamma$ acts in some way not only on inputs, but also on outputs of $f$. (For brevity, here and in the sequel we will slightly abuse notation and denote any action of $\gamma$ by $\mathcal A_\gamma$, though of course in general the input and output objects are different and $\gamma$ acts differently on them. It will be clear which action is meant in a particular context).  

A well-known important example of equivariant transformations are convolutional layers in neural networks, where the group $\Gamma$ is the group of grid translations, $\mathbb Z^d$.   

We find it convenient to roughly distinguish two conceptually different approaches to the construction of invariant and equivariant models that we refer to as the \emph{symmetrization-based} one and the \emph{intrinsic} one. The symmetrization-based approach consists in starting from some asymmetric model, and symmetrizing it by a group averaging. On the other hand, the intrinsic approach consists in imposing prior structural constraints on the model that guarantee its symmetricity.

In the general mathematical context, the difference between the two approaches is best illustrated with the example of \emph{symmetric polynomials} in the variables $x_1,\ldots,x_n$, i.e., the polynomials invariant with respect to arbitrary permutations of these variables. With the symmetrization-based approach, we can obtain any invariant polynomial by starting with an arbitrary polynomial $f$ and symmetrizing it over the group of permutations $S_n$, i.e. by defining $f_{\mathrm{sym}}(x_1,\ldots,x_n)=\frac{1}{n!}\sum_{\rho\in S_n}f(x_{\rho(1)},\ldots,x_{\rho(n)}).$ On the other hand, the intrinsic approach is associated with the fundamental theorem of symmetric polynomials, which states that any invariant polynomial $f_{\mathrm{sym}}$ in $n$ variables can be obtained as a superposition $f(s_1,\ldots,s_n)$ of some polynomial $f$ and the \emph{elementary symmetric polynomials} $s_1,\ldots,s_n$. Though both approaches yield essentially the same result (an arbitrary symmetric polynomial), the two constructions are clearly very different.    

In practical machine learning, symmetrization is ubiquitous. It is often applied both on the level of data and the level of models. This means that, first, prior to learning an invariant model, one augments the available set of training examples $(\mathbf x, f(\mathbf x))$ by new examples of the form $(\mathcal A_\gamma\mathbf x, f(\mathbf x))$ (see, for example, Section B.2 of \cite{Thoma:2017} for a list of transformations routinely used to augment datasets for image classification problems). Second, once some, generally non-symmetric, predictive model $\widehat f$ has been learned, it is symmetrized by setting $\widehat f_{\rm sym}(\mathbf x)=\frac{1}{|\Gamma_0|}\sum_{\gamma\in\Gamma_0}\widehat f(\mathcal A_\gamma\mathbf x)$, where $\Gamma_0$ is some subset of $\Gamma$ (e.g., randomly sampled). 
This can be seen as a manifestation of the symmetrization-based approach, and its practicality probably stems from the fact that the real world symmetries are usually only approximate, and in this approach one can easily account for their imperfections (e.g., by adjusting the subset $\Gamma_0$). On the other hand, the weight sharing in convolutional networks (\cite{Waibel89, lecun1989generalization}) can be seen as a manifestation of the intrinsic approach (since the translational symmetry is built into the architecture of the network from the outset), and convnets are ubiquitous in modern machine learning \cite{lecun2015deep}.  

In this paper we will be interested in the theoretical opportunities of the intrinsic approach in the context of approximations using neural-network-type models. Suppose, for example, that $f$ is an invariant map that we want to approximate with the usual ansatz of a perceptron with a single hidden layer, $\widehat f(x_1,\ldots,x_d)= \sum_{n=1}^N c_n\sigma(\sum_{k=1}^d w_{nk}x_k+h_n)$ with some nonlinear activation function $\sigma$. Obviously, this ansatz breaks the symmetry, in general. Our goal is to modify this ansatz in such a way that, first, it does not break the symmetry and, second, it is \emph{complete} in the sense that it is not too specialized and any reasonable invariant map can be arbitrarily well approximated by it. In Section \ref{sec:compact} we show how this can be done by introducing an extra polynomial layer into the model. In Sections \ref{sec:translations}, \ref{sec:charge} we will consider more complex, deep models (convnets and their modifications). We will understand completeness in the sense of the universal approximation theorem for neural networks \cite{pinkus1999approximation}.

Designing invariant and equivariant models requires us to decide how the symmetry information is encoded in the layers. A standard assumption, 
 to which we also will adhere in this paper, is that the group acts by linear transformations.  Precisely, when discussing invariant models we are looking for maps of the form
\begin{equation}\label{eqinv}f:V\to \mathbb R,
\end{equation}
where $V$ is a vector space carrying a linear representation $R:\Gamma\to \mathrm{GL}(V)$ of a group $\Gamma$. 
More generally, in the context of multi-layer models
\begin{equation}\label{eqequiv}f:V_1\stackrel{f_1}{\to}V_2\stackrel{f_2}{\to}\ldots 
\end{equation}
we assume that the vector spaces $V_k$ carry linear representations $R_k:\Gamma\to\mathrm{GL}(V_k)$ (the ``baseline architecture'' of the model), and we must then ensure equivariance in each link.
Note that a linear action of a group on the input space $V_1$ is a natural and general phenomenon. In particular, the action is linear if $V_1$ is a linear space of functions on some domain, 
and the action is induced by (not necessarily linear) transformations of the domain. Prescribing linear representations $R_k$ is then a viable strategy to encode and upkeep the symmetry in subsequent layers of the model.  

From the perspective of approximation theory, we will be interested in \emph{finite} computational models, i.e. including finitely many operations as performed on a standard computer. Finiteness is important for potential studies of approximation rates (though such a study is not attempted in the present paper). Compact groups have the nice property that their irreducible linear representations are finite--dimensional. This allows us, in the case of such groups, to modify the standard shallow neural network ansatz so as to obtain a computational model that is finite, fully invariant/equivariant and complete, see Section \ref{sec:compact}. On the other hand, irreducible representations of non-compact groups such as $\mathbb R^\nu$ are infinite-dimensional in general. As a result, finite computational models can be only approximately $\mathbb R^\nu$--invariant/equivariant. Nevertheless, we show in Sections \ref{sec:translations}, \ref{sec:charge} that complete $\mathbb R^\nu$-- and SE($\nu$)--equivariant models can be rigorously described in terms of appropriate limits of finite models.

\subsection{Related work}
Our work can be seen as an extension of results on the universal approximation property of neural networks (\cite{cybenko1989approximation,pinkus1999approximation,leshno1993multilayer,pinkus1996tdi,hornik1993some,funahashi1989approximate,hornik1989multilayer,mhaskar1992approximation}) to the setting of group invariant/equivariant maps and/or infinite-dimensional input spaces. 

Our general results in Section \ref{sec:compact} are based on classical results of the theory of polynomial invariants (\cite{hilbert1890theorie,hilbert1893vollen,weyl1946classical}).

An important element of constructing invariant and equivariant models is the extraction of invariant and equivariant features. In the present paper we do not focus on this topic, but it has been studied extensively, see e.g. general results along with applications to 2D and 3D pattern recognition in \cite{schulz1995invariant, Reisert:2008, burkhardt2001invariant,Skibbe:2013, manay2006integral}.

In a series of works reviewed in \cite{cohen2017analysis}, the authors study expressiveness of deep convolutional networks using hierarchical tensor decompositions and convolutional arithmetic circuits. In particular, representation universality of several network structures is examined in \cite{cohen2016convolutional}.

In a series of works reviewed in \cite{poggio2017and}, the authors study expressiveness of deep networks from the perspective of approximation theory and hierarchical decompositions of functions. Learning of invariant data representations and its relation to information processing in the visual cortex has been discussed in \cite{anselmi2016invariance}.

In the series of papers \cite{mallat2012group,mallat2016understanding,sifre2014rigid,bruna2013invariant}, multiscale wavelet-based group invariant scattering operators and their applications to image recognition have been studied. 

There is a large body of work proposing specific constructions of networks for applied group invariant recognition problems, in particular image recognition approximately invariant with respect to the group of rotations or some of its subgroups: deep symmetry networks of \cite{gens2014deep}, G-CNNs of \cite{cohen2016group}, networks with extra slicing operations in \cite{dieleman2016exploiting}, RotEqNets  of \cite{marcos2016rotation}, networks with warped convolutions in \cite{henriques2016warped}, Polar Transformer Networks of \cite{esteves2017polar}.

\subsection{Contribution of this paper}

As discussed above, we will be interested in the following general question: assuming there is a ``ground truth'' invariant or equivariant map $f$, how can we ``intrinsically'' approximate it by a neural-network-like model?  Our goal is to describe models that are finite, invariant/ equivariant (up to limitations imposed by the finiteness of the model) and provably complete in the sense of approximation theory.  

Our contribution is three-fold:
\begin{itemize}
\item In Section \ref{sec:compact} we consider general compact groups and approximations by shallow networks. Using the classical polynomial invariant theory, we describe a general construction of shallow networks with an extra polynomial layer which are exactly invariant/equivariant and complete (Propositions \ref{th:invar}, \ref{th:equivar}). Then, we discuss how this construction can be improved using the idea of polarization and a theorem of Weyl (Propositions \ref{th:constr}, \ref{th:constr_equiv}). Finally, as a particular illustration of the ``intrinsic'' framework, we consider maps invariant with respect to the symmetric group $S_N$, and describe a corresponding neural network model which is $S_N$--invariant and complete (Theorem \ref{th:sn}). This last result is based on another theorem of Weyl.   
\item In Section \ref{sec:translations} we prove several versions of the universal approximation theorem for convolutional networks and groups of translations. The main novelty of these results is that we approximate maps $f$ defined on the infinite--dimensional space of \emph{continuous signals} on $\mathbb R^\nu$. Specifically, one of these versions (Theorem \ref{th:convmain}) states that a signal transformation $f:L^2(\mathbb R^\nu, \mathbb R^{d_V})\to L^2(\mathbb R^\nu, \mathbb R^{d_U})$ can be approximated, in some natural sense, by convnets without pooling if and only if $f$ is continuous and translationally--equivariant (here, by $L^2(\mathbb R^\nu, \mathbb R^{d})$ we denote the space of square-integrable functions $\mathbf \Phi:\mathbb R^\nu\to\mathbb R^d$).  Another version (Theorem \ref{th:convpool}) states that a map $f:L^2(\mathbb R^\nu, \mathbb R^{d_V})\to \mathbb R$ can be approximated by convnets with pooling if and only if $f$ is continuous.
\item In Section \ref{sec:charge} we describe a convnet-like model which is a universal approximator for signal transformations $f:L^2(\mathbb R^2, \mathbb R^{d_V})\to L^2(\mathbb R^2, \mathbb R^{d_U})$ equivariant with respect to the group SE(2) of rigid two-dimensional euclidean motions. We call this model \emph{charge--conserving convnet}, based on a 2D quantum mechanical analogy (conservation of the total angular momentum). The crucial element of the construction is that the operation of the network is consistent with the decomposition of the feature space into isotypic representations of SO(2). We prove in Theorem \ref{th:charge} that a transformation $f:L^2(\mathbb R^2, \mathbb R^{d_V})\to L^2(\mathbb R^2, \mathbb R^{d_U})$ can be approximated by charge--conserving convnets if and only if $f$ is continuous and SE(2)--equivariant.  
\end{itemize}

\section{Compact groups and shallow approximations}\label{sec:compact}
In this section we give several results on invariant/equivariant approximations by neural networks in the context of compact groups, finite-dimensional representations, and shallow networks. We start by describing the standard group-averaging approach in Section \ref{sec:sym}. In Section \ref{sec:polinveq} we describe an alternative approach, based on the invariant theory. In Section \ref{sec:polar} we show how one can improve this approach using polarization. Finally, in Section \ref{sec:sn} we describe an application of this approach to the symmetric group $S_N$.

\subsection{Approximations based on symmetrization}\label{sec:sym}
We start by recalling the universal approximation theorem, which will serve as a ``template'' for our invariant and equivariant analogs. There are several versions of this theorem (see the survey \cite{pinkus1999approximation}), we will use the general and easy-to-state version given in \cite{pinkus1999approximation}. 

\begin{theorem}[\cite{pinkus1999approximation}, Theorem 3.1]\label{th:leshno} Let $\sigma:\mathbb R\to\mathbb R$ be a continuous activation function  that is not a polynomial. Let $V=\mathbb R^d$ be a real finite dimensional vector space. Then any continuous map $f:V\to\mathbb R$ can be approximated, in the sense of uniform convergence on compact sets, by maps $\widehat f: V\to\mathbb R$ of the form 
\begin{equation}\label{eq:shallownn}
\widehat{f}(x_1,\ldots,x_d)=\sum_{n=1}^Nc_{n}\sigma\Big(\sum_{s=1}^d w_{ns} x_s+h_{n}\Big)
\end{equation}
with some coefficients $c_n,w_{ns},h_n$.
\end{theorem}

Throughout the paper, we assume, as in Theorem \ref{th:leshno}, that \emph{$\sigma:\mathbb R\to\mathbb R$ is some (fixed) continuous activation function  that is not a polynomial.} 

Also, as in this theorem, we will understand approximation in the sense of uniform approximation on compact sets, i.e. meaning that for any compact $K\subset V$ and any $\epsilon>0$ one can find an approximating map $\widehat f$ such that $|f(\mathbf x)-\widehat f(\mathbf x)|\le \epsilon$ (or $\|f(\mathbf x)-\widehat f(\mathbf x)\|\le \epsilon$ in the case of vector-valued $f$) for all $\mathbf x\in K$. In the case of finite-dimensional spaces $V$ considered in the present section, one can equivalently say that there is a sequence of approximating maps $\widehat f_n$ uniformly converging to $f$ on any compact set. Later, in Sections \ref{sec:translations}, \ref{sec:charge}, we will consider infinite-dimensional signal spaces $V$ for which such an equivalence does not hold. Nevertheless, we will use the concept of uniform approximation on compact sets as a guiding principle in our precise definitions of approximation in that more complex setting.

Now suppose that the space $V$ carries a linear representation $R$ of a group $\Gamma$. Assuming $V$ is finite-dimensional, this means that $R$ is a homomorphism of $\Gamma$ to the group of linear automorphisms of $V$:  
$$R:\Gamma\to \text{GL}(V).$$
In the present section we will assume that $\Gamma$ is a \emph{compact} group, meaning, as is customary, that $\Gamma$ is a compact Hausdorff topological space and the group operations (multiplication and inversion) are continuous. Accordingly, the representation $R$ is also assumed to be continuous. We remark that an important special case of compact groups are the finite groups (with respect to the discrete topology). 

One important property of compact groups is the existence of a unique, both left- and right-invariant Haar measure normalized so that the total measure of $\Gamma$ equals 1. Another property is that any continuous representation of a compact group on a separable (but possibly infinite-dimensional) Hilbert space can be decomposed into a countable direct sum of irreducible finite-dimensional representations. There are many group representation textbooks to which we refer the reader for details, see e.g. \cite{vinberg2012linear,serre2012linear,simon1996representations}. Accordingly, in the present section we will restrict ourselves to finite-dimensional representations. Later, in Sections \ref{sec:translations} and \ref{sec:charge}, we will consider the noncompact groups $\mathbb R^\nu$ and SE($\nu$) and their natural representations on the infinite-dimensional space $L^2(\mathbb R^\nu)$, which cannot be decomposed into countably many irreducibles.

Motivated by applications to neural networks, in this section and Section \ref{sec:translations} we will consider only representations over the field $\mathbb R$ of reals (i.e. with $V$ a real vector space). Later, in Section \ref{sec:charge}, we will consider complexified spaces as this simplifies the exposition of the invariant theory for the group SO(2). 

For brevity, we will call a vector space carrying a linear representation of a group $\Gamma$  a \emph{$\Gamma$-module}. We will denote by $R_\gamma$ the  linear automorphism obtained by applying $R$ to $\gamma\in\Gamma$. The integral over the normalized Haar measure on a compact group $\Gamma$ is denoted by $\int_\Gamma\cdot d\gamma$. We will denote vectors by boldface characters; scalar components of the vector $\mathbf x$ are denoted $x_k$.  

Recall that given a $\Gamma$-module $V$, we call a map $f:V\to \mathbb R$ \emph{$\Gamma$-invariant} (or simply \emph{invariant}) if $f(R_\gamma\mathbf x)=f(\mathbf x)$ for all $\gamma\in\Gamma$ and $\mathbf x\in V$. We state now the basic result on invariant approximation, obtained by symmetrization (group averaging). 

\begin{prop}\label{th:syminv} Let $\Gamma$ be a compact group and $V$ a finite-dimensional $\Gamma$-module. Then,   any continuous invariant map $f:V\to\mathbb R$ can be approximated by $\Gamma$-invariant maps $\widehat{f}:V\to\mathbb R$ of the form 
\begin{equation}\label{eq:syminv0}
\widehat{f}(\mathbf x)=\int_\Gamma\sum_{n=1}^N c_n\sigma(l_n(R_\gamma\mathbf x)+h_n)d\gamma,
\end{equation}
where $c_n,h_n\in\mathbb R$ are some coefficients and $l_n\in V^*$ are some linear functionals on $V$, i.e. $l_n(\mathbf x)=\sum_{k}w_{nk}x_k$.
\end{prop}
\begin{proof} It is clear that the map \eqref{eq:syminv0} is $\Gamma$--invariant, and we only need to prove the completeness part. Let $K$ be a compact subset in $V$, and $\epsilon>0$. Consider the symmetrization of $K$ defined by $K_{\mathrm{sym}}=\cup_{\gamma\in \Gamma} R_\gamma(K).$ Note that $K_{\mathrm{sym}}$ is also a compact set, because it is the image of the compact set $\Gamma\times K$ under the continuous map $(\gamma,\mathbf x)\mapsto R_\gamma\mathbf x$. We can use Theorem \ref{th:leshno} to find a map $f_1:V\to \mathbb R$ of the form $f_1(\mathbf x)=\sum_{n=1}^N c_n\sigma(l_n(\mathbf x)+h_n)$ and such that $|f(\mathbf x)-f_1(\mathbf x)|\le \epsilon$ on $K_{\mathrm{sym}}$. Now consider the $\Gamma$-invariant group--averaged map $\widehat f(\mathbf x)=\int_{\Gamma} f_1(R_\gamma\mathbf x)d\gamma$. Then for any $\mathbf x\in K$, 
\begin{equation*}|\widehat f(\mathbf x)-f(\mathbf x)|=
\Big|\int_{\Gamma}\big(f_1(R_\gamma\mathbf x)-f(R_\gamma\mathbf x)\big)d\gamma\Big|\le\int_{\Gamma}\big|f_1(R_\gamma\mathbf x)-f(R_\gamma\mathbf x)\big|d\gamma\le \epsilon,
\end{equation*}
where we have used the invariance of $f$ and the fact that $|f_1(\mathbf x)-f(\mathbf x)|\le\epsilon$ for $\mathbf x\in K_{\mathrm{sym}}$.
\end{proof}

Now we establish a similar result for equivariant maps. Let $V,U$ be two $\Gamma$-modules. For brevity, we will denote by $R$ the representation of $\Gamma$ in either of them (it will be clear from the context which one is meant). We call a map $f:V\to U$ $\Gamma$-\emph{equivariant} if $f(R_\gamma\mathbf x)=R_\gamma f(\mathbf x)$ for all $\gamma\in\Gamma$ and $\mathbf x\in V$. 
\begin{prop}\label{th:sym_equiv} Let $\Gamma$ be a compact group and $V$ and $U$ two finite-dimensional $\Gamma$-modules. Then,   any continuous $\Gamma$-equivariant map $f:V\to U$ can be approximated by $\Gamma$-equivariant maps $\widehat{f}:V\to U$ of the form 
\begin{equation}\label{eq:symequiv}
\widehat{f}(\mathbf x)=\int_\Gamma\sum_{n=1}^N R_{\gamma}^{-1}\mathbf y_n\sigma(l_n(R_\gamma\mathbf x)+h_n)d\gamma,
\end{equation}
with some coefficients $h_n\in\mathbb R$, linear functionals $l_n\in V^*$, and  vectors $\mathbf y_n\in U$.
\end{prop}
\begin{proof} The proof is analogous to the proof of Proposition \ref{th:syminv}. Fix any norm $\|\cdot\|$ in $U$. Given a compact set $K$ and $\epsilon>0$, we construct the compact set $K_{\mathrm{sym}}=\cup_{\gamma\in \Gamma}R_\gamma (K)$ as before. Next,  we find $f_1:V\to U$ of the form $f_1(\mathbf x)=\sum_{n=1}^N \mathbf y_n\sigma(l_n(\mathbf x)+h_n)$ and such that $\|f(\mathbf x)-f_1(\mathbf x)\|\le \epsilon$ on $K_{\mathrm{sym}}$ (we can do it, for example, by considering scalar components of $f$ with respect to some basis in $U$, and approximating these components using Theorem \ref{th:leshno}). Finally, we define the symmetrized map by  $\widehat f(\mathbf x)=\int_{\Gamma}R_{\gamma}^{-1}f_1(R_\gamma \mathbf x) d\gamma$. This map is $\Gamma$--equivariant, and, for any $\mathbf x\in K$, 
\begin{align*}\|\widehat f(\mathbf x)-f(\mathbf x)\|=&
\Big\|\int_{\Gamma}\big(R_\gamma^{-1}f_1(R_\gamma\mathbf x)-R_\gamma^{-1}f(R_\gamma\mathbf x)\big)d\gamma\Big\|\\
\le&\max_{\gamma\in\Gamma}\|R_\gamma\|\int_{\Gamma}\big\|f_1(R_\gamma\mathbf x)-f(R_\gamma\mathbf x)\big\|d\gamma\\
\le& \epsilon\max_{\gamma\in\Gamma}\|R_\gamma\|.
\end{align*}
By continuity of $R$ and compactness of $\Gamma,$ $\max_{\gamma\in\Gamma}\|R_\gamma\|<\infty$, so we can approximate $f$ by $\widehat f$ on $K$ with any accuracy.
\end{proof}

Propositions \ref{th:syminv}, \ref{th:sym_equiv} present the ``symmetrization--based'' approach to constructing invariant/equivariant approximations relying on the shallow neural network ansatz \eqref{eq:shallownn}. The approximating expressions  \eqref{eq:syminv0},  \eqref{eq:symequiv} are $\Gamma$--invariant/equivariant and universal. Moreover, in the case of finite groups the integrals in these expressions  are finite sums, i.e. these approximations consist of finitely many arithmetic operations and evaluations of the activation function $\sigma$. In the case of infinite groups, the integrals can be approximated by sampling the group.

In the remainder of Section \ref{sec:compact} we will pursue an alternative approach to symmetrize the neural network ansatz, based on the theory of polynomial invariants. 

We finish this subsection with the following general observation.
 Suppose that we have two $\Gamma$-modules $U,V$, and $U$ can be decomposed into $\Gamma$--invariant submodules: $U=\bigoplus_\beta U_\beta^{m_\beta}$ (where $m_\beta$ denotes the multiplicity of $U_\beta$ in $U$). Then a map $f:V\to U$ is equivariant if and only if it is equivariant in each component $U_\beta$ of the output space. Moreover, if we denote by $\operatorname{Equiv}(V,U)$ the space of continuous equivariant maps $f:V\to U$, then \begin{equation}\label{eq:decomp}
\operatorname{Equiv}\Big(V,\bigoplus_\beta U_\beta^{m_\beta}\Big)=\bigoplus_\beta\operatorname{Equiv}(V,U_\beta)^{m_\beta}.
\end{equation}
This  shows that the task of describing equivariant maps $f:V\to U$ reduces to the task of describing equivariant maps $f:V\to U_\beta$. In particular, describing vector-valued invariant maps $f:V\to \mathbb R^{d_U}$ reduces to describing scalar-valued invariant maps $f:V\to \mathbb R$.

\subsection{Approximations based on polynomial invariants}\label{sec:polinveq}

The invariant theory seeks to describe \emph{polynomial invariants} of group representations, i.e. polynomial maps $f:V\to \mathbb R$ such that $f(R_\gamma \mathbf x)\equiv f(\mathbf x)$ for all $\mathbf x\in V$. A fundamental result of the invariant theory is Hilbert's finiteness theorem \cite{hilbert1890theorie,hilbert1893vollen} stating that for completely reducible representations, all the polynomial invariants  are algebraically generated by a finite number of such invariants. In particular, this holds for any representation of a compact group. 

\begin{theorem}[Hilbert] Let $\Gamma$ be a compact group and $V$ a finite-dimensional $\Gamma$-module. Then there exist finitely many polynomial invariants $f_1,\ldots,f_{N_{\mathrm{inv}}}:V\to\mathbb R$ such that any polynomial invariant $r:V\to\mathbb R$ can be expressed as $$r(\mathbf x)=\widetilde{r}(f_1(\mathbf x),\ldots,f_{N_{\mathrm{inv}}}(\mathbf x))$$
with some polynomial $\widetilde r$ of ${N_{\mathrm{inv}}}$ variables.
\end{theorem}
See, e.g., \cite{kraft2000classical} for a modern expositions of the invariant theory and Hilbert's theorem. We refer to the set $\{f_s\}_{s=1}^{N_{\mathrm{inv}}}$ from this theorem as a \emph{generating set} of polynomial invariants (note that this set is not unique and ${N_{\mathrm{inv}}}$ may be different for different generating sets). 

Thanks to the density of polynomials in the space of continuous functions, we can easily combine Hilbert's theorem with the universal approximation theorem to obtain a complete invariant ansatz for invariant maps:

\begin{prop}\label{th:invar}
Let $\Gamma$ be a compact group, $V$ a finite-dimensional $\Gamma$-module, and $f_1,\ldots,f_{N_{\mathrm{inv}}}:V\to \mathbb R$ a finite generating set of polynomial invariants on $V$ (existing by Hilbert's theorem). Then, any continuous invariant map $f:V\to\mathbb R$ can be approximated by invariant maps $\widehat{f}:V\to\mathbb R$ of the form 
\begin{equation}\label{eq:polyinvansatz}
\widehat{f}(\mathbf x)=\sum_{n=1}^N c_n\sigma\Big(\sum_{s=1}^{N_{\mathrm{inv}}} w_{ns}f_s(\mathbf x)+h_{n}\Big)
\end{equation}
with some parameter $N$ and coefficients $c_n,w_{ns},h_n$.
\end{prop}
\begin{proof} 
It is obvious that the expressions $\widehat{f}$ are $\Gamma$-invariant, so we only need to prove the completeness part. 

Let us first show that the map $f$ can be approximated by an invariant polynomial. Let $K$ be a compact subset in $V$, and, like before, consider the symmetrized set $K_{\mathrm{sym}}.$  By the Stone-Weierstrass theorem, for any $\epsilon>0$ there exists a polynomial $r$ on $V$ such that $|r(\mathbf x)-f(\mathbf x)|\le\epsilon$ for $\mathbf x\in K_{\mathrm{sym}}$. Consider the symmetrized function $r_{\mathrm{sym}}(\mathbf x)=\int_{\Gamma}r(R_\gamma \mathbf x) d\gamma$. Then the function $r_{\mathrm{sym}}$ is invariant and $|r_{\mathrm{sym}}(\mathbf x)-f(\mathbf x)|\le \epsilon$ for $\mathbf x\in K$. On the other hand, $r_{\mathrm{sym}}$ is a polynomial, since $r(R_\gamma\mathbf x)$ is a fixed degree polynomial in $\mathbf x$ for any $\gamma$.

Using Hilbert's theorem, we express $r_{\mathrm{sym}}(\mathbf x)=\widetilde{r}(f_1(\mathbf x), \ldots, f_{N_{\mathrm{inv}}}(\mathbf x))$ with some polynomial $\widetilde{r}$.

It remains to approximate the polynomial $\widetilde{r}(z_1,\ldots,z_{N_{\mathrm{inv}}})$ by an expression of the form $\widetilde{f}(z_1,\ldots,z_{N_{\mathrm{inv}}})=\sum_{n=1}^Nc_n\sigma(\sum_{s=1}^{N_{\mathrm{inv}}}w_{ns}z_s+h_n)$ on the compact set $\{(f_1(\mathbf x), \ldots, f_{N_{\mathrm{inv}}}(\mathbf x))|\mathbf x\in K\}\subset \mathbb{R}^{N_{\mathrm{inv}}}$. By Theorem \ref{th:leshno}, we can do it with any accuracy $\epsilon$. Setting finally $\widehat{f}(\mathbf x)=\widetilde{f}(f_1(\mathbf x), \ldots, f_{N_{\mathrm{inv}}}(\mathbf x)),$ we obtain $\widehat{f}$ of the required form such that $|\widehat{f}(\mathbf x)-f(\mathbf x)|\le 2\epsilon$ for all $\mathbf x\in K$.
\end{proof}

Note that Proposition \ref{th:invar} is a generalization of Theorem \ref{th:leshno}; the latter is a special case  obtained if the group is trivial ($\Gamma=\{e\}$) or its representation is trivial ($R_\gamma\mathbf x\equiv \mathbf x$), and in this case we can just take $N_{\mathrm{inv}}=d$ and $f_s(\mathbf x) = x_s$.

In terms of neural network architectures, formula \eqref{eq:polyinvansatz} can be viewed as a shallow neural network with an extra polynomial layer that precedes the conventional linear combination and nonlinear activation layers. 

We extend now the obtained result to equivariant maps. Given two $\Gamma$-modules $V$ and $U$, we say that a map $f:V\to U$ is \emph{polynomial} if $l\circ f$ is a polynomial for any linear functional $l:U\to \mathbb R$.  
We rely on the extension of Hilbert's theorem to polynomial equivariants:
\begin{lemma}\label{th:polequivar} Let $\Gamma$ be a compact group and $V$ and $U$ two finite-dimensional $\Gamma$-modules. Then there exist finitely many polynomial invariants $f_1,\ldots,f_{N_{\mathrm{inv}}}:V\to \mathbb R$ and polynomial equivariants $g_1,\ldots,g_{N_{\mathrm{eq}}}:V\to U$ such that any polynomial equivariant $r_{\mathrm{sym}}: V\to U$ can be represented in the form $r_{\mathrm{sym}}(\mathbf x)=\sum_{m=1}^{N_{\mathrm{eq}}} g_m(\mathbf x)\widetilde r_m(f_1(\mathbf x), \ldots, f_{N_{\mathrm{inv}}}(\mathbf x))$ with some polynomials $\widetilde r_m$.
\end{lemma}
\begin{proof} We give a sketch of the proof, see e.g. Section 4 of \cite{worfolk1994zeros} for details. A polynomial equivariant $r_{\mathrm{sym}}: V\to U$ can be viewed as an invariant element of the space $\mathbb R[V]\otimes U$ with the naturally induced action of $\Gamma$, where $\mathbb R[V]$ denotes the space of polynomials on $V$. The space $\mathbb R[V]\otimes U$ is in turn a subspace of the algebra $\mathbb R[V\oplus U^*],$ where $U^*$ denotes the dual of $U$. By Hilbert's theorem, all invariant elements in $\mathbb R[V\oplus U^*]$ can be generated as polynomials of finitely many invariant elements of this algebra. The algebra $\mathbb R[V\oplus U^*]$ is graded by the degree of the  $U^*$ component, and the corresponding decomposition of $\mathbb R[V\oplus U^*]$ into the direct sum of $U^*$-homogeneous spaces indexed by the $U^*$-degree $d_{U^*}=0,1,\ldots,$ is preserved by the group action. The finitely many polynomials generating all invariant polynomials in $\mathbb R[V\oplus U^*]$ can also be assumed to be $U^*$-homogeneous. Let $\{f_s\}_{s=1}^{N_{\mathrm{inv}}}$ be those of these generating polynomials with $d_{U^*}=0$ and $\{g_s\}_{s=1}^{N_{\mathrm{eq}}}$ be those with $d_{U^*}=1.$ Then, a polynomial in the generating invariants is $U^*$-homogeneous with $d_{U^*}=1$ if and only if it is a linear combination of monomials   $g_sf_1^{n_1}f_2^{n_2}\cdots f_{N_{\mathrm{inv}}}^{n_{N_{\mathrm{inv}}}}.$ This yields the representation stated in the lemma. 
\end{proof}
We will refer to the set $\{g_s\}_{s=1}^{N_{\mathrm{eq}}}$ as a generating set of polynomial equivariants.

The equivariant analog of Proposition \ref{th:invar} now reads:
\begin{prop}\label{th:equivar}
Let $\Gamma$ be a compact group, $V$ and $U$ be two finite-dimensional $\Gamma$-modules. Let $f_1,\ldots,f_{N_{\mathrm{inv}}}:V\to \mathbb R$ be a finite generating set of polynomial invariants and  $g_1,\ldots,g_{N_{\mathrm{eq}}}:V\to U$ be a finite generating set of polynomial equivariants (existing by Lemma \ref{th:polequivar}). Then, any continuous equivariant map $f:V\to U$ can be approximated by equivariant maps $\widehat{f}:V\to U$ of the form 
\begin{equation*}
\widehat{f}(\mathbf x)=\sum_{n=1}^N\sum_{m=1}^{N_{\mathrm{eq}}} c_{mn}g_m(\mathbf x) \sigma\Big(\sum_{s=1}^{N_{\mathrm{inv}}} w_{mns}f_s(\mathbf x)+h_{mn}\Big)
\end{equation*}
with some parameter $N$ and coefficients $c_{mn},w_{mns}, h_{mn}.$
\end{prop}

\begin{proof} The proof is similar to the proof of Proposition \ref{th:invar}, with the difference that the polynomial map $r$ is now vector-valued, its symmetrization is defined by $r_{\rm sym}(\mathbf x)=\int_{\Gamma}R_\gamma^{-1}r(R_\gamma \mathbf x) d\gamma$, and Lemma \ref{th:polequivar} is used in place of Hilbert's theorem.
\end{proof}
We remark that, in turn, Proposition \ref{th:equivar} generalizes Proposition \ref{th:invar}; the latter is a special case obtained  when $U=\mathbb R$, and in this case we just take $N_{\mathrm{eq}}=1$ and $g_1= 1.$

\subsection{Polarization and multiplicity reduction}\label{sec:polar}

The main point of Propositions \ref{th:invar} and \ref{th:equivar} is that the representations described there use \emph{finite} generating sets of invariants and equivariants $\{f_s\}_{s=1}^{N_{\mathrm{inv}}}, \{g_m\}_{m=1}^{N_{\mathrm{eq}}}$ independent of the function $f$ being approximated. However, the obvious drawback of these results is their non-constructive nature with regard to the functions $f_s, g_m$. In general, finding generating sets is not easy. Moreover, the sizes $N_{\mathrm{inv}},N_{\mathrm{eq}}$ of these sets in general grow rapidly  with the dimensions of the spaces $V,U$.   

This issue can be somewhat ameliorated using polarization and Weyl's theorem. Suppose that a $\Gamma$--module $V$ admits a decomposition into a direct sum of invariant submodules:
\begin{equation}\label{eq:isotyp}
V=\bigoplus_{\alpha}V_\alpha^{m_\alpha}.
\end{equation}
Here, $V_\alpha^{m_\alpha}$ is a direct sum of $m_\alpha$ submodules isomorphic to $V_\alpha$:  
\begin{equation}\label{eq:isotypcomp} V_\alpha^{m_\alpha}=V_\alpha\otimes \mathbb R^{m_\alpha}=\underbrace{V_\alpha\oplus\ldots\oplus V_{\alpha}}_{m_\alpha}.\end{equation}
Any finite-dimensional representation of a compact group is completely reducible and has a decomposition of the form \eqref{eq:isotyp} with non-isomorphic irreducible submodules $V_\alpha$. In this case the decomposition \eqref{eq:isotyp} is referred to as the \emph{isotypic decomposition}, and the subspaces $V_\alpha^{m_\alpha}$ are known as \emph{isotypic components}. Such isotypic components and their multiplicities $m_\alpha$ are uniquely determined (though individually, the $m_\alpha$ spaces $V_\alpha$ appearing in the direct sum \eqref{eq:isotypcomp} are not uniquely determined, in general, as subspaces in $V$).

For finite groups the number of non-isomorphic irreducibles $\alpha$ is finite. In this case, if the module $V$ is high-dimensional, then this necessarily means that (some of) the multiplicities $m_\alpha$ are large. This is not so, in general, for infinite groups, since infinite compact groups have countably many non-isomorphic irreducible representations. Nevertheless, it is in any case useful to simplify the structure of invariants for high--multiplicity modules, which is what polarization and Weyl's theorem do. 

Below, we slightly abuse the terminology and speak of isotypic components and decompositions in the broader sense, assuming decompositions \eqref{eq:isotyp}, \eqref{eq:isotypcomp} but not requiring the submodules $V_\alpha$ to be irreducible or mutually non-isomorphic. 

The idea of polarization is to generate polynomial invariants of a representation with large multiplicities from invariants of a representation with small multiplicities. Namely, note that in each isotypic component $V_\alpha^{m_\alpha}$ written as $V_\alpha\otimes \mathbb R^{m_\alpha}$ the group essentially acts only on the first factor, $V_\alpha$. So, given two isotypic $\Gamma$-modules of the same type, $V_\alpha^{m_\alpha}=V_\alpha\otimes \mathbb R^{m_\alpha}$ and $V_\alpha^{m'_\alpha}=V_\alpha\otimes \mathbb R^{m'_\alpha}$, the group action commutes with any linear map $\mathbbm{1}_{V_\alpha}\otimes A:V_\alpha^{m_\alpha}\to V_\alpha^{m'_\alpha}$, where $A$ acts on the second factor, $A:\mathbb R^{m_\alpha}\to \mathbb R^{m'_\alpha}$. 
Consequently, given two modules $V=\bigoplus_\alpha V_\alpha^{m_\alpha}$, $V'=\bigoplus_\alpha V_\alpha^{m'_\alpha}$ and a linear map $A_\alpha:V_\alpha^{m_\alpha}\to V_\alpha^{m'_\alpha}$ for each $\alpha$, the linear operator $\mathbf A:V\to V'$ defined by 
\begin{equation}\label{a}
\mathbf A=\bigoplus_\alpha \mathbbm{1}_{V_\alpha}\otimes A_\alpha\end{equation} will commute with the group action. In particular, if $f$ is a polynomial invariant on $V'$, then $f\circ \mathbf A$ will be a polynomial invariant on $V$.    

The fundamental theorem of Weyl states that it suffices to take $m'_\alpha=\dim V_\alpha$ to generate in this way a complete set of invariants for $V$. We will state this theorem in the following form suitable for our purposes.
\begin{theorem}[\cite{weyl1946classical}, sections II.4-5]\label{th:weyl} Let $F$ be the set of polynomial invariants for a $\Gamma$-module $V'=\bigoplus_\alpha V_\alpha^{\dim V_\alpha}.$ Suppose that a $\Gamma$-module $V$ admits a decomposition $V=\bigoplus_\alpha V_\alpha^{m_\alpha}$ with the same $V_\alpha$, but arbitrary multiplicities $m_\alpha$. Then the polynomials $\{f\circ \mathbf A\}_{f\in F}$ linearly span the space of polynomial invariants on $V$, i.e. any polynomial invariant $f$ on $V$ can be expressed as $f(\mathbf x) = \sum_{t=1}^T f_{t}(\mathbf A_t\mathbf x)$ with some polynomial invariants $f_t$ on $V'$.
\end{theorem}
\begin{proof} 
A detailed exposition of polarization and a proof of Weyl's theorem based on the Capelli--Deruyts expansion can be found in Weyl's book or in Sections 7--9 of \cite{kraft2000classical}. We sketch the main idea of the proof. 

Consider first the case where $V$ has only one isotypic component: $V=V_\alpha^{m_\alpha}$. We may assume without loss of generality that $m_\alpha>\dim V_\alpha$ (otherwise the statement is trivial). It is also convenient to identify the space $V'=V_\alpha^{\dim V_\alpha}$ with the subspace of $V$ spanned by the first $\dim V_\alpha$ components $V_\alpha$.  
It suffices to establish the claimed expansion for polynomials $f$ multihomogeneous with respect to the decomposition $V=V_\alpha\oplus\ldots \oplus V_\alpha$, i.e. homogeneous with respect to each of the $m_\alpha$ components. For any such polynomial, the Capelli--Deruyts expansion represents $f$ as a finite sum $f=\sum_{n}C_nB_nf$. Here $C_n,B_n$ are linear operators on the space of polynomials on $V$, and they belong to the algebra generated by polarization operators on $V$. Moreover, for each $n$, the polynomial $\widetilde f_n=B_nf$ depends only on variables from the first $\dim V_\alpha$ components of $V=V_\alpha^{m_\alpha}$, i.e. $\widetilde f_n$ is a polynomial on $V'$. This polynomial is invariant, since polarization operators commute with the group action. Since $C_n$ belongs to the algebra generated by polarization operators, we can then argue (see Proposition 7.4 in \cite{kraft2000classical}) that $C_nB_nf$ can be represented as a finite sum $C_nB_nf(\mathbf x)=\sum_{k}\widetilde f_n((\mathbbm{1}_{V_\alpha}\otimes A_{kn})\mathbf x)$ with some $m_\alpha\times \dim V_\alpha$ matrices $A_{kn}$. This implies the claim of the theorem in the case of a single isotypic component.

Generalization to several isotypic components is obtained by iteratively applying the Capelli--Deruyts expansion to each component. 
\end{proof}
Now we can give a more constructive version of Proposition \ref{th:invar}:

\begin{prop}\label{th:constr} Let $(f_s)_{s=1}^{N_{\mathrm{inv}}}$ be a generating set of polynomial invariants for a $\Gamma$-module $V'=\bigoplus_\alpha V_\alpha^{\dim V_\alpha}.$ Suppose that a $\Gamma$-module $V$ admits a decomposition $V=\bigoplus_\alpha V_\alpha^{m_\alpha}$ with the same $V_\alpha$, but arbitrary multiplicities $m_\alpha$. Then any continuous invariant map $f:V\to\mathbb R$ can be approximated by invariant maps $\widehat{f}:V\to\mathbb R$ of the form 
\begin{equation}\label{constr_ansatz1}
\widehat{f}(\mathbf x)=\sum_{t=1}^T c_{t}\sigma\Big(\sum_{s=1}^{N_{\mathrm{inv}}} w_{st}f_{s}(\mathbf A_{t}\mathbf x)+h_{t}\Big)
\end{equation}
with some parameter $T$ and coefficients $c_{t},w_{st},h_{t},\mathbf A_{t},$ where each $\mathbf A_{t}$ is formed by an arbitrary collection of $(m_\alpha\times \dim V_\alpha)$-matrices $A_\alpha$ as in \eqref{a}.
\end{prop}
\begin{proof}
We follow the proof of Proposition \ref{th:invar} and approximate the function $f$ by an invariant polynomial $r_{\mathrm{sym}}$ on a compact set $K_{\mathrm{sym}}\subset V$. Then, using Theorem \ref{th:weyl}, we represent 
\begin{equation}\label{eq:rmsym}
r_{\mathrm{sym}}(\mathbf x)=\sum_{t=1}^T r_t(\mathbf A_t\mathbf x)
\end{equation}
with some invariant polynomials $r_t$ on $V'$. Then, by Proposition \ref{th:invar}, for each $t$ we can approximate $r_t(\mathbf y)$ on $\mathbf A_tK_{\mathrm{sym}}$ by an expression 
\begin{equation}\label{eq:ncnt}\sum_{n=1}^N\widetilde{c}_{nt}\sigma\Big(\sum_{s=1}^{N_{\mathrm{inv}}} \widetilde{w}_{nst}f_{s}(\mathbf y)+\widetilde{h}_{nt}\Big)\end{equation} with some $\widetilde{c}_{nt},\widetilde{w}_{nst},\widetilde{h}_{nt}.$ Combining \eqref{eq:rmsym} with \eqref{eq:ncnt}, it follows that $f$ can be approximated on $K_{\mathrm{sym}}$ by  $$\sum_{t=1}^T\sum_{n=1}^N\widetilde{c}_{nt}\sigma\Big(\sum_{s=1}^{N_{\mathrm{inv}}} \widetilde{w}_{nst}f_{s}(\mathbf A_{t}\mathbf x)+\widetilde{h}_{nt}\Big).$$ The final expression \eqref{constr_ansatz1} is obtained now by removing the superfluous summation over $n$. 
\end{proof}

Proposition \ref{th:constr} is more constructive than Proposition \ref{th:invar} in the sense that the approximating ansatz \eqref{constr_ansatz1} only requires us to know an isotypic decomposition $V=\bigoplus_\alpha V_\alpha^{m_\alpha}$ of the $\Gamma$-module under consideration and a generating set $(f_s)_{s=1}^{N_{\mathrm{inv}}}$ for the reference module $V'=\bigoplus_\alpha V_\alpha^{\dim V_\alpha}$. In particular, suppose that the group $\Gamma$ is finite, so that there are only finitely many non-isomorphic irreducible modules $V_\alpha$. Then, for any $\Gamma$--module $V$, the universal approximating ansatz \eqref{constr_ansatz1} includes not more than $CT\dim V$ scalar weights, with some constant $C$ depending only on $\Gamma$ (since $\dim V=\sum_{\alpha}m_\alpha\dim V_\alpha$). 

We remark that in terms of the network architecture, formula \eqref{constr_ansatz1} can be interpreted as the network \eqref{eq:polyinvansatz} from Proposition \ref{th:invar} with an extra linear layer performing multiplication of the input vector by $\mathbf A_t$.

We establish now an equivariant analog of Proposition \ref{th:constr}. We start with an equivariant analog of Theorem \ref{th:weyl}.  
\begin{prop}\label{th:weyl_equiv} Let $V'=\bigoplus_\alpha V_\alpha^{\dim V_\alpha}$ and $G$ be the space of polynomial equivariants $g:V'\to U$.  Suppose that a $\Gamma$-module $V$ admits a decomposition $V=\bigoplus_\alpha V_\alpha^{m_\alpha}$ with the same $V_\alpha$, but arbitrary multiplicities $m_\alpha$. Then, the functions $\{g\circ \mathbf A\}_{g\in G}$ linearly span the space of polynomial equivariants $g:V\to U$, i.e. any such equivariant can be expressed as $g(\mathbf x) = \sum_{t=1}^T g_{t}(\mathbf A_t\mathbf x)$ with some polynomial equivariants $g_t:V'\to U$.
\end{prop}
\begin{proof} As mentioned in the proof of Lemma \ref{th:polequivar}, polynomial equivariants $g:V\to U$ can be viewed as invariant elements of the extended polynomial algebra $\mathbb R[V\oplus U^*]$. The proof of the theorem is then completely analogous to the proof of Theorem \ref{th:weyl} and consists in applying the Capelli--Deruyts expansion to each isotypic component of the submodule $V$ in $V\oplus U^*$.
\end{proof}

The equivariant analog of Proposition \ref{th:constr} now reads:
\begin{prop}\label{th:constr_equiv} Let $(f_s)_{s=1}^{N_{\mathrm{inv}}}$ be a generating set of polynomial invariants for a $\Gamma$-module $V'=\bigoplus_\alpha V_\alpha^{\dim V_\alpha}$, and  $(g_s)_{s=1}^{N_\mathrm{eq}}$ be a generating sets of polynomial equivariants mapping $V'$ to a $\Gamma$-module $U$. Let  $V=\bigoplus_\alpha V_\alpha^{m_\alpha}$ be a $\Gamma$-module with the same $V_\alpha$. Then any continuous equivariant map $f:V\to U$ can be approximated by equivariant maps $\widehat{f}:V\to U$ of the form 
\begin{equation}\label{eq:constr_ansatz_equiv}
\widehat{f}(\mathbf x)=\sum_{t=1}^T\sum_{m=1}^{N_\mathrm{eq}} c_{mt}g_m(\mathbf A_{t} \mathbf x)\sigma\Big(\sum_{s=1}^{N_{\mathrm{inv}}} w_{mst}f_{s}(\mathbf A_{t}\mathbf x)+h_{mt}\Big)
\end{equation}
with some coefficients $c_{mt},w_{mst},h_{mt},\mathbf A_{t},$ where each $\mathbf A_{t}$ is given by a collection of $(m_\alpha\times \dim V_\alpha)$-matrices $A_\alpha$ as in \eqref{a}. 
\end{prop}
\begin{proof} 
As in the proof of Theorem \ref{th:equivar}, we approximate the function $f$ by a polynomial equivariant $r_{\mathrm{sym}}$ on a compact $K_{\mathrm{sym}}\subset V$. Then, using Theorem \ref{th:weyl_equiv}, we represent 
\begin{equation}\label{eq:rsym1}r_{\mathrm{sym}}(\mathbf x)=\sum_{t=1}^T r_t(\mathbf A_t\mathbf x)
\end{equation}
with some polynomial equivariants $r_t:V'\to U$. Then, by Proposition \ref{th:equivar}, for each $t$ we can approximate 
$r_t(\mathbf x')$ on $\mathbf A_tK_{\mathrm{sym}}$ by expressions \begin{equation}\label{eq:rtx} \sum_{n=1}^N\sum_{m=1}^{N_\mathrm{eq}}\widetilde c_{mnt}g(\mathbf x')\sigma\Big(\sum_{s=1}^{N_\mathrm{inv}} \widetilde w_{mnst}f_{s}(\mathbf x')+\widetilde h_{mnt}\Big).\end{equation}
Using \eqref{eq:rsym1} and \eqref{eq:rtx}, $f$ can be approximated on $K_{\mathrm{sym}}$ by expressions 
\begin{equation*}\sum_{t=1}^T \sum_{n=1}^N\sum_{m=1}^{N_\mathrm{eq}}\widetilde c_{mnt}g(\mathbf A_{t}\mathbf x)\sigma\Big(\sum_{s=1}^{N_{\mathrm{inv}}} \widetilde w_{mnst}f_{s}(\mathbf A_{t}\mathbf x)+\widetilde h_{mnt}\Big).\end{equation*}
We obtain the final form \eqref{eq:constr_ansatz_equiv} by removing the superfluous summation over $n$.
\end{proof}
We remark that Proposition \ref{th:constr_equiv} improves the earlier Proposition \ref{th:equivar} in the equivariant setting in the same sense in which  
Proposition \ref{th:constr} improves Proposition \ref{th:invar} in the invariant setting: construction of a universal approximator in the case of arbitrary isotypic multiplicities is reduced to the construction with particular multiplicities by adding an extra equivariant linear layer to the network.

\subsection{The symmetric group $S_N$}\label{sec:sn}
Even with the simplification resulting from polarization, the general results of the previous section are not immediately useful, since one still needs to find the isotypic decomposition of the analyzed $\Gamma$-modules and to find the relevant generating invariants and equivariants. In this section we describe one particular case where the approximating expression can be reduced to a fully explicit form.

Namely, consider the natural action of the symmetric group $S_N$ on $\mathbb R^N$:
$$R_\gamma \mathbf e_n=\mathbf e_{\gamma(n)},$$
where $\mathbf e_n\in\mathbb R^N$ is a coordinate vector and $\gamma\in S_N$ is a permutation.

Let $V=\mathbb R^N\otimes \mathbb R^M$ and consider $V$ as a $S_N$-module by assuming that the group acts on the first factor, i.e.  $\gamma$ acts on $\mathbf x=\sum_{n=1}^N \mathbf e_n\otimes \mathbf x_n\in V$ by $$R_\gamma \sum_{n=1}^N\mathbf e_n\otimes \mathbf x_n=\sum_{n=1}^N\mathbf e_{\gamma(n)}\otimes \mathbf x_n.$$

We remark that this module appears, for example, in the following scenario. Suppose that $f$ is a map defined on the set of \emph{sets} $X=\{\mathbf x_1,\ldots,\mathbf x_N\}$ of $N$ vectors from $\mathbb R^M$. We can identify the set $X$ with the element $\sum_{n=1}^N \mathbf e_n\otimes \mathbf x_n$ of $V$ and in this way view $f$ as defined on a subset of $V$. However, since the set $X$ is unordered, it can also be identified with $\sum_{n=1}^N \mathbf e_{\gamma(n)}\otimes \mathbf x_n$ for any permutation $\gamma\in \mathcal S_N$. Accordingly, if the map $f$ is to be extended to the whole $V$, then this extension needs to be invariant with respect to the above action of $S_N$.    

We describe now an explicit complete ansatz for $S_N$-invariant approximations of functions on $V$. This is made possible by another classical theorem of Weyl and by a simple form of a generating set of permutation invariants on $\mathbb R^N$. We will denote by $x_{nm}$ the coordinates of $\mathbf x\in V$ with respect to the canonical basis  in $V$:
$$\mathbf x = \sum_{n=1}^N\sum_{m=1}^M x_{nm}\mathbf e_n\otimes\mathbf e_m.$$

\begin{theorem}\label{th:sn}
Let $V=\mathbb R^N\otimes \mathbb R^M$ and $f:V\to \mathbb R$ be a $S_N$-invariant continuous map. Then $f$ can be approximated by $S_N$-invariant expressions 
\begin{equation}\label{eq:ansatzsn}
\widehat f(\mathbf x)=\sum_{t=1}^{T_1}c_t\sigma \bigg(\sum_{q=1}^{T_2}w_{qt}\sum_{n=1}^N\sigma\Big(b_q\sum_{m=1}^{M}a_{tm}x_{nm}+e_{q}\Big)+{h}_{t}\bigg),
\end{equation}
with some parameters $T_1,T_2$ and coefficients $c_t,w_{qt},b_q,a_{tm},e_{q},{h}_{t}$.
\end{theorem}
\begin{proof}
It is clear that expression \eqref{eq:ansatzsn} is $S_N$-invariant and we only need to prove its completeness. The theorem of Weyl (\cite{weyl1946classical}, Section II.3) states that a generating set of symmetric polynomials on $V$ can be obtained by polarizing a generating set of symmetric polynomials  $\{f_p\}_{p=1}^{N_\mathrm{inv}}$ defined on a \emph{single} copy of $\mathbb R^N$. Arguing as in Proposition \ref{th:constr}, it follows that any $S_N$-invariant continuous map $f:V\to\mathbb R$ can be approximated by expressions $$\sum_{t=1}^{T_1}\widetilde c_t\sigma \Big(\sum_{p=1}^{N_\mathrm{inv}} \widetilde w_{pt} f_p(\widetilde{\mathbf A}_t\mathbf x)+\widetilde h_t\Big),$$
where $\widetilde{\mathbf A}_t\mathbf x=\sum_{n=1}^N\sum_{m=1}^M \widetilde a_{tm}x_{nm}\mathbf e_n.$ A well-known generating set of symmetric polynomials on $\mathbb R^N$ is  the first $N$ coordinate power sums:
$$f_p(\mathbf y)=\sum_{n=1}^N \widetilde f_p(y_n), \text{ where } \mathbf y = (y_1,\ldots,y_N),\quad \widetilde f_p(y_n) = y^p_n, \quad p=1,\ldots,N.$$
It follows that $f$ can be approximated by expressions
\begin{equation}\label{eq:snans2}\sum_{t=1}^{T_1}\widetilde c_t\sigma \bigg(\sum_{p=1}^N \widetilde w_{pt}\sum_{n=1}^N \widetilde f_p\Big(\sum_{m=1}^M \widetilde a_{tm}x_{nm}\Big)+\widetilde h_t\bigg).\end{equation}
Using Theorem \ref{th:leshno}, we can approximate $\widetilde f_p(y)$ by expressions $\sum_{q=1}^{T}\widetilde d_{pq}\sigma(\widetilde b_{pq}y+\widetilde h_{pq})$. It follows that \eqref{eq:snans2} can be approximated by
$$\sum_{t=1}^{T_1}\widetilde c_t\sigma \bigg(\sum_{p=1}^N\sum_{q=1}^{T} \widetilde w_{pt}\widetilde{d}_{pq}\sum_{n=1}^N \sigma\Big(\widetilde{b}_{pq}\sum_{m=1}^M \widetilde a_{tm}x_{nm}+h_{pq}\Big)+\widetilde h_t\bigg).$$
Replacing the double summation over $p,q$ by a single summation over $q$, we arrive at \eqref{eq:ansatzsn}. 
\end{proof}
Note that expression \eqref{eq:ansatzsn} resembles the formula of the usual (non-invariant) feedforward network with two hidden layers of sizes $T_1$ and $T_2$:
\begin{equation*}
\widehat f(\mathbf x)=\sum_{t=1}^{T_1}c_t\sigma \bigg(\sum_{q=1}^{T_2}w_{qt}\sigma\Big(\sum_{n=1}^N\sum_{m=1}^{M}a_{qnm}x_{nm}+e_{q}\Big)+{h}_{t}\bigg).
\end{equation*}

Let us also compare ansatz \eqref{eq:ansatzsn} with the ansatz obtained by direct symmetrization (see Proposition \eqref{th:syminv}), which in our case has the form $$\widehat f(\mathbf x)=\sum_{\gamma\in S_N}\sum_{t=1}^T c_t\sigma\Big(\sum_{n=1}^N\sum_{m=1}^Mw_{\gamma(n),m,t}x_{nm}+h_t\Big).$$
From the application perspective, since $|S_N|=N!,$ at large $N$ this expression has prohibitively many terms and is therefore impractical without subsampling of $S_N$, which would break the exact $S_N$-invariance. In contrast, ansatz  \eqref{eq:ansatzsn} is complete, fully $S_N$-invariant and  involves only $O(T_1N(M+T_2))$ arithmetic operations and evaluations of $\sigma$.

\section{Translations and deep convolutional networks}\label{sec:translations}
Convolutional neural networks (convnets, \cite{lecun1989generalization}) play a key role in many modern applications of deep learning. Such networks operate on input data having grid-like structure (usually, spatial or temporal) and consist of multiple stacked convolutional layers transforming initial object description into increasingly complex features necessary to recognize complex patterns in the data. The shape of earlier layers in the network mimics the shape of input data, but later layers gradually become ``thinner'' geometrically while acquiring ``thicker'' \emph{feature dimensions}. We refer the reader to deep learning literature for details on these networks, e.g. see Chapter 9 in \cite{Goodfellow-et-al-2016} for an introduction.

There are several important concepts associated with convolutional networks, in particular \emph{weight sharing} (which  ensures approximate \emph{translation equivariance} of the layers with respect to grid shifts); \emph{locality} of the layer operation; and \emph{pooling}.  Locality means that the layer output at a certain geometric point of the domain depends only on a small neighborhood of this point. Pooling is a grid subsampling that helps reshape the data flow by removing excessive spatial detalization. Practical usefulness of convnets stems from the interplay between these various elements of convnet design. 

From the perspective of the main topic of the present work -- group invariant/equivariant networks -- we are mostly interested in invariance/equivariance of convnets with respect to Lie groups such as the group of translations or the group of rigid motions (to be considered in Section \ref{sec:charge}), and we would like to establish relevant universal approximation theorems. However, we first point out some serious difficulties that one faces when trying to formulate and prove such results. 

\textbf{Lack of symmetry in finite computational models.} Practically used convnets are finite models; in particular they operate on discretized and bounded domains that do not possess the full symmetry of the spaces $\mathbb R^d$. While the translational symmetry is partially preserved by discretization to a regular grid, and the group $\mathbb R^d$ can be in a sense approximated by the groups $(\lambda\mathbb Z)^d$ or $(\lambda\mathbb Z_n)^d$, one cannot reconstruct, for example, the rotational symmetry in a similar way. If a group $\Gamma$ is compact, then, as discussed in Section \ref{sec:compact}, we can still obtain finite and fully $\Gamma$-invariant/equivariant computational models by considering finite-dimensional representations of $\Gamma$, but this is not the case with noncompact groups such as $\mathbb R^d$. Therefore, in the case of the group $\mathbb R^d $ (and the group of rigid planar motions considered later in Section \ref{sec:charge}), we will need to prove the desired results on invariance/eqiuvariance and completeness of convnets only in the limit of infinitely large domain and infinitesimal grid spacing.

\textbf{Erosion of translation equivariance by pooling.} Pooling reduces the translational symmetry of the convnet model. For example, if a few first layers of the network define a map equivariant with respect to the group $(\lambda\mathbb Z)^2$ with some spacing $\lambda$, then after pooling with stride $m$ the result will only be equivariant with respect to the subgroup $(m\lambda\mathbb Z)^2$. (We remark in this regard that in practical applications, weight sharing and accordingly translation equivariance are usually only important for earlier layers of convolutional networks.) Therefore, we will consider separately the cases of convnets without or with pooling; the $\mathbb R^d$--equivariance will only apply in the former case.

In view of the above difficulties, in this section we will give several versions of the universal approximation theorem for convnets, with different treatments of these issues.  

In Section \ref{sec:convnet_one_layer} we prove a universal approximation theorem for a single non-local convolutional layer on a finite discrete grid with periodic boundary conditions (Proposition \ref{th:convnetzn}). This basic result is a straightforward consequence of the general Proposition \ref{th:sym_equiv} when applied to finite abelian groups.

In Section \ref{sec:convnet_deep} we prove the main result of Section \ref{sec:translations}, Theorem \ref{th:convmain}. This theorem extends Proposition \ref{th:convnetzn} in several important ways. First, we will consider continuum signals, i.e. assume that the approximated map is defined on functions on $\mathbb R^n$ rather than on functions on a discrete grid. This extension will later allow us to rigorously formulate a universal approximation theorem for rotations and euclidean motions in Section \ref{sec:charge}. Second, we will consider stacked convolutional layers and assume each layer to act locally (as in convnets actually used in applications). However, the setting of Theorem \ref{sec:convnet_deep} will not involve pooling, since, as remarked above, pooling destroys the translation equivariance of the model.

In Section \ref{sec:convpool} we prove Theorem \ref{th:convpool},  relevant for convnets most commonly used in practice. Compared to the setting of Section \ref{sec:convnet_deep}, this computational model will be spatially bounded, will include pooling, and will not assume translation invariance of the approximated map.

\subsection{Finite abelian groups and single convolutional layers}\label{sec:convnet_one_layer}
We consider a group 
\begin{equation}\label{eq:abe}
\Gamma=\mathbb Z_{n_1}\times\cdots\times \mathbb Z_{n_\nu},
\end{equation}
where $\mathbb Z_n=\mathbb Z/(n\mathbb Z)$ is the cyclic group of order $n$. Note that the group $\Gamma$ is abelian and conversely, by the fundamental theorem of finite abelian groups, any such group can be represented in the form \eqref{eq:abe}.

We consider the ``input'' module $V=\mathbb R^{\Gamma}\otimes \mathbb R^{d_V}$ and the ``output'' module $U=\mathbb R^{\Gamma}\otimes \mathbb R^{d_U}$, with some finite dimensions $d_V, d_U$ and with the natural representation of $\Gamma$:
\begin{equation*}
R_\gamma (\mathbf e_{\theta}\otimes \mathbf v)=\mathbf e_{\theta+\gamma}\otimes \mathbf v,\quad \gamma,\theta\in \Gamma, \quad \mathbf v\in \mathbb R^{d_V}\text{ or }\mathbb R^{d_U}.
\end{equation*}

We will denote elements of $V,U$ by boldface characters $\mathbf \Phi$ and interpret them as $d_V$- or $d_U$-component \emph{signals} defined on the set $\Gamma$. For example, in the context of 2D image processing we have $\nu=2$ and the group $\Gamma=\mathbb Z_{n_1}\times\mathbb Z_{n_2}$ corresponds to a discretized rectangular image with periodic boundary conditions, where $n_1,n_2$ are the geometric sizes of the image while $d_V$ and $d_U$ are the numbers of input and output features, respectively (in particular, if the input is a usual RGB image, then $d_V=3$).  

Denote by $\Phi_{\theta k}$ the coefficients in the expansion of a vector $\mathbf \Phi$ from $V$ or $U$ over the standard product bases in these spaces:
\begin{equation}\label{eq:Phi}
\mathbf \Phi=\sum_{\theta\in \Gamma}\sum_{k=1}^{d_V\text{ or }d_U}\Phi_{\theta  k}\mathbf e_{\theta}\otimes \mathbf e_{k}.
\end{equation}

We describe now a complete equivariant ansatz for approximating $\Gamma$-equivariant maps $f:V\to U$. Thanks to decomposition \eqref{eq:decomp}, we may assume without loss that $d_U=1$. By \eqref{eq:Phi}, any map $f:V\to U$ is then specified by the coefficients $f(\mathbf \Phi)_{\theta}(\equiv f(\mathbf \Phi)_{\theta,1})\in \mathbb R$ as $\mathbf \Phi$ runs over $V$ and $\theta$ runs over $\Gamma$.

\begin{prop}\label{th:convnetzn} Any continuous $\Gamma$-equivariant map $f:V\to U$ can be approximated by $\Gamma$-equivariant maps $\widehat{f}:V\to U$ of the form 
\begin{equation}\label{eq:syminv}
\widehat{f}(\mathbf \Phi)_{\gamma}=
\sum_{n=1}^N c_n\sigma\Big(\sum_{\theta\in\Gamma}\sum_{k=1}^{d_{V}}w_{n\theta k}\Phi_{\gamma+\theta,k}+h_n\Big),
\end{equation}
where $\mathbf \Phi =\sum_{\gamma\in\Gamma}\sum_{k=1}^{d_V}\Phi_{\gamma k}\mathbf e_{\gamma}\otimes \mathbf e_{k}$, $N$ is a parameter, and $c_n,w_{n\gamma k}, h_n$ are some coefficients.
\end{prop}
\begin{proof} We apply Proposition \ref{th:sym_equiv} with $l_n(\mathbf\Phi)=\sum_{\theta'\in\Gamma}\sum_{k=1}^{d_V}w'_{n\theta' k}\Phi_{\theta' k}$ and $\mathbf y_n=\sum_{\varkappa\in\Gamma}y_{n\varkappa}\mathbf e_{\varkappa}$, and obtain the ansatz 
\begin{equation*}\widehat{f}(\mathbf \Phi) = \sum_{\gamma'\in\Gamma}\sum_{n=1}^N \sum_{\varkappa\in\Gamma} y_{n\varkappa}\sigma\Big(\sum_{\theta'\in\Gamma}\sum_{k=1}^{d_{V}}w'_{n\theta' k}\Phi_{\theta'-\gamma',k}+h_n\Big)\mathbf e_{\varkappa-\gamma'}
 =  \sum_{\varkappa\in\Gamma}\sum_{n=1}^N  y_{n\varkappa}\mathbf  a_{\varkappa n}, 
\end{equation*}
where 
\begin{equation}\label{eq:akappan}
\mathbf a_{\varkappa n} = \sum_{\gamma'\in \Gamma}\sigma\Big(\sum_{\theta'\in\Gamma}\sum_{k=1}^{d_{V}}w'_{n\theta' k}\Phi_{\theta'-\gamma',k}+h_n\Big)\mathbf e_{\varkappa-\gamma'}.
\end{equation}
By linearity of the expression on the r.h.s. of \eqref{eq:syminv}, it suffices to check that each $\mathbf  a_{\varkappa n}$ can be written in the form
\begin{equation*}
\sum_{\gamma\in \Gamma}\sigma\Big(\sum_{\theta\in\Gamma}\sum_{k=1}^{d_{V}}w_{n\theta k}\Phi_{\theta+\gamma,k}+h_n\Big)\mathbf e_{\gamma}.
\end{equation*}
But this expression results if we make in \eqref{eq:akappan} the substitutions $\gamma=\varkappa-\gamma', \theta=\theta'-\varkappa$ and $w_{n\theta k}=w'_{n,\theta+\varkappa,k}.$ 
\end{proof}
The expression \eqref{eq:syminv} resembles the standard convolutional layer without pooling as described, e.g., in \cite{Goodfellow-et-al-2016}. Specifically, this expression can be viewed as a linear combination of $N$ scalar filters obtained as  compositions of linear convolutions with pointwise non-linear activations. An important difference with the standard convolutional layers is that the convolutions in \eqref{eq:syminv} are non-local, in the sense that the weights $w_{n\theta k}$ do not vanish at large $\theta$. Clearly, this non-locality is inevitable if approximation is to be performed with just a single convolutional layer.

We remark that it is possible to use Proposition \ref{th:equivar} to describe an alternative complete $\Gamma$-equivariant ansatz based on polynomial invariants and equivariants. However, this approach seems to be less efficient because it is relatively difficult to specify a small explicit set of generating polynomials for abelian groups (see, e.g. \cite{schmid1991finite} for a number of relevant results). Nevertheless, we will use polynomial invariants of the abelian group SO(2) in our construction of ``charge-conserving convnet'' in Section \ref{sec:charge}.   

\subsection{Continuum signals and deep convnets}\label{sec:convnet_deep}
In this section we extend Proposition \ref{th:convnetzn} in several ways. 

First, instead of the group $\mathbb Z_{n_1}\times\cdots\times \mathbb Z_{n_\nu}$ we  consider the group $\Gamma=\mathbb R^\nu$. Accordingly, we will consider infinite-dimensional 
$\mathbb R^\nu$--modules 
\begin{align*}V&=L^2(\mathbb R^\nu)\otimes \mathbb R^{d_V}\cong L^2(\mathbb R^\nu,\mathbb R^{d_V}),\\
U&=L^2(\mathbb R^\nu)\otimes \mathbb R^{d_U}\cong L^2(\mathbb R^\nu,\mathbb R^{d_U})
\end{align*} 
with some finite $d_V, d_U$. Here, $L^2(\mathbb R^\nu,\mathbb R^{d})$ is the Hilbert space of maps $\mathbf\Phi:\mathbb R^\nu\to \mathbb R^{d}$ with $\int_{\mathbb R^d}|\mathbf\Phi(\gamma)|^2d\gamma<\infty,$ equipped with the standard scalar product $\langle \mathbf\Phi,\mathbf\Psi\rangle = \int_{\mathbb R^d} \mathbf\Phi(\gamma)\cdot \mathbf\Psi(\gamma) d\gamma$, where $\mathbf\Phi(\gamma)\cdot \mathbf\Psi(\gamma)$ denotes the scalar product of $\mathbf\Phi(\gamma)$ and $\mathbf\Psi(\gamma)$ in $\mathbb R^{d}$. The group $\mathbb R^\nu$ is naturally represented on $V,U$ by 
\begin{equation}\label{eq:rgammaphi}
R_\gamma \mathbf\Phi(\theta)=\mathbf\Phi(\theta-\gamma),\quad \mathbf\Phi\in V,\quad\gamma,\theta\in \mathbb R^\nu.
\end{equation}
Compared to the setting of the previous subsection, we  interpret the modules $V,U$ as carrying now ``infinitely extended'' and ``infinitely detailed'' $d_V$- or $d_U$-component signals. We will be interested in approximating arbitrary $\mathbb R^\nu$--equivariant continuous maps $f:V\to U$.

The second extension is that we will perform this approximation using stacked convolutional layers with local action. Our approximation will be a finite computational model, and to define it we first need to apply a discretization and a spatial cutoff to vectors from $V$ and $U$. 

Let us first describe the discretization. For any \emph{grid spacing} $\lambda>0$, let $V_\lambda$ be the subspace in $V$ formed by signals $\mathbf\Phi:\mathbb R^\nu\to\mathbb R^{d_V}$ constant on all cubes $$Q^{(\lambda)}_{\mathbf k}=\bigtimes_{s=1}^\nu \Big[\big(k_s-\tfrac{1}{2}\big)\lambda,\big(k_s-\tfrac{1}{2}\big)\lambda\Big],$$ where $\mathbf k = (k_1,\ldots,k_\nu)\in \mathbb Z^\nu.$  Let $P_\lambda$ be the orthogonal projector onto $V_\lambda$ in $V$: 
\begin{equation}\label{eq:ptau}
P_\lambda \mathbf\Phi(\gamma)=\frac{1}{\lambda^\nu}\int_{Q^{(\lambda)}_{\mathbf k}} \mathbf\Phi(\theta)d \theta,\quad \text{where }Q^{(\lambda)}_{\mathbf k}\ni \mathbf \gamma.
\end{equation}
A function $\mathbf\Phi\in V_\lambda$ can naturally be viewed as a function on the lattice $(\lambda\mathbb Z)^d$, so that we can also view $V_\lambda$ as a Hilbert space $$V_\lambda\cong L^2((\lambda\mathbb Z)^\nu,\mathbb R^{d_V}),$$ with the scalar product $\langle \mathbf\Phi,\mathbf\Psi\rangle=\lambda^\nu\sum_{\gamma\in(\lambda\mathbb Z)^\nu}\mathbf\Phi(\gamma)\cdot\mathbf\Psi(\gamma)$. We define the subspaces $U_\lambda\subset U$ similarly to the subspaces $V_\lambda\subset V$.  

Next, we define the spatial cutoff. For an integer $L\ge 0$ we denote by $Z_L$ the size-$2L$ cubic subset of the grid $\mathbb Z^\nu$:
\begin{equation}\label{eq:zl}
Z_L=\{\mathbf k\in\mathbb Z^\nu|\|\mathbf k\|_\infty\le L\},\end{equation}
where $\mathbf k=(k_1,\ldots,k_\nu)\in \mathbb Z^\nu$ and $\|\mathbf k\|_\infty=\max_{n=1,\ldots,\nu}|k_n|$. Let $\lfloor\cdot\rfloor$ denote the standard floor function. For any $\Lambda\ge 0$ (referred to as the \emph{spatial range} or \emph{cutoff}) we define the subspace $V_{\lambda,\Lambda}\subset V_\lambda$ by
\begin{align}\label{eq:convvtaul}
V_{\lambda,\Lambda}&=\{\mathbf\Phi:(\lambda \mathbb Z)^\nu\to\mathbb R^{d_V}|\mathbf\Phi(\lambda\mathbf k)=0\text{ if } \mathbf k\notin Z_{\lfloor\Lambda/\lambda\rfloor}\}\nonumber\\
&\cong\{\mathbf\Phi:\lambda Z_{\lfloor\Lambda/\lambda\rfloor}\to\mathbb R^{d_V}\}\nonumber\\
&\cong L^2(\lambda Z_{\lfloor\Lambda/\lambda\rfloor},\mathbb R^{d_V}).
\end{align}
Clearly, $\dim V_{\lambda,\Lambda}=(2\lfloor\Lambda/\lambda\rfloor+1)^\nu d_V.$ The subspaces $U_{\lambda,\Lambda}\subset U_\lambda$ are defined in a similar fashion. We will denote by $P_{\lambda,\Lambda}$ the linear operators orthogonally projecting $V$ to $V_{\lambda, \Lambda}$ or $U$ to $U_{\lambda,\Lambda}$.

In the following, we will assume that the convolutional layers have a finite \emph{receptive field} $Z_{L_\mathrm{rf}}$ -- a set of the form \eqref{eq:zl} with some fixed $L_\mathrm{rf}>0$. 

We can now describe our model of stacked convnets that will be used to approximate maps $f:V\to U$ (see Fig.\ref{fig:basic}). Namely, our approximation will be a composition of the form
\begin{equation}\label{eq:wfp}
\widehat f:V\stackrel{P_{\lambda,\Lambda+(T-1)\lambda L_{\mathrm{rf}}}}{\longrightarrow}V_{\lambda,\Lambda+(T-1)\lambda L_{\mathrm{rf}}}(\equiv W_1)
\stackrel{\widehat f_1}{\to}W_2\stackrel{\widehat f_2}{\to}\ldots\stackrel{\widehat f_T}{\to} W_{T+1}(\equiv U_{\lambda,\Lambda}).
\end{equation}
Here, the first step $P_{\lambda,\Lambda+(T-1)\lambda L_{\mathrm{rf}}}$ is an orthogonal finite-dimensional projection implementing the initial discretization and spatial cutoff of the signal. The maps $\widehat f_t$ are convolutional layers connecting intermediate spaces 
\begin{equation}\label{eq:wt}
W_t=
\begin{cases}
\{\mathbf\Phi:\lambda Z_{\lfloor\Lambda/\lambda\rfloor+(T-t)L_\mathrm{rf}}\to\mathbb R^{d_t}\},& t\le T\\
\{\mathbf\Phi:\lambda Z_{\lfloor\Lambda/\lambda\rfloor}\to\mathbb R^{d_t}\},& t= T+1
\end{cases}
\end{equation}
with some \emph{feature dimensions} $d_t$ such that $d_1=d_V$ and $d_{T+1}=d_U$.  The first intermediate space $W_1$ is identified with the space $V_{\lambda,\Lambda+(T-1)\lambda L_{\mathrm{rf}}}$ (the image of the projector $P_{\lambda,\Lambda+(T-1)\lambda L_{\mathrm{rf}}}$ applied to $V$), while the end space $W_{T+1}$ is identified with $U_{\lambda,\Lambda}$ (the respective discretization and cutoff of $U$).

The convolutional layers are defined as follows. Let $(\Phi_{\gamma n})_{\stackrel{\gamma\in Z_{\lfloor\Lambda/\lambda\rfloor+(T-t) L_{\mathrm{rf}}}}{n=1,\ldots,d_t}}$ be the coefficients in the expansion of $\mathbf\Phi\in W_t$ over the standard basis in $W_t$, as in \eqref{eq:Phi}. Then, for $t<T$ we define $\widehat f_t$ using the conventional ``linear convolution followed by nonlinear activation'' formula,
\begin{equation}\label{eq:convnonlin}
\widehat{f}_t(\mathbf \Phi)_{\gamma n}=
\sigma\Big(\sum_{\theta\in Z_{L_\mathrm{rf}}}\sum_{k=1}^{d_{t}}w^{(t)}_{n\theta k}\Phi_{\gamma+\theta,k}+h^{(t)}_n\Big), \quad \gamma\in Z_{\lfloor\Lambda/\lambda\rfloor+(T-t-1) L_{\mathrm{rf}}}, n=1,\ldots, d_{t+1},
\end{equation}
while in the last layer ($t=T$) we drop nonlinearities and only form a linear combination of values at the same point of the grid:
\begin{equation}\label{eq:convlin}
\widehat{f}_T(\mathbf \Phi)_{\gamma n}=
\sum_{k=1}^{d_{T}}w^{(T)}_{n k}\Phi_{\gamma k}+h^{(T)}_n, \quad \gamma\in Z_{\lfloor\Lambda/\lambda\rfloor}, n=1,\ldots, d_U.
\end{equation}
Note that the grid size ${\lfloor\Lambda/\lambda\rfloor+(T-t)L_{\mathrm{rf}}}$ associated with the space $W_t$ is consistent with the rule \eqref{eq:convnonlin} which evaluates the new signal $\widehat{f}(\mathbf \Phi)$ at each node of the grid as a function of the signal $\mathbf \Phi$ in the $L_\mathrm{rf}$-neighborhood of that node (so that the domain $\lambda Z_{\lfloor\Lambda/\lambda\rfloor+(T-t) L_{\mathrm{rf}}}$ ``shrinks'' slightly as $t$ grows).    

Note that we can interpret the map $\widehat f$ as a map between $V$ and $U$, since $U_{\lambda, \Lambda}\subset U$. 

\begin{figure}
\centering
\includegraphics[width=0.7\linewidth,clip,trim=40mm 8mm 10mm 12mm]{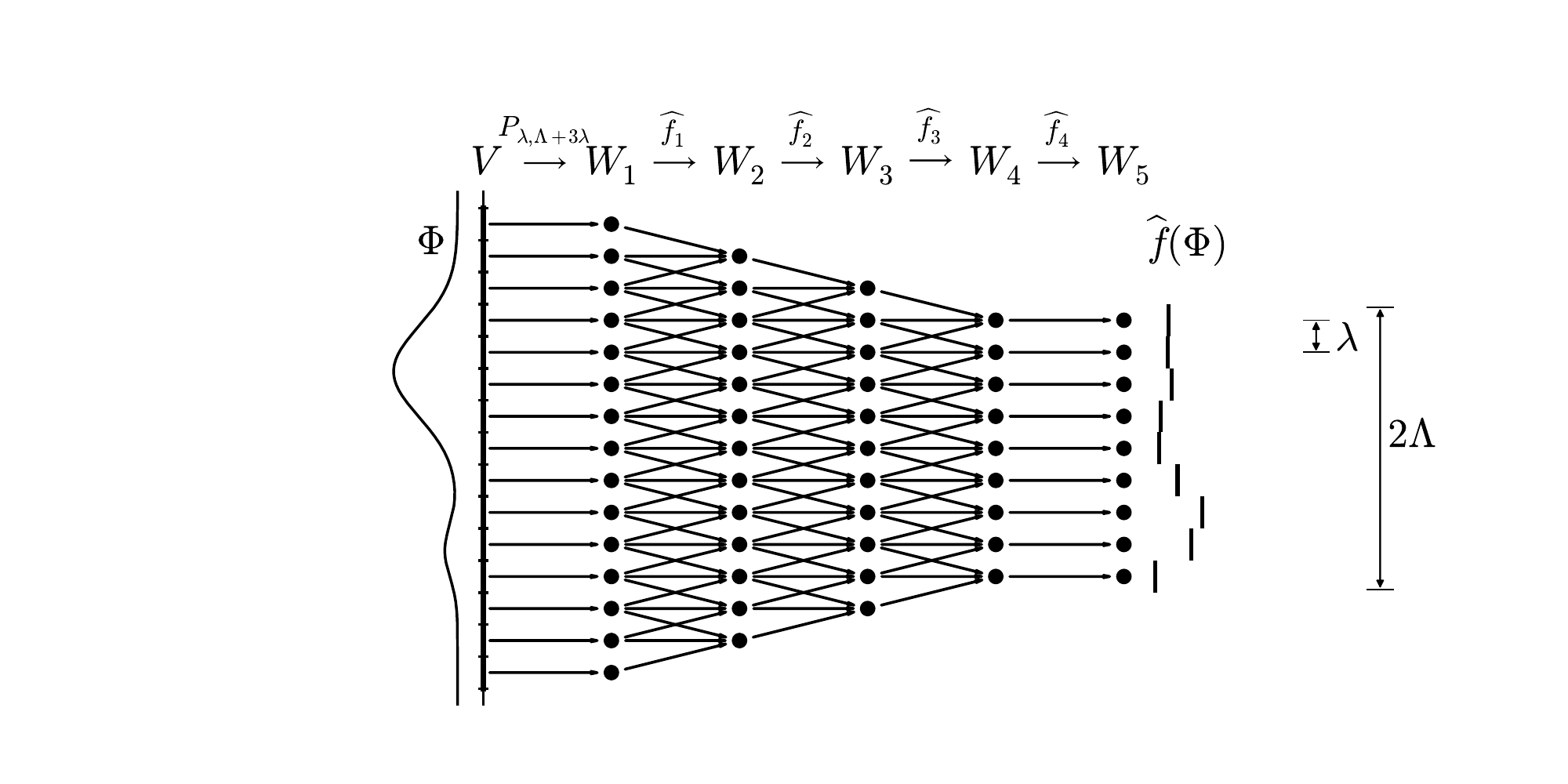}
\caption{A one-dimensional ($\nu=1$) basic convnet with the receptive field parameter $L_{\mathrm{rf}}=1$. The dots show feature spaces $\mathbb R^{d_t}$ associated with particular points of the grid $\lambda\mathbb Z$.}
\centering
\label{fig:basic}
\end{figure}
 
\begin{defin}\label{def:basicconvnetmodel} A \textbf{basic convnet} is a map $\widehat f:V\to U$ defined by \eqref{eq:wfp}, \eqref{eq:convnonlin}, \eqref{eq:convlin}, and characterized by parameters $\lambda, \Lambda, L_{\mathrm{rf}}, T, d_1,\ldots,d_{T+1}$ and coefficients $w_{n\theta k}^{(t)}$ and $h_n^{(t)}$.
\end{defin}

Note that, defined in this way, a basic convnet is a finite computational model in the following sense: while being a map between infinite-dimensional spaces $V$ and $U$, all the steps in $\widehat f$ except the initial discretization and cutoff involve only finitely many arithmetic operations and evaluations of the activation function. 

We aim to prove an analog of Theorem \ref{th:leshno}, stating that any continuous $\mathbb R^\nu$-equivariant map $f:V\to U$ can be approximated by basic convnets in the topology of uniform convergence on compact sets. However, there are some important caveats due to the fact that the space $V$ is now infinite-dimensional. 

First, in contrast to the case of finite-dimensional spaces, balls in $L^2(\mathbb R^\nu,\mathbb R^{d_V})$ are not compact. The well-known general criterion states that in a complete metric space, and in particular in $V=L^2(\mathbb R^\nu,\mathbb R^{d_V})$, a set is compact iff it is closed and \emph{totally bounded}, i.e. for any $\epsilon>0$ can be covered by finitely many $\epsilon$--balls. 

The second point (related to the first) is that a finite-dimensional space is \emph{hemicompact}, i.e., there is a sequence of compact sets such that any other compact set is contained in one of them. As a result, the space of maps $f:\mathbb R^{n}\to\mathbb R^{m}$ is \emph{first-countable} with respect to the topology of compact convergence, i.e. each point has a countable base of neighborhoods, and a point $f$ is a limit point of a set $S$ if and only if there is a sequence of points in $S$ converging to $f$. In a general topological space, however, a limit point of a set $S$ may not be representable as the limit of a sequence of points from $S$. In particular, the space $L^2(\mathbb R^\nu,\mathbb R^{d_V})$ is not hemicompact and the space of maps $f:L^2(\mathbb R^\nu,\mathbb R^{d_V})\to L^2(\mathbb R^\nu,\mathbb R^{d_U})$ is not first countable with respect to the topology of compact convergence, so that, in particular, we must distinguish between the notions of limit points of the set of convnets and the limits of sequences of convnets. We refer the reader, e.g., to the book \cite{munkres2000topology} for a general discussion of this and other topological questions and in particular to \S46 for a discussion of compact convergence. 

When defining a limiting map, we would like to require the convnets to increase their detalization $\frac{1}{\lambda}$ and range $\Lambda$. At the same time, we will regard the receptive field and its range parameter $L_\mathrm{rf}$ as arbitrary but fixed (the current common practice in applications is to use small values such as $L_\mathrm{rf}=1$ regardless of the size of the network; see, e.g., the architecture of residual networks \cite{he2016deep} providing state-of-the-art performance on image recognition tasks).

With all these considerations in mind, we introduce the following definition of a limit point of convnets.

\begin{defin}\label{def:basicconvnet}
With $V=L^2(\mathbb R^\nu,\mathbb R^{d_V})$ and $U=L^2(\mathbb R^\nu,\mathbb R^{d_U})$, we say that a map $f:V\to U$ is a \textbf{limit point of basic convnets} if for any $L_\mathrm{rf}$, any compact set $K\subset V$, and any $\epsilon>0, \lambda_0>0$ and $\Lambda_0>0$ there exists a basic convnet $\widehat f$ with the receptive field parameter $L_{\mathrm{rf}}$, spacing $\lambda \le \lambda_0$ and range $\Lambda\ge \Lambda_0$ such that $\sup_{\mathbf \Phi\in K}\|\widehat f(\mathbf \Phi)-f(\mathbf \Phi)\|<\epsilon$.  
\end{defin}
We can state now the main result of this section.
\begin{theorem}\label{th:convmain}
A map $f:V\to U$ is a limit point of basic convnets if and only if $f$ is $\mathbb R^\nu$--equivariant and continuous in the norm topology. 
\end{theorem}

Before giving the proof of the theorem, we recall the useful notion of \emph{strong convergence} of linear operators on Hilbert spaces. Namely, if $A_n$ is a sequence of bounded linear operators on a Hilbert space and $A$ is another such operator, then we say that the sequence $A_n$ converges strongly to $A$ if $A_n\mathbf \Phi$ converges  to $A\mathbf \Phi$ for any vector $\mathbf \Phi$ from this Hilbert space. More generally, strong convergence can be defined, by the same reduction, for any family $\{A_\alpha\}$ of linear operators once the convergence of the family of vectors $\{A_\alpha \mathbf \Phi\}$ is specified.  

An example of a strongly convergent family is the family of discretizing projectors $P_\lambda$ defined in \eqref{eq:ptau}. These projectors converge strongly to the identity as the grid spacing tends to 0: $P_{\lambda}\mathbf \Phi\stackrel{\lambda\to 0}{\longrightarrow}\mathbf \Phi.$ Another example is the family of projectors $P_{\lambda,\Lambda}$ projecting $V$ onto the subspace $V_{\lambda,\Lambda}$ of discretized and cut-off signals defined in \eqref{eq:convvtaul}. It is easy to see that $P_{\lambda,\Lambda}$ converge strongly to the identity as the spacing tends to 0 and the cutoff is lifted, i.e. as $\lambda\to 0$ and $\Lambda\to\infty$. Finally, our representations $R_\gamma$  defined in \eqref{eq:rgammaphi} are strongly continuous in the sense that $R_{\gamma'}$ converges strongly to $R_\gamma$ as $\gamma'\to\gamma$.  

A useful standard tool in proving strong convergence is the \emph{continuity argument}: if the family $\{A_\alpha\}$ is uniformly bounded, then the convergence $A_\alpha\mathbf \Phi\to A\mathbf \Phi$ holds for {all} vectors $\mathbf \Phi$ from the Hilbert space once it holds for a {dense subset} of vectors. This follows by approximating any $\mathbf\Phi$ with $\mathbf\Psi$'s from the dense subset and applying the inequality $\|A_\alpha\mathbf \Phi- A\mathbf \Phi\|\le \|A_\alpha\mathbf \Psi- A\mathbf \Psi\|+(\|A_\alpha\|+\|A\|)\|\mathbf \Phi-\mathbf \Psi\|$. In the sequel, we will consider strong convergence only in the settings where $A_\alpha$ are orthogonal projectors or norm-preserving operators, so the continuity argument will be applicable. 

\begin{proof}[Proof of Theorem \ref{th:convmain}.]\mbox{}

\paragraph*{Necessity}\emph{(a limit point of basic convnets is $\mathbb R^\nu$--equivariant and continuous)}.

The continuity of $f$ follows by a standard argument from the uniform convergence on compact sets and the continuity of convnets (see Theorem 46.5 in \cite{munkres2000topology}). 

Let us prove the $\mathbb R^\nu$--equivariance of $f$, i.e.  
\begin{equation}\label{eq:frgp}
f(R_\gamma \mathbf\Phi)=R_\gamma f(\mathbf\Phi).
\end{equation}
Let $D_{M} =[-M,M]^\nu\subset\mathbb R^\nu$ with some $M>0$, and $P_{D_M}$ be the orthogonal projector in $U$ onto the subspace of signals supported on the set $D_{M}$. Then $P_{D_M}$ converges strongly to the identity as $M\to +\infty$. Since $R_\gamma$ is a bounded linear operator, \eqref{eq:frgp} will follow if we prove that for any $M$
\begin{equation}\label{eq:convpdfr}
P_{D_M}f(R_\gamma \mathbf\Phi)=R_\gamma P_{D_M}f(\mathbf\Phi).
\end{equation}
Let $\epsilon>0$. Let $\gamma_\lambda\in (\lambda\mathbb Z)^\nu$ be the nearest point to $\gamma\in \mathbb R^\nu$ on the grid $(\lambda\mathbb Z)^\nu$. Then, since $R_{\gamma_\lambda}$ converges strongly to $R_{\gamma}$ as $\lambda\to 0$, there exist $\lambda_0$ such that for any $\lambda<\lambda_0$ 
\begin{equation}\label{eq:rgtpd}
\|R_{\gamma_\lambda} P_{D_M} f(\mathbf\Phi)-R_{\gamma} P_{D_M} f(\mathbf\Phi)\|\le\epsilon,
\end{equation}
and
\begin{equation}\label{eq:frgmp}
\|f(R_{\gamma}\mathbf\Phi)-f(R_{\gamma_\lambda}\mathbf\Phi)\|\le \epsilon,
\end{equation}
where we have also used the already proven continuity of $f$.

Observe that the discretization/cutoff projectors $P_{\lambda, M}$ converge strongly to $P_{D_M}$ as $\lambda\to 0$, hence we can ensure that for any $\lambda<\lambda_0$ we also have
\begin{align}\label{eq:convpdfrgf}
\begin{split}
\|P_{D_M}f(R_\gamma \mathbf\Phi)-P_{\lambda,M}f(R_{\gamma}
\mathbf\Phi)\|\le\epsilon,\\
\|P_{\lambda,M} f(\mathbf\Phi)-P_{D_\Lambda} f(\mathbf\Phi)\|\le\epsilon.
\end{split}
\end{align}

Next, observe that basic convnets are partially translationally equivariant by our definition, in the sense that if the cutoff parameter $\Lambda$ of the convnet is sufficiently large then
\begin{equation}\label{eq:conveqpt}
P_{\lambda,M}\widehat f(R_{\gamma_\lambda} \mathbf\Phi)=R_{\gamma_\lambda} P_{\lambda,M}\widehat f(\mathbf\Phi).
\end{equation}
Indeed, this identity holds as long as both sets $\lambda Z_{\lfloor M/\lambda\rfloor}$ and $\lambda Z_{\lfloor M/\lambda\rfloor}-\gamma_\lambda$ are subsets of $\lambda Z_{\lfloor\Lambda/\lambda\rfloor}$ (the domain where convnet's output is defined, see \eqref{eq:wt}). This condition is satisfied if we require that $\Lambda>\Lambda_0$ with $\Lambda_0=M+\lambda(1+\|\gamma\|_\infty)$.

Now, take the compact set $K=\{R_\theta\mathbf \Phi| \theta\in \mathcal N\},$ where $\mathcal N\subset \mathbb R^\nu$ is some compact set including 0 and all points $\gamma_\lambda$ for $\lambda<\lambda_0$. Then, by our definition of a limit point of basic convnets, there is a convnet $\widehat f$ with $\lambda<\lambda_0$ and $L>L_0$ such that for all $\theta\in\mathcal N$ (and in particular for $\theta=0$ or $\theta=\gamma_\lambda$)
\begin{equation}\label{eq:convfrtm}
\|f(R_\theta\mathbf\Phi)-\widehat f(R_\theta\mathbf\Phi)\|<\epsilon.
\end{equation}
We can now write a bound for the difference of the two sides of \eqref{eq:convpdfr}: 
\begin{align*}
\|P_{D_M}f(R_\gamma \mathbf\Phi)- & R_\gamma P_{D_M}f(\mathbf\Phi)\| \\
{}\le{} & \|P_{D_M}f(R_\gamma \mathbf\Phi)-P_{\lambda,M}f(R_{\gamma}
\mathbf\Phi)\|+\|P_{\lambda,M}f(R_{\gamma}
\mathbf\Phi)-P_{\lambda,M}f(R_{\gamma_\lambda}
\mathbf\Phi)\| \\
&+ \|P_{\lambda,M}f(R_{\gamma_\lambda}
\mathbf\Phi)-P_{\lambda,M}\widehat f(R_{\gamma_\lambda}
\mathbf\Phi)\|+\|P_{\lambda,M}\widehat f(R_{\gamma_\lambda} \mathbf\Phi)-R_{\gamma_\lambda} P_{\lambda,M}\widehat f(\mathbf\Phi)\|\\
&+ \|R_{\gamma_\lambda} P_{\lambda,M}\widehat f(\mathbf\Phi)-R_{\gamma_\lambda} P_{\lambda,M} f(\mathbf\Phi)\|+\|R_{\gamma_\lambda} P_{\lambda,M} f(\mathbf\Phi)-R_{\gamma_\lambda} P_{D_M} f(\mathbf\Phi)\|\\
&+\|R_{\gamma_\lambda} P_{D_M} f(\mathbf\Phi)-R_{\gamma} P_{D_M} f(\mathbf\Phi)\|\\
\le{} & \|P_{D_M}f(R_\gamma \mathbf\Phi)-P_{\lambda,M}f(R_{\gamma}
\mathbf\Phi)\|+\|f(R_{\gamma}
\mathbf\Phi)-f(R_{\gamma_\lambda}
\mathbf\Phi)\| \\
&+ \|f(R_{\gamma_\lambda}
\mathbf\Phi)-\widehat f(R_{\gamma_\lambda}
\mathbf\Phi)\|+ \|\widehat f(\mathbf\Phi)- f(\mathbf\Phi)\|\\
&+\|P_{\lambda,M} f(\mathbf\Phi)-P_{D_\Lambda} f(\mathbf\Phi)\|+\|R_{\gamma_\lambda} P_{D_M} f(\mathbf\Phi)-R_{\gamma} P_{D_M} f(\mathbf\Phi)\|\\
\le{}& 6\epsilon,
\end{align*}
Here in the first step we split the difference into several parts, in the second step we used the identity \eqref{eq:conveqpt} and the fact that $P_{\lambda,M}, R_{\gamma_\lambda}$ are linear operators with the operator norm 1, and in the third step we applied the inequalities \eqref{eq:rgtpd}--\eqref{eq:convpdfrgf} and \eqref{eq:convfrtm}. Since $\epsilon$ was arbitrary, we have proved \eqref{eq:convpdfr}.

\paragraph*{Sufficiency}
\emph{(an $\mathbb R^\nu$--equivariant and continuous map is a limit point of basic convnets)}. We start by proving a key lemma on the approximation capability of basic convnets in the special case when they have the degenerate output range, $\Lambda=0$. In this case, by \eqref{eq:wfp}, the output space $W_T=U_{\lambda,0}\cong \mathbb R^{d_U},$ and 
the first auxiliary space $W_1=V_{\lambda, (T-1)\lambda L_{\mathrm{rf}}}\subset V$.    
\begin{lemma}\label{lm:conv}
Let $\lambda,T$ be fixed and $\Lambda=0$. Then any continuous map $f:V_{\lambda,(T-1)\lambda L_{\mathrm{rf}}}\to U_{\lambda,0}$ can be approximated by basic convnets having spacing $\lambda$, depth $T$, and range $\Lambda=0$.
\end{lemma}

Note that this is essentially a finite-dimensional approximation result, in the sense that the input  space $V_{\lambda,(T-1)\lambda L_{\mathrm{rf}}}$ is finite-dimensional and fixed. The approximation is achieved by choosing sufficiently large feature dimensions $d_t$ and suitable weights in the intermediate layers. 

\begin{proof}

The idea of the proof is to divide the operation of the convnet into two stages. The first stage is implemented by the first $T-2$ layers and consists in approximate ``contraction'' of the input vectors, while the second stage, implemented by the remaining two layers, performs the actual approximation. 

The contraction stage is required because the components of the input signal $\mathbf \Phi_{\rm in}\in V_{\lambda,(T-1)\lambda L_\mathrm{rf}}\cong L^2(\lambda Z_{(T-1) L_\mathrm{rf}}, \mathbb R^{d_V})$ are distributed over the large spatial domain $\lambda Z_{(T-1) L_\mathrm{rf}}$. In this stage we will map the input signal to the spatially localized space $W_{T-1}\cong L^2(\lambda Z_{L_\mathrm{rf}},\mathbb R^{d_{T-1}})$  so as to approximately preserve the information in the signal. 

Regarding the second  stage, observe that the last two layers of the convnet (starting from $W_{T-1}$) act on signals in $W_{T-1}$ by an expression analogous to the one-hidden-layer network from the basic universal approximation theorem (Theorem \ref{th:leshno}):
\begin{equation}\label{eq:wftft}\Big(\widehat f_{T}\circ \widehat f_{T-1} (\mathbf \Phi)\Big)_{n}=\sum_{k=1}^{d_T}w_{nk}^{(T)}\sigma\Big(\sum_{\theta\in Z_{L_\mathrm{rf}}}\sum_{m=1}^{d_{T-1}}w_{k\theta m}^{(T-1)}\Phi_{\theta m}+h_{k}^{(T-1)}\Big)+h_{n}^{(T)}.
\end{equation}
This expression involves all components of $\mathbf \Phi\in W_{T-1},$ and so we can conclude by Theorem \ref{th:leshno} that by choosing a sufficiently large dimension $d_T$ and appropriate weights we can approximate an arbitrary continuous map from $W_{T-1}$ to $U_{\lambda,0}$. 

Now, given a continuous map $f:V_{\lambda,(T-1)\lambda L_{\mathrm{rf}}}\to U_{\lambda,0}$, consider the map $g=f\circ I\circ P:W_{T-1}\to U_{\lambda,0}$, where $I$ is some linear isometric map from a subspace $W'_{T-1}\subset W_{T-1}$ to $V_{\lambda,(T-1)\lambda L_{\mathrm{rf}}}$, and $P$ is the projection in $W_{T-1}$ to $W'_{T-1}$. Such isometric  $I$ exists if $\dim W_{T-1}\ge \dim V_{\lambda,(T-1)\lambda L_{\mathrm{rf}}}$, which we can assume w.l.o.g. by choosing sufficiently large $d_{T-1}$. Then the map $g$ is continuous, and the previous argument shows that we can approximate $g$ using the second stage of the convnet. Therefore, we can also approximate the given map $f=g\circ I^{-1}$ by the whole convnet if we manage to exactly implement or approximate the isometry $I^{-1}$  in the contraction stage. 

Implementing such an isometry would be straightforward if the first $T-2$ layers had no activation function (i.e., if $\sigma$ were the identity function in the nonlinear layers \eqref{eq:convnonlin}). In this case for all $t=2,3,\ldots,T-1$ we can choose the feature dimensions $d_t=|Z_{L_\mathrm{rf}}|d_{t-1}=(2L_\mathrm{rf}+1)^{\nu(t-1)}d_V$ and set $h_n^{(t)}=0$ and $$w^{(t)}_{n\theta k}=\begin{cases}1,& n =\psi_t(\theta,k), \\ 0, & \text{otherwise,}\end{cases}$$
where $\psi_t$ is some bijection between $Z_{L_\mathrm{rf}}\times \{1,\ldots,d_{t}\}$ and $\{1,\ldots,d_{t+1}\}$. In this way, each component of the network input vector $\mathbf \Phi_{\rm in}$ gets copied, layer by layer, to subsequent layers  and eventually ends up among the components of the resulting vector in $W_{T-1}$ (with some repetitions due to multiple possible trajectories of copying).

However, since $\sigma$ is not an identity, copying needs to be approximated. Consider the first layer, $\widehat f_1$. For each $\gamma\in Z_{L_\mathrm{rf}}$ and each $s\in \{1,\ldots,d_1\}$, consider the corresponding coordinate map $$g_{\gamma s}:L^2(\lambda Z_{L_\mathrm{rf}}, \mathbb R^{d_1})\to \mathbb R,\quad g_{\gamma s}:\mathbf\Phi\mapsto \Phi_{\gamma s}.$$ 
By Theorem \ref{th:leshno}, the map $g_{\gamma s}$ can be approximated with arbitrary accuracy on any compact set in  $L^2(\lambda Z_{L_\mathrm{rf}}, \mathbb R^{d_1})$ by  maps of the form 
\begin{equation}\label{eq:convlemma1}
\mathbf \Phi\mapsto\sum_{m=1}^{N} c_{\gamma sm}\sigma\Big(\sum_{\theta\in Z_{L_\mathrm{rf}}}\sum_{k=1}^{d_1}w_{\gamma sm\theta k}\Phi_{\theta k}+h_{\gamma sm}\Big),
\end{equation}
where we may assume without loss of generality that $N$ is the same for all $\gamma, s$. We then set the second feature dimension $d_2=N|Z_{L_\mathrm{rf}}|d_{1}=N(2L_\mathrm{rf}+1)^\nu d_V$ and assign the weights $w_{\gamma sm\theta k}$ and $h_{\gamma sm}$ in \eqref{eq:convlemma1} to be the weights $w_{n\theta k}^{(1)}$ and $h_{n}^{(1)}$ of the first convnet layer, where the index $n$ somehow enumerates the triplets $(\gamma,s,m)$. Defined in this way, the first convolutional layer $f_1$ only partly reproduces the copy operation, since this layer does not include the linear weighting corresponding to the external summation over $m$ in \eqref{eq:convlemma1}. However, we can include this weighting into the next layer, since this operation involves only values at the same spatial location $\gamma\in \mathbb Z^\nu$, and prepending this operation to the convolutional layer \eqref{eq:convlemma1} does not change the functional form of the layer.

By repeating this argument for the subsequent layers $t=2,3,\ldots,T-2$, we can make the sequence of the first $T-2$ layers to arbitrarily accurately copy all the components of the input vector $\mathbf \Phi_{\rm in}$ into a vector $\mathbf \Phi\in W_{T-1}$, up to some additional linear transformations that need to be included in the $(T-1)$'th layer (again, this is legitimate since prepending a linear operation does not change the functional form of the $(T-1)$'th layer). Thus, we can approximate $f=g\circ I^{-1}$ by arranging the first stage of the convnet to approximate $I^{-1}$ and the second to approximate  $g$.  
\end{proof} 

Returning to the proof of sufficiency, let $f:V\to U$ be an $\mathbb R^\nu$--equivariant continuous map that we need to approximate with accuracy $\epsilon$ on a compact set $K\subset V$ by a convnet with $\lambda<\lambda_0$ and $\Lambda>\Lambda_0$. For any $\lambda$ and $\Lambda$, define the map $$f_{\lambda,\Lambda}=P_{\lambda,\Lambda}\circ f\circ P_\lambda.$$ Observe that we can find $\lambda<\lambda_0$ and $\Lambda>\Lambda_0$ such that \begin{equation}\label{eq:convsupfk}
\sup_{\mathbf\Phi\in K}\|f_{\lambda,\Lambda}(\mathbf\Phi)-f(\mathbf\Phi)\|\le\frac{\epsilon}{3}.
\end{equation}
Indeed, this can be proved as follows. Denote by $B_\delta(\mathbf\Phi)$ the radius--$\delta$ ball centered at $\mathbf\Phi$. By compactness of $K$ and continuity of $f$ we can find finitely many signals $\mathbf \Phi_n\in V,n=1,\ldots,N,$ and some $\delta>0$ so that, first, $K\subset\cup_n B_{\delta/2}(\mathbf\Phi_n)$, and second, 
\begin{equation}\label{eq:convfmphif}\|f(\mathbf\Phi)-f(\mathbf\Phi_n)\|\le \frac{\epsilon}{9},\quad \mathbf \Phi\in B_{\delta}(\mathbf \Phi_n).
\end{equation} 
For any $\mathbf \Phi\in K$, pick $n$ such that $\mathbf \Phi\in B_{\delta/2}(\mathbf\Phi_n)$, then
\begin{align}\|f_{\lambda,\Lambda}(\mathbf\Phi)-f(\mathbf\Phi)\|\le {}&\|P_{\lambda,\Lambda}f(P_\lambda\mathbf\Phi)-P_{\lambda,\Lambda}f(\mathbf\Phi_n)\|+\|P_{\lambda,\Lambda}f(\mathbf\Phi_n)-f(\mathbf\Phi_n)\|+\|f(\mathbf\Phi_n)-f(\mathbf\Phi)\|\nonumber\\
\le{} &\|f(P_\lambda\mathbf\Phi)-f(\mathbf\Phi_n)\|+\|P_{\lambda,\Lambda}f(\mathbf\Phi_n)-f(\mathbf\Phi_n)\|+\frac{\epsilon}{9}.\label{eq:convfptaump}
\end{align}
Since $\mathbf \Phi\in B_{\delta/2}(\mathbf\Phi_n)$, if $\lambda$ is sufficiently small then $P_\lambda\mathbf \Phi\in B_{\delta}(\mathbf\Phi_n)$ (by the strong convergence of $P_\lambda$ to the identity) and hence $\|f(P_\lambda\mathbf\Phi)-f(\mathbf\Phi_n)\|<\frac{\epsilon}{9}$, again by \eqref{eq:convfmphif}. Also, we can choose sufficiently small $\lambda$ and then sufficiently large $\Lambda$ so that $\|P_{\lambda,L}f(\mathbf\Phi_n)-f(\mathbf\Phi_n)\|<\frac{\epsilon}{9}$. Using these inequalities in \eqref{eq:convfptaump}, we obtain \eqref{eq:convsupfk}. 

Having thus chosen $\lambda$ and $\Lambda$, observe that, by translation equivariance of $f$, the map $f_{\lambda,\Lambda}$ can be written as
$$f_{\lambda,\Lambda}(\mathbf\Phi)=\sum_{\gamma\in Z_{\lfloor\Lambda/\lambda\rfloor}}R_{\lambda\gamma}P_{\lambda,0}f(P_\lambda R_{-\lambda\gamma}\mathbf\Phi),$$
where $P_{\lambda,0}$ is the projector $P_{\lambda,\Lambda}$ in the degenerate case $\Lambda=0$. 
Consider the map
$$f_{\lambda,\Lambda,T}(\mathbf\Phi)=\sum_{\gamma\in Z_{\lfloor\Lambda/\lambda\rfloor}}R_{\lambda\gamma}P_{\lambda,0}f(P_{\lambda,(T-1)\lambda L_\mathrm{rf}}R_{-\lambda\gamma}\mathbf\Phi).$$
Then, by choosing $T$ sufficiently large, we can ensure that 
\begin{equation}\label{eq:convftlt1}
\sup_{\mathbf\Phi\in K}\|f_{\lambda,\Lambda,T}(\mathbf\Phi)-f_{\lambda,\Lambda}(\mathbf\Phi)\|<\frac{\epsilon}{3}.
\end{equation}
Indeed, this can be proved in the same way as \eqref{eq:convsupfk}, by using compactness of $K$, continuity of $f$, finiteness of $Z_{\lfloor\Lambda/\lambda\rfloor}$ and the strong convergence $P_{\lambda,(T-1)\lambda L_\mathrm{rf}}R_{-\lambda\gamma}\mathbf\Phi\stackrel{T\to\infty}{\longrightarrow}P_\lambda R_{-\lambda\gamma}\mathbf\Phi$.

Observe that $f_{\lambda,\Lambda,T}$ can be alternatively written as
\begin{equation}\label{eq:convftaultmf}
f_{\lambda,\Lambda,T}(\mathbf\Phi)=\sum_{\gamma\in Z_{\lfloor\Lambda/\lambda\rfloor}}R_{\lambda\gamma}f_{\lambda,0,T}(R_{-\lambda\gamma}\mathbf\Phi),
\end{equation}
where 
$$f_{\lambda,0,T}(\mathbf\Phi)=P_{\lambda,0}f(P_{\lambda,(T-1)\lambda L_\mathrm{rf}}\mathbf\Phi).$$
We can view the map $f_{\lambda,0,T}$ as a map from $V_{\lambda,(T-1)\lambda L_\mathrm{rf}}$ to $U_{\lambda,0}$, which makes Lemma \ref{lm:conv} applicable to $f_{\lambda,0,T}$. Hence, since $\cup_{\gamma\in Z_{\lfloor\Lambda/\lambda\rfloor}} R_{-\lambda\gamma}K$ is compact, we can find a convnet $\widehat f_0$ with  spacing $\lambda$, depth $T$ and range $\Lambda=0$ such that 
\begin{equation}\label{eq:convwhf0}
\|\widehat f_0(\mathbf\Phi)-f_{\lambda,0,T}(\mathbf\Phi)\|<\frac{\epsilon}{3|Z_{\lfloor\Lambda/\lambda\rfloor}|},
\quad\mathbf\Phi\in\cup_{\gamma\in Z_{\lfloor\Lambda/\lambda\rfloor}} R_{-\lambda\gamma}K.
\end{equation}
Consider the convnet $\widehat f_\Lambda$ different from $\widehat f_0$ only by the range parameter $\Lambda$; such a convnet can be written in terms of $\widehat f_0$ in the same way as $f_{\lambda,\Lambda,T}$ is written in terms of $f_{\lambda,0,T}$:
\begin{equation}\label{eq:convwhfl}\widehat f_\Lambda(\mathbf\Phi)=\sum_{\gamma\in Z_{\lfloor\Lambda/\lambda\rfloor}}R_{\lambda\gamma}\widehat f_0(R_{-\lambda\gamma }\mathbf\Phi).\end{equation}
Combining \eqref{eq:convftaultmf}, \eqref{eq:convwhf0} and \eqref{eq:convwhfl}, we obtain
$$\sup_{\mathbf \Phi\in K}\|\widehat f_\Lambda(\mathbf\Phi)-f_{\lambda,\Lambda,T}(\mathbf\Phi)\|<\frac{\epsilon}{3}.$$
Combining this bound with bounds \eqref{eq:convsupfk} and \eqref{eq:convftlt1}, we obtain the desired bound  
$$\sup_{\mathbf \Phi\in K}\|\widehat f_\Lambda(\mathbf\Phi)-f(\mathbf\Phi)\|<\epsilon.$$
\end{proof}

Theorem \ref{th:convmain} suggests that our definition of limit points of basic convnets provides a reasonable rigorous framework for the analysis of convergence and invariance properties of convnet-like models in the limit of continual and infinitely extended signals. We will use these definition and theorem as templates when considering convnets with pooling in the next subsection and charge--conserving convnets in Section \ref{sec:charge}.

\subsection{Convnets with pooling}\label{sec:convpool}
As already mentioned, pooling erodes the equivariance of models with respect to translations. Therefore, we will consider convnets with pooling as universal approximators without assuming the approximated maps to be translationally invariant. Also, rather than considering $L^2(\mathbb R^\nu,\mathbb R^{d_U})$--valued maps, we will be interested in approximating simply $\mathbb R$--valued maps, i.e., those of the form  $f:V\to \mathbb R$, where, as in Section \ref{sec:convnet_deep}, $V=L^2(\mathbb R^\nu, \mathbb R^{d_V}).$

\begin{figure}
\centering
\includegraphics[width=0.7\linewidth,clip,trim=30mm 10mm 30mm 6mm]{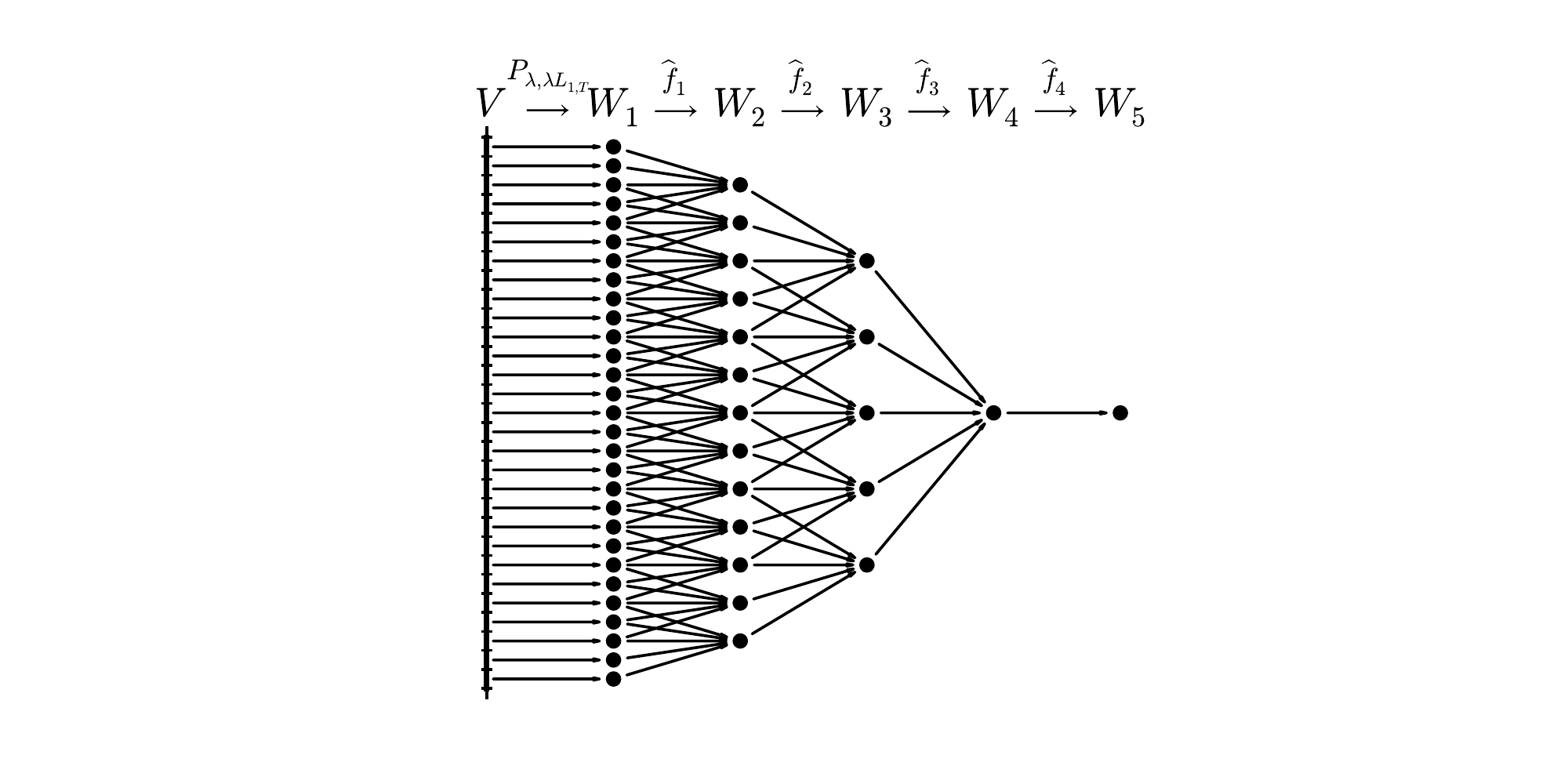}
\caption{A one-dimensional ($\nu=1$) convnet with downsampling having stride $s=2$ and the receptive field parameter $L_{\mathrm{rf}}=2$. }
\centering
\label{fig:downsampled}
\end{figure}

While the most popular kind of pooling in applications seems to be max-pooling, we will only consider pooling by decimation (i.e., grid downsampling), which appears to be about as efficient in practice (see \cite{springenberg2014striving}). Compared to basic convnets of Section \ref{sec:convnet_deep}, convnets with downsampling then have a new parameter, \emph{stride}, that we denote by $s$. The stride can take values $s=1,2,\ldots$ and determines the geometry scaling when passing information to the next convnet layer: if the current layer operates on a grid $(\lambda \mathbb Z)^\nu$, then the next layer will operate on the subgrid $(s\lambda \mathbb Z)^\nu$. Accordingly, the current layer only needs to perform the operations having outputs located in this subgrid. We will assume $s$ to be fixed and to be the same for all layers. Moreover, we assume that 
\begin{equation}\label{eq:convstride}
s\le 2L_\mathrm{rf}+1,
\end{equation}
i.e., the stride is not larger than the size of the receptive field: this ensures that information from each node of the current grid can reach the next layer. 

Like the basic convnet of Section \ref{sec:convnet_deep}, a convnet with downsampling can be written as a chain: 
\begin{equation}\label{eq:convwfpool}
\widehat f:V\stackrel{P_{\lambda,\lambda L_{1,T}}}{\longrightarrow}V_{\lambda,\lambda L_{1,T}}(\equiv W_1)\stackrel{\widehat f_1}{\to}W_2\stackrel{\widehat f_2}{\to}\ldots\stackrel{\widehat f_{T}}{\to} W_{T+1}(\cong \mathbb R).
\end{equation}
Here the space $V_{\lambda,\lambda L_{1,T}}$ is defined as in \eqref{eq:convvtaul} (with $\Lambda=\lambda L_{1,T}$) and $P_{\lambda,\lambda L_{1,T}}$ is the orthogonal projector to this subspace. 
The intermediate spaces are defined by
$$W_t= L^2(s^{t-1}\lambda Z_{L_{t,T}}, \mathbb R^{d_t}).$$
The range parameters $L_{t,T}$ are given by $$L_{t,T}=
\begin{cases}
L_\mathrm{rf}(1+s+s^2+\ldots+s^{T-t-1}), & t< T,\\
0, & t=T, T+1.\end{cases}
$$
This choice of $L_{t,T}$ is equivalent to the identities $$L_{t,T}=sL_{t+1,T}+L_{\mathrm{rf}},\quad t=1,\ldots,T-1,$$ 
expressing the domain transformation under downsampling.

The feature dimensions $d_t$ can again take any values, aside from the fixed values $d_1=d_V$ and $d_{T+1}=1$. 

As the present convnet model is $\mathbb R$--valued, in contrast to the basic convnet of Section \ref{sec:convnet_deep}, it does not have a separate output cutoff parameter $\Lambda$ (we essentially have $\Lambda=0$ now). The geometry of the input domain $\lambda Z_{L_{1,T}}$ is fully determined by  stride $s$, the receptive field parameter $L_\mathrm{rf}$, grid spacing $\lambda$, and depth $T$. Thus, the architecture of the model is fully specified by these parameters and feature dimensions $d_2,\ldots,d_{T}$.

The layer operation formulas differ from the formulas \eqref{eq:convnonlin},\eqref{eq:syminv} by the inclusion of downsampling:   
\begin{equation}\label{eq:convnonlinpool}
\widehat{f}_t(\mathbf \Phi)_{\gamma n}=
\sigma\Big(\sum_{\theta\in Z_{L_\mathrm{rf}}}\sum_{k=1}^{d_{t}}w^{(t)}_{n\theta k}\Phi_{s\gamma+\theta,k}+h^{(t)}_n\Big), \quad \gamma\in Z_{L_{t+1}}, n=1,\ldots, d_{t+1},
\end{equation}
\begin{equation}\label{eq:convlinpool}
\widehat{f}_{T+1}(\mathbf \Phi)=
\sum_{k=1}^{d_{T}}w^{(T)}_{n k}\Phi_{k}+h^{(T)}_n.
\end{equation}
Summarizing, we define convnets with downsampling as follows. 
\begin{defin}\label{def:convnetwithpoolmodel} 
A \textbf{convnet with downsampling} is a map $\widehat f:V\to \mathbb R$ defined by \eqref{eq:convwfpool}, \eqref{eq:convnonlinpool}, \eqref{eq:convlinpool}, and characterized by parameters $s, \lambda, L_{\mathrm{rf}}, T, d_1,\ldots,d_{T}$ and coefficients $w_{n\theta k}^{(t)}$ and $h_n^{(t)}$.   
\end{defin}
Next, we give a definition of a limit point of convnets with downsampling analogous to Definition \ref{def:basicconvnet} for basic convnets. In this definition, we require that the input domain grow in detalization $\frac{1}{\lambda}$ and in the spatial range $\lambda L_{1,T}$, while the stride and receptive field are fixed. 
\begin{defin}\label{def:convpool}
With $V=L^2(\mathbb R^\nu,\mathbb R^{d_V})$, we say that a map $f:V\to \mathbb R$ is a \textbf{limit point of convnets with downsampling} if for any $s$ and  $L_{\mathrm{rf}}$ subject to Eq.\eqref{eq:convstride}, any compact set $K\in V$, any $\epsilon>0, \lambda_0>0$ and $\Lambda_0>0$ there exists a convnet with downsampling $\widehat f$ with stride $s$,  receptive field parameter $L_{\mathrm{rf}}$, depth $T$, and spacing $\lambda \le \lambda_0$  such that $\lambda L_{1,T}\ge \Lambda_0$ and $\sup_{\mathbf \Phi\in K}\|\widehat f(\mathbf \Phi)-f(\mathbf \Phi)\|<\epsilon$.  
\end{defin}  

The analog of Theorem \ref{th:convmain} then reads:

\begin{theorem}\label{th:convpool}
A map $f:V\to \mathbb R$ is a limit point of convnets with downsampling if and only if $f$ is continuous in the norm topology. 
\end{theorem}
\begin{proof}
The proof is completely analogous to, and in fact simpler than, the proof of Theorem \ref{th:convmain}, so we only sketch it. 

The necessity only involves the claim of continuity and follows again by a basic topological argument.

In the proof of sufficiency, an analog of Lemma \ref{lm:conv} holds for convnets with downsampling, since, thanks to the constraint \eqref{eq:convstride} on the stride, all points of the input domain $\lambda Z_{L_1}$ are connected by the network architecture to the output (though there are fewer connections now due to pooling), so that our construction of approximate copy operations remains valid. 

To approximate $f:V\to \mathbb R$ on a compact $K$, first approximate it by a map $f\circ P_{\lambda, Z_{L_{1,T}}}$ with a sufficiently small $\lambda$ and large $T$, then use the lemma to approximate $f\circ P_{\lambda, Z_{L_{1,T}}}$ by a convnet. 
\end{proof}

\section{Charge-conserving convnets}\label{sec:charge}

The goal of the present section is to describe a complete convnet-like model for approximating arbitrary continuous maps equivariant with respect to rigid planar motions. A rigid motion of $\mathbb R^\nu$ is an affine transformation  preserving the distances and the orientation in $\mathbb R^\nu$. The group $\mathrm{SE}(\nu)$ of all such motions can be described as a \emph{semidirect product} of the translation group $\mathbb R^\nu$ with the special orthogonal group $\mathrm{SO}(\nu)$:
$$\mathrm{SE}(\nu)=\mathbb R^\nu \rtimes \mathrm{SO}(\nu).$$
An element of $\mathrm{SE}(\nu)$ can be represented as a pair $(\gamma,\theta)$ with $\gamma\in \mathbb R^\nu$ and $\theta\in \mathrm{SO}(\nu).$ The group operations are given by 
\begin{align*}
(\gamma_1,\theta_1)(\gamma_2,\theta_2)&=(\gamma_1+\theta_1\gamma_2, \theta_1\theta_2),\\
(\gamma,\theta)^{-1}&=(-\theta^{-1}\gamma,\theta^{-1}).
\end{align*} 
The group $\mathrm{SE}(\nu)$ acts on $\mathbb R^\nu$ by
$$\mathcal A_{(\gamma,\theta)}\mathbf x=\gamma+\theta\mathbf x.$$
It is easy to see that this action is compatible with the group operation, i.e. $\mathcal A_{(0,1)} = \mathrm{Id}$  and $\mathcal A_{(\gamma_1,\theta_1)}\mathcal A_{(\gamma_2,\theta_2)}=\mathcal A_{(\gamma_1,\theta_1)(\gamma_2,\theta_2)}$ (implying, in particular, $\mathcal A_{(\gamma,\theta)}^{-1}=\mathcal A_{(\gamma,\theta)^{-1}}$).

As in Section \ref{sec:convnet_deep}, consider the space $V=L^2(\mathbb R^\nu,\mathbb R^{d_V}).$ We can view this space as a $\mathrm{SE}(\nu)$--module with the representation canonically associated with the action $\mathcal A$: 
\begin{equation}\label{eq:chargergta}
R_{(\gamma,\theta)}\mathbf\Phi(\mathbf x)=\mathbf\Phi(\mathcal A_{(\gamma,\theta)^{-1}}\mathbf x),\end{equation}
where $\mathbf \Phi:\mathbb R^\nu\to\mathbb R^{d_V}$ and $\mathbf x\in\mathbb R^\nu$. We define in the same manner the module $U=L^2(\mathbb R^\nu,\mathbb R^{d_U})$. In the remainder of the paper we will be interested in approximating \emph{continuous and SE($\nu$)-equivariant} maps $f:V\to U$. Let us first give some examples of such maps.
\begin{description}
\item[Linear maps.] 
Assume for simplicity that $d_V=d_U=1$ and consider a \emph{linear} SE($\nu$)-equivariant map $f:L^2(\mathbb R^\nu)\to L^2(\mathbb R^\nu)$. Such a map can be written as a convolution $f(\mathbf \Phi)=\mathbf\Phi*\mathbf\Psi_f,$ where $\mathbf\Psi_f$ is a radial signal, $\mathbf\Psi_f(\mathbf x)=\widetilde{\mathbf\Psi}_f(|\mathbf x|).$ In general, $\mathbf\Psi_f$ should be understood in a distributional sense. 

By applying Fourier transform $\mathcal F$, the map $f$ can be equivalently described in the Fourier dual space as pointwise multiplication of the given signal by $const \mathcal F{\mathbf\Psi}_f$ (with the constant depending on the choice of the coefficient in the Fourier transfrom), so $f$ is SE($\nu$)-equivariant and continuous if and only if $\mathcal F{\mathbf\Psi}_f$ is a radial function belonging to $L^\infty(\mathbb R^\nu)$. Note that in this argument we have tacitly complexified the space $L^2(\mathbb R^\nu, \mathbb R)$ into $L^2(\mathbb R^\nu, \mathbb C)$. The condition that $f$ preserves real-valuedness of the signal $\mathbf \Phi$ translates into $\overline{\mathcal F{\mathbf\Psi}_f(\mathbf x)}=\mathcal F{\mathbf\Psi}_f(-\mathbf x)$, where the bar denotes complex conjugation.  

Note that linear SE($\nu$)-equivariant differential operators, such as the Laplacian $\Delta$, are not included in our class of maps, since they are not even defined on the whole space $V=L^2(\mathbb R^\nu)$. However, if we consider a smoothed version of the Laplacian given by $f:\mathbf\Phi\mapsto\Delta(\mathbf \Phi*g_\epsilon)$, where $g_\epsilon$ is the variance-$\epsilon$ Gaussian kernel, then this map will be well-defined on the whole $V$, norm-continuous and SE($\nu$)-equivariant. 

\item[Pointwise maps.] Consider a \emph{pointwise} map $f:V\to U$ defined by $f(\mathbf\Phi)(\mathbf x)=f_0(\mathbf x)$, where $f_0:\mathbb R^{d_V}\to\mathbb R^{d_U}$ is some map. In this case $f$ is SE($\nu$)-equivariant. Note that if $f_0(0)\ne 0$, then $f$ is not well-defined on $V=L^2(\mathbb R^\nu,\mathbb R^{d_V})$, since $f(\mathbf \Phi)\notin L^2(\mathbb R^\nu,\mathbb R^{d_U})$ for the trivial signal $\mathbf \Phi(\mathbf x)\equiv 0$. An easy-to-check sufficient condition for $f$ to be well-defined  and continuous on the whole $V$ is that $f_0(0)=0$ and $f_0$ be globally Lipschitz (i.e., $|f_0(\mathbf x)-f_0(\mathbf y)|\le c|\mathbf x-\mathbf y|$ for all $\mathbf x,\mathbf y\in \mathbb R^\nu$ and some $c<\infty$).  
\end{description}

Our goal in this section is to describe a finite computational model that would be a universal approximator for all continuous and $\mathrm{SE}(\nu)$--equivariant maps $f:V\to U$. Following the strategy of Section \ref{sec:convnet_deep}, we aim to define limit points of such finite models and then prove that the limit points are exactly the continuous and $\mathrm{SE}(\nu)$--equivariant maps.

We focus on approximating $L^2(\mathbb R^\nu,\mathbb R^{d_U})$-valued $\mathrm{SE}(\nu)$-equivariant maps rather than $\mathbb R^{d_U}$-valued $\mathrm{SE}(\nu)$-invariant maps because, as discussed in Section \ref{sec:translations}, we find it hard to reconcile the $\mathrm{SE}(\nu)$-invariance with pooling.   

Note that, as in the previous sections, there is a straightforward symmetrization-based approach to constructing universal $\mathrm{SE}(\nu)$--equivariant models. In particular, the group $\mathrm{SE}(\nu)$ extends the group of translations $\mathbb R^\nu$ by the compact group $\mathrm{SO}(\nu),$ and we can construct $\mathrm{SE}(\nu)$--equivariant maps simply by symmetrizing $\mathbb R^\nu$--equivariant maps over $\mathrm{SO}(\nu)$, as in Proposition \ref{th:sym_equiv}.

\begin{prop}
If a map $f_{\mathbb R^\nu}:V\to U$ is continuous and $\mathbb R^\nu$--equivariant, then the map $f_{\mathrm{SE}(\nu)}:V\to U$ defined by
$$f_{\mathrm{SE}(\nu)}(\mathbf\Phi)=\int_{\mathrm{SO}(\nu)}R_{(0,\theta)^{-1}}f_{\mathbb R^\nu}(R_{(0,\theta)}\mathbf\Phi)d\theta$$
is continuous and $\mathrm{SE}(\nu)$--equivariant. 
\end{prop}
\begin{proof} The continuity of $f_{\mathrm{SE}(\nu)}$ follows by elementary arguments using the continuity of $f_{\mathbb R^\nu}:V\to U$,   uniform boundedness of the operators $R_{(0,\theta)}$, and compactness of $\mathrm{SO}(\nu)$. The $\mathrm{SE}(\nu)$--equivariance follows since for any $(\gamma,\theta')\in\mathrm{SE}(\nu)$ and $\mathbf \Phi\in V$
\begin{align*}
f_{\mathrm{SE}(\nu)}(R_{(\gamma,\theta')}\mathbf\Phi) 
&= \int_{\mathrm{SO}(\nu)}R_{(0,\theta)^{-1}}f_{\mathbb R^\nu}(R_{(0,\theta)}R_{(\gamma,\theta')}\mathbf\Phi)d\theta\\
&= \int_{\mathrm{SO}(\nu)}R_{(0,\theta)^{-1}}f_{\mathbb R^\nu}(R_{(\theta\gamma,1)}R_{(0,\theta\theta')}\mathbf\Phi)d\theta\\
&= \int_{\mathrm{SO}(\nu)}R_{(0,\theta)^{-1}}R_{(\theta\gamma,1)}f_{\mathbb R^\nu}(R_{(0,\theta\theta')}\mathbf\Phi)d\theta\\
&= \int_{\mathrm{SO}(\nu)}R_{(\gamma,\theta')}R_{(0,\theta\theta')^{-1}}f_{\mathbb R^\nu}(R_{(0,\theta\theta')}\mathbf\Phi)d\theta\\
&= R_{(\gamma,\theta')}f_{\mathrm{SE}(\nu)}(\mathbf\Phi).
\end{align*}
\end{proof}
This proposition implies, in particular, that $\mathrm{SO}(\nu)$-symmetrizations of merely $\mathbb R^\nu$-equivariant basic convnets considered in Section \ref{sec:convnet_deep} can serve as universal $\mathrm{SE}(\nu)$--equivariant approximators. However, like in the previous sections, we will be instead interested in an intrinsically $\mathrm{SE}(\nu)$--equivariant network construction not involving explicit symmetrization of the approximation over the group $\mathrm{SO}(\nu)$. In particular, our approximators will not use rotated grids. 

Our construction relies heavily on the representation theory of the group $\mathrm{SO}(\nu),$ and in the present paper we restrict ourselves to the case $\nu=2$, in which the group $\mathrm{SO}(\nu)$ is abelian and the representation theory is much easier than in the general case. 

Section \ref{sec:charge_prelim} contains preliminary considerations suggesting the network construction appropriate for our purpose. The formal detailed description of the model is given in Section \ref{sec:chargeconvnet}. In Section \ref{sec:chargeresult} we formulate and prove the main result of the section, the $\mathrm{SE}(2)$--equivariant universal approximation property of the model.

\subsection{Preliminary considerations}\label{sec:charge_prelim}  
In this section we explain the idea behind our construction of the universal $\mathrm{SE}(2)$-equivariant convnet (to be formulated precisely in Section \ref{sec:chargeconvnet}). We start by showing in Section \ref{sec:chargepointwisecharact} that a $\mathrm{SE}(2)$-equivariant map $f:V\to U$ can be described using a $\mathrm{SO}(2)$-invariant map $f_{\mathrm{loc}}:V\to \mathbb R^{d_V}$. Then, relying on this observation, in Section \ref{sec:chargediff} we show that, heuristically, $f$ can be reconstructed by first equivariantly extracting local ``features'' from the original signal using equivariant differentiation, and then transforming these features using a $\mathrm{SO}(2)$-invariant pointwise map. In Section \ref{sec:chargediscr} we describe discretized differential operators and smoothing operators that we require in order to formulate our model as a finite computation model with sufficient regularity. Finally, in Section \ref{sec:chargepolyinv} we consider polynomial approximations on $\mathrm{SO}(2)$-modules.
  
\subsubsection{Pointwise characterization of $\mathrm{SE}(\nu)$--equivariant maps}\label{sec:chargepointwisecharact}

In this subsection we show that, roughly speaking, $\mathrm{SE}(\nu)$-equivariant maps $f:V\to U$ can be described in terms of $\mathrm{SO}(\nu)$-invariant maps $f:V\to\mathbb R^\nu$ obtained by observing the output signal at a fixed position. 

(The proposition below has one technical subtlety: we consider signal values $\mathbf \Phi(0)$ at a particular point $\mathbf x=0$ for generic signals $\mathbf \Phi$ from the space $L^2(\mathbb R^\nu,\mathbb R^{d_U})$. Elements of this spaces are  defined as equivalence classes of signals that can differ on sets of zero Lebesgue measure, so, strictly speaking, $\mathbf \Phi(0)$ is not well-defined. We can circumvent this difficulty by fixing a particular canonical representative of the equivalence class, say
$$\mathbf \Phi_{\mathrm{canon}}(\mathbf x)=\begin{cases}\lim_{\epsilon\to 0}\frac{1}{|B_\epsilon(\mathbf x)|}\int_{B_\epsilon(\mathbf x)}\mathbf \Phi(\mathbf y)d\mathbf y,&\text{if the limit exists,}\\ 0, &\text{otherwise.}\end{cases}$$ Lebesgue's differentiation theorem ensures that the limit exists and agrees with $\mathbf \Phi$ almost everywhere, so that $\mathbf \Phi_{\mathrm{canon}}$ is indeed a representative of the equivalence class. This choice of the representative is clearly $\mathrm{SE}(\nu)$-equivariant. In the proposition below, the signal value at $\mathbf x=0$ can be understood as the value of such a canonical representative.)

\begin{prop}\label{lm:chargecharequiv}
Let $f:L^2(\mathbb R^2,\mathbb R^{d_V})\to L^2(\mathbb R^2,\mathbb R^{d_U})$ be a $\mathbb R^\nu$--equivariant map. Then $f$ is $\mathrm{SE}(\nu)$--equivariant if and only if $f(R_{(0,\theta)}\mathbf \Phi)(0)=f(\mathbf \Phi)(0)$ for all $\theta\in \mathrm{SO}(\nu)$ and $\mathbf\Phi\in V$.   
\end{prop}
\begin{proof}
One direction of the statement is obvious: if $f$ is $\mathrm{SE}(\nu)$--equivariant, then $f(R_{(0,\theta)}\mathbf \Phi)(0)=R_{(0,\theta)}f(\mathbf \Phi)(0)=f(\mathbf \Phi)(\mathcal A_{(0,\theta^{-1})}0)=f(\mathbf \Phi)(0).$  

Let us prove the opposite implication, i.e. that $f(R_{(0,\theta)}\mathbf \Phi)(0)\equiv f(\mathbf \Phi)(0)$ implies the $\mathrm{SE}(\nu)$--equivariance. We need to show that for all $(\gamma,\theta)\in \mathrm{SE}(\nu)$, $\mathbf\Phi\in V$ and $\mathbf x\in \mathbb R^\nu$ we have
$$f(R_{(\gamma,\theta)}\mathbf \Phi)(\mathbf x)=R_{(\gamma,\theta)}f(\mathbf \Phi)(\mathbf x).$$
Indeed,
\begin{align*}
f(R_{(\gamma,\theta)}\mathbf \Phi)(\mathbf x) 
&=R_{(-\mathbf x,1)}f(R_{(\gamma,\theta)}\mathbf \Phi)( 0)\\
&=f(R_{(-\mathbf x,1)}R_{(\gamma,\theta)}\mathbf \Phi)( 0)\\
&=f(R_{(0,\theta)}R_{(\theta^{-1}(\gamma-\mathbf x),1)}\mathbf \Phi)( 0)\\
&=f(R_{(\theta^{-1}(\gamma-\mathbf x),1)}\mathbf \Phi)( 0)\\
&=R_{(\theta^{-1}(\gamma-\mathbf x),1)}f(\mathbf \Phi)(0)\\
&=R_{(\mathbf x,\theta)}R_{(\theta^{-1}(\gamma-\mathbf x),1)}f(\mathbf \Phi)(\mathcal A_{(\mathbf x,\theta)} 0)\\
&=R_{(\gamma,\theta)}f(\mathbf \Phi)(\mathbf x),
\end{align*}
where we used definition \eqref{eq:chargergta} (steps 1 and 6), the $\mathbb R^\nu$--equivariance of $f$ (steps 2 and 5), and the hypothesis of the lemma (step 4). 
\end{proof}

Now, if $f:V\to U$ is an $\mathrm{SE}(\nu)$--equivariant map, then we can define the $\mathrm{SO}(\nu)$--invariant map $f_{\mathrm{loc}}:V\to\mathbb R^{d_U}$  by 
\begin{equation}\label{eq:chargefloc}
f_{\mathrm{loc}}(\mathbf\Phi)=f(\mathbf\Phi)(0).\end{equation}
Conversely, suppose that $f_{\mathrm{loc}}:V\to\mathbb R^{d_U}$ is an $\mathrm{SO}(\nu)$--invariant map. Consider the map $f:V\to \{\mathbf \Psi: \mathbb R^\nu\to\mathbb R^{d_U}\}$ defined by
\begin{equation}\label{eq:chargefphifloc}
f(\mathbf \Phi)(\mathbf x):=f_{\mathrm{loc}}(R_{(-\mathbf x, 1)}\mathbf\Phi).
\end{equation}
In general, $f(\mathbf \Phi)$ need not be in $L^2(\mathbb R^\nu, \mathbb R^{d_U}).$ Suppose, however, that this is the case for all $\mathbf \Phi\in V.$ Then $f$ is clearly $\mathbb R^\nu$-equivariant and, moreover,  $\mathrm{SE}(\nu)$--equivariant, by the above proposition. 

Thus, under some additional regularity assumption, the task of reconstructing $\mathrm{SE}(\nu)$--equivariant maps $f:V\to U$ is equivalent to the task of reconstructing $\mathrm{SO}(\nu)$--invariant maps $f_{\mathrm{loc}}:V\to\mathbb R^{d_U}$.

\subsubsection{Equivariant differentiation}\label{sec:chargediff} 
It is convenient to describe rigid motions of $\mathbb R^2$ by identifying this two-dimensional real space with the one-dimensional complex space $\mathbb C$. Then an element of $\mathrm{SE}(2)$ can be written as $(\gamma,\theta)=(x+iy,e^{i\phi})$ with some $x,y\in\mathbb R$ and $\phi\in[0,2\pi)$. The action of $\mathrm{SE}(2)$ on $\mathbb R^2\cong \mathbb C$ can be written as
$$\mathcal A_{(x+iy,e^{i\phi})}z=x+iy+e^{i\phi}z,\quad z\in\mathbb C.$$
Using analogous notation  $R_{(x+iy,e^{i\phi})}$ for the canonically associated representation of $\mathrm{SE}(2)$ in $V$ defined in \eqref{eq:chargergta}, consider the generators of this representation: 
$$J_x=i\lim_{\delta x\to 0}\frac{R_{(\delta x, 1)}-1}{\delta x},\quad 
J_y=i\lim_{\delta y\to 0}\frac{R_{(i\delta y, 1)}-1}{\delta y},\quad 
J_\phi=i\lim_{\delta \phi\to 0}\frac{R_{(0, e^{i\delta\phi})}-1}{\delta \phi}.
$$
The generators can be explicitly written as
$$J_x=-i\partial_x, \quad J_y=-i\partial_y,\quad J_\phi = -i\partial_\phi=-i(x\partial_y-y\partial_x)$$
and obey the commutation relations
\begin{equation}\label{eq:chargecommrel}[J_x,J_y]=0,\quad [J_x,J_\phi]=-iJ_y, \quad [J_y,J_\phi]=iJ_x.\end{equation}
We are interested in local transformations of signals $\mathbf\Phi\in V$, so it is natural to consider the action of differential operators on the signals. We would like, however, to ensure the equivariance of this action. This can be done as follows. Consider the first-order operators
$$\partial_z=\frac{1}{2}(\partial_x-i\partial_y), \quad \partial_{\overline{z}}=\frac{1}{2}(\partial_x+i\partial_y).$$
These operators commute with $J_x,J_y$, and have the following commutation relations with $J_\phi$:
$$[\partial_z,J_\phi]=\partial_z,\quad [\partial_{\overline{z}},J_\phi]=-\partial_{\overline{z}}$$
or, equivalently,
\begin{equation}\label{eq:chargepzjp}\partial_z J_\phi=(J_\phi+1)\partial_z,\quad\partial_{\overline{z}}J_\phi= (J_\phi-1)\partial_{\overline{z}}.\end{equation}
Let us define, for any $\mu\in\mathbb Z,$ $$J_\phi^{(\mu)}=J_\phi+\mu=\mu-i\partial_\phi.$$
Then the triple $(J_x,J_y,J_z^{(\mu)})$ obeys the same commutation relations \eqref{eq:chargecommrel}, i.e., constitutes another representation of the Lie algebra of the group $\mathrm{SE}(2)$. The corresponding representation of the group differs from the original representation \eqref{eq:chargergta} by the extra phase factor: 
\begin{equation}\label{eq:chargermugta}R^{(\mu)}_{(x+iy,e^{i\phi})}\mathbf\Phi(\mathbf x)=e^{-i\mu\phi}\mathbf\Phi(\mathcal A_{(x+iy,e^{i\phi})^{-1}}\mathbf x).\end{equation}
The identities \eqref{eq:chargepzjp} imply $\partial_z J^{(\mu)}_\phi=J^{(\mu+1)}_\phi\partial_z$ and $\partial_{\overline{z}} J^{(\mu)}_\phi=J^{(\mu-1)}_\phi\partial_{\overline{z}}.$ Since the operators $\partial_{z},\partial_{\overline{z}}$ also commute with $J_x,J_y$, we see that the operators $\partial_{z},\partial_{\overline{z}}$ can serve as \emph{ladder operators} equivariantly mapping 
\begin{equation}\label{eq:chargedzvmu}\partial_{z}:V_\mu\to V_{\mu+1},\quad \partial_{\overline{z}}:V_\mu\to V_{\mu-1},\end{equation}
where $V_\mu$ is the space $L^2(\mathbb R^2,\mathbb R^{d_V})$ equipped with the representation \eqref{eq:chargermugta}. Thus, we can equivariantly differentiate signals as long as we appropriately switch the representation. In the sequel, we will for brevity refer to the parameter $\mu$ characterizing the representation as its \emph{global charge}. 

It is convenient to also consider another kind of charge, associated with angular dependence of the signal with respect to rotations about fixed points; let us call it \emph{local charge} $\eta$ in contrast to the above global charge $\mu$. Namely, for any fixed $\mathbf x_0\in \mathbb R^2$, decompose the module $V_\mu$ as 
\begin{equation}\label{eq:chargevmuoplus}
V_\mu=\bigoplus_{\eta\in\mathbb Z}V_{\mu,\eta}^{(\mathbf x_0)},
\end{equation}
where 
\begin{equation}\label{eq:chargevmueta}
V_{\mu,\eta}^{(\mathbf x_0)}=R_{(\mathbf x_0,1)}V_{\mu,\eta}^{(0)},
\end{equation}
and
\begin{equation}\label{eq:chargevmueta0}
V_{\mu,\eta}^{(0)}=\{\mathbf\Phi\in V_\mu|\mathbf\Phi(\mathcal A_{(0,e^{i\phi})^{-1}}\mathbf x)=e^{-i\eta\phi}\mathbf\Phi(\mathbf x)\; \forall \phi \}.\end{equation}
Writing $\mathbf x_0=(x_0,y_0),$ we can characterize $V_{\mu,\eta}^{(\mathbf x_0)}$ as the eigenspace of the operator $$J_\phi^{(\mathbf x_0)}:=R_{(\mathbf x_0,1)}J_\phi R_{(\mathbf x_0,1)^{-1}}= -i(x-x_0)\partial_y+i(y-y_0)\partial_x$$ corresponding to the eigenvalue $\eta.$  
The operator $J_\phi^{(\mathbf x_0)}$ has the same commutation relations with $\partial_z, \partial_{\overline z}$ as $J_\phi$: 
$$[\partial_z,J_\phi^{(\mathbf x_0)}]=\partial_z,\quad [\partial_{\overline{z}},J_\phi^{(\mathbf x_0)}]=-\partial_{\overline{z}}.$$
We can then describe the structure of equivariant maps \eqref{eq:chargedzvmu} with respect to decomposition \eqref{eq:chargevmuoplus} as follows: for any $\mathbf x_0$, the decrease or increase of the global charge by the respective ladder operator is compensated by the opposite effect of this operator on the local charge, i.e. $\partial_z$ maps $V_{\mu,\eta}^{(\mathbf x_0)}$ to $V_{\mu+1,\eta-1}^{(\mathbf x_0)}$ while $\partial_{\overline z}$ maps $V_{\mu,\eta}^{(\mathbf x_0)}$ to $V_{\mu-1,\eta+1}^{(\mathbf x_0)}$:
\begin{equation}\label{eq:chargeladder}\partial_z V_{\mu,\eta}^{(\mathbf x_0)}\to V_{\mu+1,\eta-1}^{(\mathbf x_0)}, \quad \partial_{\overline z}: V_{\mu,\eta}^{(\mathbf x_0)}\to V_{\mu-1,\eta+1}^{(\mathbf x_0)}.
\end{equation}
We interpret these identities as \emph{conservation of the total charge}, $\mu+\eta$. We remark that there is some similarity between our total charge and the total angular momentum in quantum mechanics; the total angular momentum there consists of the spin component and the orbital component that are analogous to our global and local charge, respectively.

Now we give a heuristic argument showing how to express an arbitrary equivariant map $f:V\to U$ using our equivariant differentiation. As discussed in the previous subsection, the task of expressing $f$ reduces to expressing $f_{\mathrm{loc}}$ using formulas \eqref{eq:chargefloc},\eqref{eq:chargefphifloc}. Let a signal $\mathbf \Phi$ be analytic as a function of the real variables $x,y$, then it can be Taylor expanded as
\begin{equation}\label{eq:chargetaylor}\mathbf\Phi=\sum_{a,b=0}^\infty \frac{1}{a!b!}\partial_z^a\partial_{\overline z}^b\mathbf \Phi(0)\mathbf \Phi_{a,b},\end{equation}
with the basis signals $\mathbf\Phi_{a,b}$ given by $$\mathbf\Phi_{a,b}(z)=z^a\overline{z}^b.$$ The signal $\mathbf \Phi$ is fully determined by the coefficients $\partial_z^a\partial_{\overline z}^b\mathbf \Phi(0),$ so the map $f_{\mathrm{loc}}$ can be expressed as a function of these coefficients: 
\begin{equation}\label{eq:chargefloctilde}
f_{\mathrm{loc}}(\mathbf \Phi)=\widetilde f_{\mathrm{loc}}\big((\partial_z^a\partial_{\overline z}^b\mathbf \Phi(0))_{a,b=0}^\infty\big).\end{equation} At $\mathbf x_0=0,$ the signals $\mathbf\Phi_{a,b}$ have local charge $\eta=a-b,$ and, if viewed as elements of $V_{\mu=0}$, transform under rotations by $$R_{(0,e^{i\phi})}\mathbf\Phi_{a,b}=e^{-i(a-b)\phi}\mathbf\Phi_{a,b}.$$   
Accordingly, if we write $\mathbf\Phi$ in the form $\mathbf\Phi=\sum_{a,b}c_{a,b}\mathbf\Phi_{a,b},$ then 
$$R_{(0,e^{i\phi})}\mathbf\Phi=\sum_{a,b}e^{-i(a-b)\phi}c_{a,b}\mathbf\Phi_{a,b}.$$ 
It follows that the SO(2)-invariance of $f_{\mathrm{loc}}$ is equivalent to $\widetilde f_{\mathrm{loc}}$ being invariant with respect to simultaneous multiplication of the arguments by the factors $e^{-i(a-b)\phi}$:
$$\widetilde f_{\mathrm{loc}}\big((e^{-i(a-b)\phi}c_{a,b})_{a,b=0}^\infty\big)=\widetilde f_{\mathrm{loc}}\big((c_{a,b})_{a,b=0}^\infty\big)\quad\forall\phi.$$
Having determined the invariant map $\widetilde f_{\mathrm{loc}}$, we can express the value of $f(\mathbf\Phi)$ at an arbitrary point $\mathbf x\in\mathbb R^2$ by
\begin{equation}\label{eq:chargeffloc}
f(\mathbf\Phi)(\mathbf x)=\widetilde f_{\mathrm{loc}}\big((\partial_z^a\partial_{\overline z}^b\mathbf \Phi(\mathbf x))_{a,b=0}^\infty\big).
\end{equation}
Thus, the map $f$ can be expressed, at least heuristically, by first computing various derivatives of the signal and then applying to them the invariant map $\widetilde f_{\mathrm{loc}}$, independently at each $\mathbf x\in \mathbb R^2$. 

The expression \eqref{eq:chargeffloc} has the following interpretation in terms of information flow and the two different kinds of charges introduced above. Given an input signal $\mathbf\Phi\in V$ and $\mathbf x\in \mathbb R^2$, the signal has global charge $\mu=0$, but, in general, contains multiple components having different values of the local charge $\eta$ with respect to $\mathbf x$, according to the decomposition $V=V_{\mu=0}=\oplus_{\eta\in\mathbb Z}V_{0,\eta}^{(\mathbf x)}.$ By \eqref{eq:chargeladder}, a differential operator $\partial_z^a\partial_{\overline z}^b$ maps the space $V_{0,\eta}^{(\mathbf x)}$ to the space $V_{a-b,\eta+b-a}^{(\mathbf x)}.$ However, if a signal $\mathbf\Psi\in V_{a-b,\eta+b-a}^{(\mathbf x)}$ is continuous at $\mathbf x$, then $\mathbf\Psi$ must vanish there unless $\eta+b-a=0$ (see the definition \eqref{eq:chargevmueta},\eqref{eq:chargevmueta0}), i.e., only information from the $V_{0,\eta}^{(\mathbf x)}$--component of $\mathbf\Phi$ with $\eta=a-b$ is observed in $\partial_z^a\partial_{\overline z}^b\mathbf\Phi(\mathbf x)$. Thus, at each point $\mathbf x$, the differential operator $\partial_z^a\partial_{\overline z}^b$ can be said to transform information contained in $\mathbf\Phi$ and associated with global charge $\mu=0$ and local charge $\eta=a-b$ into  information associated with global charge $\mu=a-b$ and local charge $\eta=0$. This transformation is useful to us because the local charge only reflects the structure of the input signal, while the global charge is a part of the architecture of the computational model and can be used to directly control the information flow. The operators $\partial_z^a\partial_{\overline z}^b$ deliver to the point $\mathbf x$ information about the signal values away from this point -- similarly to how this is done by local convolutions in the convnets of Section \ref{sec:translations} -- but now this information flow is equivariant with respect to the action of $\mathrm{SO}(2)$. 

By \eqref{eq:chargeffloc}, the SE(2)--equivariant map $f$ can be heuristically decomposed into the family of SE(2)--equivariant differentiations producing ``local features'' $\partial_z^a\partial_{\overline z}^b\mathbf \Phi(\mathbf x)$ and followed by the SO(2)--invariant map $\widetilde f_{\mathrm{loc}}$ acting independently at each $\mathbf x$. In the sequel, we use this decomposition as a general strategy in our construction of the finite convnet-like approximation model in Section \ref{sec:chargeconvnet} -- the ``\emph{charge--conserving convnet}'' -- and in the proof of its universality in Section \ref{sec:chargeresult}.

The Taylor expansion \eqref{eq:chargetaylor} is  not rigorously applicable to generic signals $\mathbf\Phi\in L^2(\mathbb R^{\nu}, \mathbb R^{d_V})$. Therefore, we will add smoothing in our convnet-like model, to be performed before the differentiation operations. This will be discussed below in Section \ref{sec:chargediscr}. Also, we will discuss there the discretization of the differential operators, in order to formulate the charge--conserving convnet as a finite computational model.

The invariant map $\widetilde f_{\mathrm{loc}}$ can be approximated using invariant polynomials, as 
we discuss  in Section \ref{sec:chargepolyinv} below. 
As discussed earlier in Section \ref{sec:compact}, invariant polynomials can be produced from a set of generating polynomials; however,  in the present setting this set is rather large and grows rapidly as charge is increased, so it will be more efficient to just generate new invariant polynomials by multiplying general polynomials of lower degree subject to charge conservation. 
As a result, we will approximate the map $\widetilde f_{\mathrm{loc}}$  by a series of multiplication layers in the charge-conserving convnet.

\subsubsection{Discretized differential operators}\label{sec:chargediscr}
Like in Section \ref{sec:translations}, we aim to formulate the approximation model as a computation which is fully finite except for the initial discretization of the input signal. Therefore we need to discretize the equivariant differential operators considered in Section \ref{sec:chargediff}.   
Given a discretized signal   $\mathbf \Phi: (\lambda \mathbb Z)^2\to \mathbb R^{d_V}$ on the grid of spacing $\lambda$, and writing grid points as ${\gamma}=(\lambda\gamma_x,\lambda\gamma_y)\in (\lambda\mathbb Z)^2$, we define the discrete derivatives $\partial_{z}^{(\lambda)}, \partial_{\overline{z}}^{(\lambda)}$ by 
\begin{align}\label{eq:dzlam}
\partial_{z}^{(\lambda)}\mathbf\Phi(\lambda \gamma_x, \lambda \gamma_y) = & \frac{1}{4\lambda}\bigg( \mathbf\Phi\big(\lambda (\gamma_x+1), \lambda \gamma_y\big)-\mathbf\Phi\big(\lambda (\gamma_x-1), \lambda \gamma_y\big)\\
& -i\Big(\mathbf\Phi\big(\lambda \gamma_x, \lambda (\gamma_y+1)\big)-\mathbf\Phi\big(\lambda \gamma_x, \lambda (\gamma_y-1)\big)\Big)\bigg),\nonumber\\ \label{eq:dzovlam}
\partial_{\overline{z}}^{(\lambda)}\mathbf\Phi(\lambda \gamma_x, \lambda \gamma_y) = & \frac{1}{4\lambda}\bigg( \mathbf\Phi\big(\lambda (\gamma_x+1), \lambda \gamma_y\big)-\mathbf\Phi\big(\lambda (\gamma_x-1), \lambda \gamma_y\big)\\
& +i\Big(\mathbf\Phi\big(\lambda \gamma_x, \lambda (\gamma_y+1)\big)-\mathbf\Phi\big(\lambda \gamma_x, \lambda (\gamma_y-1)\big)\Big)\bigg).\nonumber
\end{align}
Since general signals $\mathbf\Phi\in L^2(\mathbb R^\nu,\mathbb R^{d_V})$ are not differentiable, we will smoothen them prior to differentiating. Smoothing will also be a part of the computational model and can be implemented by local operations as follows. Consider the discrete Laplacian $\Delta^{(\lambda)}$ defined by 
\begin{align}\label{eq:laplacelam}
\Delta^{(\lambda)}\mathbf\Phi(\lambda \gamma_x, \lambda \gamma_y) = & \frac{1}{\lambda^2}\Big( \mathbf\Phi\big(\lambda (\gamma_x+1), \lambda \gamma_y\big)+\mathbf\Phi\big(\lambda (\gamma_x-1), \lambda \gamma_y\big)\\
& +\mathbf\Phi\big(\lambda \gamma_x, \lambda (\gamma_y+1)\big)+\mathbf\Phi\big(\lambda \gamma_x, \lambda (\gamma_y-1)\big)-4\mathbf\Phi(\lambda \gamma_x, \lambda \gamma_y)\Big).\nonumber
\end{align}
Then, a single smoothing layer can be implemented by the positive definite operator $1+\tfrac{\lambda^2}{8}\Delta^{(\lambda)}:$
\begin{align}\label{eq:charge1plusl2delta}
\Big(1+\frac{\lambda^2}{8}\Delta^{(\lambda)}\Big)\mathbf\Phi(\lambda \gamma_x, \lambda \gamma_y) = & \frac{1}{8}\Big( \mathbf\Phi\big(\lambda (\gamma_x+1), \lambda \gamma_y\big)+\mathbf\Phi\big(\lambda (\gamma_x-1), \lambda \gamma_y\big)\nonumber\\
& +\mathbf\Phi\big(\lambda \gamma_x, \lambda (\gamma_y+1)\big)+\mathbf\Phi\big(\lambda \gamma_x, \lambda (\gamma_y-1)\big)+4\mathbf\Phi(\lambda \gamma_x, \lambda \gamma_y)\Big).
\end{align}
We will then replace the differential operators $\partial_z^a\partial_{\overline z}^b$ used in the heuristic argument in Section \ref{sec:chargediff} by the discrete operators \begin{equation}\label{eq:chargedefltau}
\mathcal L_\lambda^{(a,b)}=(\partial_{z}^{(\lambda)})^a(\partial_{\overline{z}}^{(\lambda)})^b \Big(1+\frac{\lambda^2}{8}\Delta^{(\lambda)}\Big)^{\lceil 4/\lambda^2\rceil} P_\lambda.\end{equation}
Here $P_\lambda$ is the discretization projector \eqref{eq:ptau}. The power $\lceil 4/\lambda^2\rceil$ (i.e., the number of smoothing layers) scales with $\lambda$ so that in the continuum limit $\lambda\to 0$ the operators $\mathcal L_\lambda^{(a,b)}$ converge to convolution operators. Specifically, consider the function $\Psi_{a,b}:\mathbb R^2\to\mathbb R$: \begin{equation}\label{eq:chargepsiab}\Psi_{a,b}=\partial_z^a\partial_{\overline{z}}^b\Big(\frac{1}{2\pi}e^{-|\mathbf x|^2/2}\Big),\end{equation} where we identify $|\mathbf x|^2\equiv z\overline{z}$. Define the operator $\mathcal L_0^{(a,b)}$  by $\mathcal L_0^{(a,b)}\mathbf\Phi=\mathbf\Phi*\Psi_{a,b},$ i.e.
\begin{equation}\label{eq:chargel0ab}
\mathcal L_0^{(a,b)}\mathbf\Phi(\mathbf x) = \int_{\mathbb R^2}\mathbf\Phi(\mathbf x-\mathbf y)\Psi_{a,b}(\mathbf y)d^2\mathbf y.
\end{equation}
Then we have the following lemma proved in Appendix \ref{sec:clt}.

\begin{lemma}\label{lm:chargeclt} Let $a,b$ be fixed nonnegative integers. For all $\lambda\in [0,1]$, consider the linear operators $\mathcal L_\lambda^{(a,b)}$ as operators from $L^2(\mathbb R^2,\mathbb R^{d_V})$ to $L^\infty(\mathbb R^2,\mathbb R^{d_V})$. Then:
\begin{enumerate}
\item The operators $\mathcal L_\lambda^{(a,b)}$ are bounded uniformly in $\lambda$;
\item As $\lambda\to 0,$ the operators $\mathcal L_\lambda^{(a,b)}$ converge strongly to the operator $\mathcal L_0^{(a,b)}$. Moreover, this convergence is uniform on compact sets $K\subset V$ (i.e., $\lim_{\lambda\to 0}\sup_{\mathbf\Phi\in K}\|\mathcal L_\lambda^{(a,b)}\mathbf\Phi-\mathcal L_0^{(a,b)}\mathbf\Phi\|_\infty= 0$).
\end{enumerate}
\end{lemma}

This lemma is essentially just a slight modification of Central Limit Theorem. It will be convenient to consider $L^\infty$ rather than $L^2$ in the target space because of the pointwise polynomial action of the layers following the smoothing and differentiation layers.

\subsubsection{Polynomial approximations on $\mathrm{SO}(2)$-modules}\label{sec:chargepolyinv}
Our derivation of the approximating model in Section \ref{sec:chargediff} was based on identifying the $\mathrm{SO}(2)$-invariant map $f_{\mathrm{loc}}$ introduced in \eqref{eq:chargefloc} and expressing it via $\widetilde f_{\mathrm{loc}}$ by Eq.\eqref{eq:chargefloctilde}. It is convenient to approximate the map $\widetilde f_{\mathrm{loc}}$ by invariant polynomials on appropriate $\mathrm{SO}(2)$-modules, and in this section we state several general facts relevant for this purpose.

First, the following lemma is obtained immediately using symmetrization and the Weierstrass theorem (see e.g. the proof of Proposition \eqref{eq:polyinvansatz}).
\begin{lemma}\label{lm:chargeso2poly} Let $f:W\to \mathbb R$ be a continuous $\mathrm{SO}(2)$-invariant map on a real finite-dimensional $\mathrm{SO}(2)$-module $W$. Then $f$ can be approximated by polynomial invariants on $W$.
\end{lemma}
We therefore focus on constructing general polynomial invariants on $\mathrm{SO}(2)$-modules. This can be done in several ways; we will describe just one particular construction performed in a ``layerwise'' fashion resembling convnet layers.

It is convenient to first consider the case of $\mathrm{SO}(2)$-modules over the field $\mathbb C$, since the representation theory of the group $\mathrm{SO}(2)$ is especially easily described when the underlying field is $\mathbb C$. Let us identify elements of $\mathrm{SO}(2)$ with the unit complex numbers $e^{i\phi}$. Then all complex irreducible representations of $\mathrm{SO}(2)$ are one-dimensional characters indexed by the number $\xi\in\mathbb Z$:
\begin{equation}\label{eq:charge1drep}
R_{e^{i\phi}}\mathbf x=e^{i\xi\phi}\mathbf x.\end{equation}
The representation $R$ induces the dual representation acting on functions $f(\mathbf x)$: 
$$R^*_{e^{i\phi}}f(\mathbf x)=f(R_{e^{-i\phi}}\mathbf x).$$
In particular, if $z_\xi$ is the variable associated with the one-dimensional space where representation \eqref{eq:charge1drep} acts, then it is transformed by the dual representation as  
$$R^*_{e^{i\phi}}z_\xi=e^{-i\xi\phi}z_\xi.$$
Now let $W$ be a general finite--dimensional $\mathrm{SO}(2)$--module over $\mathbb C.$ Then $W$ can be decomposed as   
\begin{equation}\label{eq:chargewc}
W=\bigoplus_{\xi} W_\xi,
\end{equation}
where $W_\xi\cong \mathbb C^{d_{\xi}}$ is the isotypic component of the representation \eqref{eq:charge1drep}. Let $z_{\xi k}, k=1,\ldots,d_\xi,$ denote the variables associated with the subspace $W_\xi$. If $f$ is a polynomial on $W$, we can write it as a linear combination of monomials:
\begin{equation}\label{eq:chargepolyexp}
f=\sum_{\mathbf{a}=(a_{\xi k})}c_{\mathbf a}\prod_{\xi,k}z_{\xi k}^{a_{\xi k}}.
\end{equation}
Then the dual representation acts on $f$ by
$$R^*_{e^{i\phi}}f=\sum_{\mathbf{a}=(a_{\xi k})}e^{-i\sum_{\xi,k}\xi a_{\xi k}\phi}c_{\mathbf a}\prod_{\xi,k}z_{\xi k}^{a_{\xi k}}.$$
We see that a polynomial is invariant iff it consists of invariant monomials, and a monomial is invariant iff $\sum_{\xi,k}\xi a_{\xi k}=0$.

We can generate an arbitrary $\mathrm{SO}(2)$-invariant polynomial on $W$ in the following ``layer-wise'' fashion. Suppose that $\{f_{t-1,\xi,n}\}_{n=1}^{N_{t-1}}$ is a collection of polynomials generated after $t-1$ layers so that 
\begin{equation}\label{eq:chargerftmn}
R^*_{e^{i\phi}}f_{t-1,\xi,n}=e^{-i\xi \phi}f_{t-1,\xi,n}
\end{equation}
for all $\xi,n$. Consider new polynomials $\{f_{t,\xi,n}\}_{n=1}^{N_{t}}$ obtained from $\{f_{t-1,\xi,n}\}_{n=1}^{N_{t-1}}$ by applying the second degree expressions
\begin{equation}\label{eq:chargeftmun}f_{t,\xi,n}=w^{(t)}_{0,n}\mathbf{1}_{\xi=0}+\sum_{n_1=1}^{N_{t-1}}
w_{1, \xi,n,n_1}^{(t)}f_{t-1,\xi,n_1}+\sum_{{\xi_1+\xi_2=\xi} }\sum_{n_1=1}^{{N_{t-1}}}\sum_{n_2=1}^{N_{t-1}}
w_{2, \xi_1,\xi_2,n,n_1,n_2}^{(t)}f_{t-1,\xi_1,n_1}f_{t-1,\xi_2,n_2}
\end{equation}
with some (complex) coefficients $w^{(t)}_{0,n}, w_{1, \xi,n,n_1}^{(t)}, w_{2, \xi_1,\xi_2,n,n_1,n_2}^{(t)}$. The first term is present only for $\xi=0$. The third term includes the ``charge conservation'' constraint $\xi=\xi_1+\xi_2$. 
It is clear that ones condition \eqref{eq:chargerftmn} holds for $\{f_{t-1,\xi,n}\}_{n=1}^{N_{t-1}}$, it also holds for $\{f_{t,\xi,n}\}_{n=1}^{N_{t}}$. 

On the other hand, suppose that the initial set $\{f_{1,\xi,n}\}_{n=1}^{N_1}$ includes all variables $z_{\xi k}$. Then for any invariant polynomial $f$ on $W$, we can arrange the parameters $N_t$ and the coefficients in Eq.\eqref{eq:chargeftmun} so that at some $t$ we obtain $f_{t,\xi=0,1}=f$. Indeed, first note that thanks to the second term in Eq.\eqref{eq:chargeftmun} it suffices to show this for the case when $f$ is an invariant monomial (since any invariant polynomial is a linear combination of invariant monomials, and  the second term allows us to form and pass forward such linear combinations). If $f$ is a constant, then it can be produced using the first term in Eq.\eqref{eq:chargeftmun}. If $f$ is a monomial of a positive degree, then it can be produced by multiplying lower degree monomials, which is afforded by the third term in Eq.\eqref{eq:chargeftmun}.

Now we discuss the case of the underlying field $\mathbb R$. In this case, apart from the trivial one-dimensional representation, all irreducible representations of $\mathrm{SO}(2)$ are two-dimensional and indexed by $\xi=1,2,\ldots$: 
\begin{equation}\label{eq:chargereiphixy}
R_{e^{i\phi}}\begin{pmatrix}x\\y\end{pmatrix}=
\begin{pmatrix}\cos \xi\phi&\sin \xi\phi\\-\sin \xi\phi&\cos \xi\phi\end{pmatrix}
\begin{pmatrix}x\\y\end{pmatrix}.\end{equation}
It is convenient to diagonalize such a representation, turning it into a pair of complex conjugate one-dimensional representations:
\begin{equation}\label{eq:chargereiphizz}R_{e^{i\phi}}\begin{pmatrix}z\\ \overline{z}\end{pmatrix}=
\begin{pmatrix}e^{-i\xi\phi}&0\\0&e^{i\xi\phi}\end{pmatrix}
\begin{pmatrix}z\\ \overline{z}\end{pmatrix},\end{equation}
where $$z = x+iy, \quad \overline{z}=x-iy.$$

More generally, any real $\mathrm{SO}(2)$--module $W$ can be decomposed exactly as in \eqref{eq:chargewc} into isotypic components $W_\xi$ associated with complex characters, but with the additional constraints \begin{equation}\label{eq:chargewnconj}W_\xi=\overline{W_{-\xi}},\end{equation}
meaning that $d_\xi=d_{-\xi}$ and $$W_\xi=W_{\pm \xi, \mathrm{Re}}+i W_{\pm \xi, \mathrm{Im}}, \quad W_{-\xi}=W_{\pm \xi, \mathrm{Re}}-i W_{\pm \xi, \mathrm{Im}}, \quad (\xi\ne 0)$$
with some real $d_\xi$--dimensional spaces $W_{\pm \xi, \mathrm{Re}},W_{\pm \xi, \mathrm{Im}}$. 

Any polynomial on $W$ can then be written in terms of real variables $z_{0,k}$ corresponding to $\xi=0$ and complex variables \begin{equation}\label{eq:chargezmuk}z_{\xi, k}=x_{\xi k}+iy_{\xi k}, \quad z_{-\xi,k}=x_{\xi k}-iy_{\xi k} \quad (\xi> 0)\end{equation} constrained by the relations $$z_{\xi,k}=\overline{z_{-\xi,k}}.$$
Suppose that a polynomial $f$ on $W$ is expanded over monomials in $z_{\xi,k}$ as in Eq.\eqref{eq:chargepolyexp}. This expansion is unique (the coefficients are given by $$c_{\mathbf a}=\Big(\prod_{\xi,k}\frac{\partial_{z_{\xi,k}}^{a_{\xi,k}}}{a_{\xi,k}!}\Big)f(0),$$ where $\partial _{z_{\xi,k}}=\frac{1}{2}(\partial_{x_{\xi k}}-i\partial_{y_{\xi k}})$  for $\xi> 0$ and $\partial _{z_{\xi,k}}=\frac{1}{2}(\partial_{x_{-\xi, k}}+i\partial_{y_{-\xi, k}})$  for $\xi< 0$). This implies that the condition for the polynomial $f$ to be invariant on $W$ is the same as in the previously considered complex case: the polynomial must consist of invariant monomials, and a monomial is invariant iff  $\sum_{\xi,k}\xi a_{\xi k}=0$.

Therefore, in the case of real $\mathrm{SO}(2)$-modules, any invariant polynomial can be generated using the same procedure described earlier for the complex case, i.e., by taking the complex extension of the module and iteratively generating  (complex) polynomials $\{f_{t,\xi,n}\}_{n=1}^{N_{t}}$ using Eq.\eqref{eq:chargeftmun}. The real part of a complex invariant polynomial on a real module is a real invariant polynomial. Thus, to ensure that in the case of real modules $W$ the procedure produces all real invariant polynomials, and only such polynomials, we can just add taking the real part of $f_{t,\xi=0,1}$ at the last step of the procedure.

\subsection{Charge-conserving convnet}\label{sec:chargeconvnet}

\begin{figure}
\centering
\includegraphics[width=\linewidth,clip,trim=35mm 55mm 72mm 70mm]{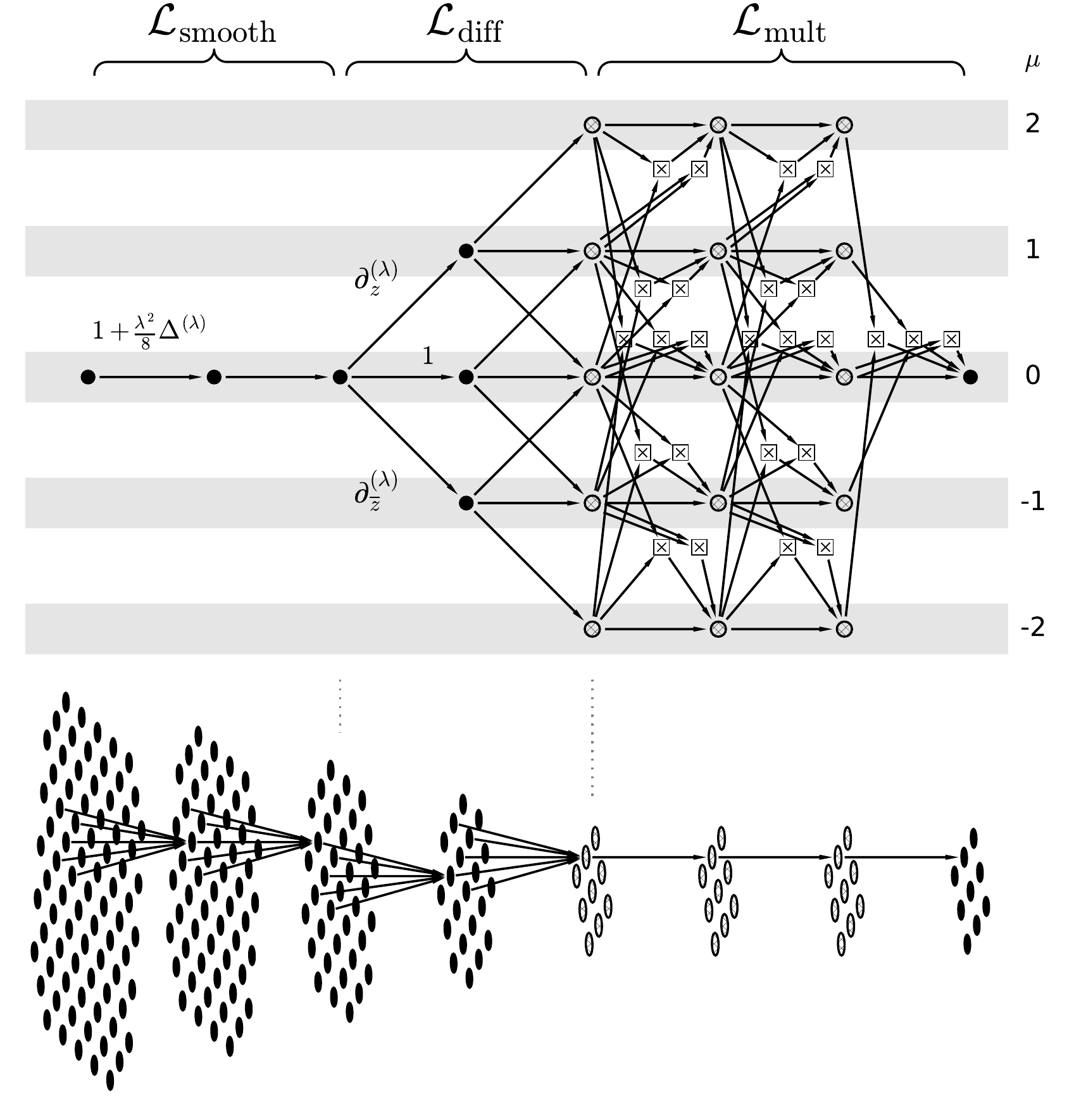}
\caption{Architecture of the charge-conserving convnet. The top figure shows the information flow in the fixed-charge subspaces of the feature space, while the bottom figure shows the same flow in the spatial coordinates. The smoothing layers only act on spatial dimensions, the multiplication layers only on feature dimensions, and the differentiation layers both on spatial and feature dimensions. Operation of smoothing and differentiation layers only involves nearest neighbors while in the multiplication layers the transitions are constrained by the requirement of charge conservation. The smoothing and differentiation layers are linear; the multiplication layers are not. The last multiplication layer only has zero-charge ($\mathrm{SO}(2)$--invariant) output. }
\centering
\label{fig:arch}
\end{figure}

We can now describe precisely our  convnet-like model for approximating arbitrary $\mathrm{SE}(2)$-equivariant continuous maps $f:V\to U$, where $V=L^2(\mathbb R^2,\mathbb R^{d_V}), U=L^2(\mathbb R^2,\mathbb R^{d_U})$. The overview of the model is given in Fig.\ref{fig:arch}. Like the models of Section \ref{sec:translations}, the present model starts with the discretization projection followed by some finite computation. The model includes three groups of layers: smoothing layers ($\mathcal L_{\mathrm{smooth}}$), differentiation layers ($\mathcal L_{\mathrm{diff}}$) and multiplication layers ($\mathcal L_{\mathrm{mult}}$). The parameters of the model are the lattice spacing $\lambda$, cutoff range $\Lambda$ of the output, dimension $d_{\mathrm{mult}}$ of auxiliary spaces, and the numbers $T_{\mathrm{diff}}, T_{\mathrm{mult}}$ of differentiation and multiplication layers. The overall operation of the model can be described as the chain
\begin{equation}\label{eq:chargefhatvp}
\widehat f:V\stackrel{P_{\lambda,\Lambda'}}{\longrightarrow}V_{\lambda,\Lambda'}(\equiv W_1)
\stackrel{\mathcal L_{\mathrm{smooth}}}{\longrightarrow}W_{\mathrm{smooth}}\stackrel{\mathcal L_{\mathrm{diff}}}{\longrightarrow}W_{\mathrm{diff}}\stackrel{\mathcal L_{\mathrm{mult}}}{\longrightarrow} U_{\lambda,\Lambda}.
\end{equation}  
We describe now all these layers in detail.

\subparagraph*{Initial projection.} The initial discretization projection $P_{\lambda,\Lambda'}$ is defined as explained in Section \ref{sec:translations} after Eq.\eqref{eq:convvtaul}. The input cutoff range $\Lambda'$ is given by $\Lambda'=\Lambda+(T_{\mathrm{diff}}+\lceil 4/\lambda^2\rceil)\lambda$. This padding ensures that the output cutoff range will be equal to the specified value $\Lambda$.  With respect to the spatial grid structure, the space $W_1$ can be decomposed as 
$$W_1=\oplus_{\gamma\in\lambda Z_{\lfloor\Lambda'/\lambda\rfloor}} \mathbb R^{d_V},$$
where $Z_L$ is the cubic subset of the grid defined in \eqref{eq:zl}.

\subparagraph*{Smoothing layers.} The model contains $\lceil 4/\lambda^2\rceil$ smoothing layers performing the same elementary smoothing operation $1+\tfrac{\lambda^2}{8}\Delta^{(\lambda)}$:
$$\mathcal L_{\mathrm{smooth}} = \Big(1+\frac{\lambda^2}{8}\Delta^{(\lambda)}\Big)^{\lceil 4/\lambda^2\rceil},$$
where the discrete Laplacian $\Delta^{(\lambda)}$ is defined as in Eq.\eqref{eq:laplacelam}.
In each layer the value of the transformed signal at the current spatial position is determined by the values of the signal in the previous layer at this position and its 4 nearest neigbors as given in Eq.\eqref{eq:charge1plusl2delta}. Accordingly, the domain size shrinks with each layer so that the output space of $\mathcal L_{\mathrm{smooth}}$ can be written as $$W_{\mathrm{smooth}}=\oplus_{\gamma\in\lambda Z_{\lfloor\Lambda''/\lambda\rfloor}} \mathbb R^{d_V},$$   
where $\Lambda''=\Lambda'-\lceil 4/\lambda^2\rceil\lambda=\Lambda+T_{\mathrm{diff}}\lambda.$

\subparagraph*{Differentiation layers.} The model contains $T_{\mathrm{diff}}$ differentiation layers computing the discretized derivatives $\partial_{z}^{(\lambda)}, \partial_{\overline{z}}^{(\lambda)}$ as defined in \eqref{eq:dzlam},\eqref{eq:dzovlam}. Like the smoothing layers, these derivatives shrink the domain, but additionally, as discussed in Section \ref{sec:chargediff}, they change the representation of the group $\mathrm{SE}(2)$ associated with the global charge $\mu$ (see Eq.\eqref{eq:chargedzvmu}).

Denoting the individual differentiation layers by $\mathcal L_{\mathrm{diff},t},t=1,\ldots,T_{\mathrm{diff}},$, their action  can be described as the chain
\begin{equation*}
\mathcal L_{\mathrm{diff}}:W_{\mathrm{smooth}}\stackrel{\mathcal L_{\mathrm{diff,1}}}{\longrightarrow}W_{\mathrm{diff},1}
\stackrel{\mathcal L_{\mathrm{diff,2}}}{\longrightarrow}W_{\mathrm{diff},2}\ldots
\stackrel{\mathcal L_{\mathrm{diff},T_{\mathrm{diff}}}}{\longrightarrow}W_{\mathrm{diff},T_{\mathrm{diff}}}(\equiv W_{\mathrm{diff}})
\end{equation*}  
We decompose each intermediate space $W_{\mathrm{diff},t}$ into subspaces characterized by degree $s$ of the derivative and by charge $\mu$:
\begin{equation}\label{eq:chargewdifftdecomp}
W_{\mathrm{diff},t} = \oplus_{s=0}^t\oplus_{\mu=-s}^s W_{\mathrm{diff},t,s,\mu}. 
\end{equation}
Each $W_{\mathrm{diff},t,s,\mu}$ can be further decomposed as a direct sum over the grid points:
\begin{equation}\label{eq:chargewdifftsmu}
W_{\mathrm{diff},t,s,\mu}= \oplus_{\gamma\in\lambda Z_{\lfloor\Lambda/\lambda\rfloor+T_{\mathrm{diff}}-t}} \mathbb C^{d_V}. 
\end{equation} 
Consider the operator $L_{\mathrm{diff},t}$ as a block matrix with respect to decomposition \eqref{eq:chargewdifftdecomp} of the input and output spaces $W_{\mathrm{diff},t-1}, W_{\mathrm{diff},t}$, and denote by $(\mathcal L_{\mathrm{diff},t})_{(s_{t-1},\mu_{t-1})\to(s_t,\mu_t)}$ the respective blocks. Then we define 
\begin{equation}\label{eq:chargeldiffblock}(\mathcal L_{\mathrm{diff},t})_{(s_{t-1},\mu_{t-1})\to(s_t,\mu_t)}=
\begin{cases}
\partial_z^{(\lambda)},& \text{ if } s_t=s_{t-1}+1, \mu_t=\mu_{t-1}+1,\\ 
\partial_{\overline{z}}^{(\lambda)},& \text{ if } s_t=s_{t-1}+1, \mu_t=\mu_{t-1}-1,\\
\mathbf 1,& \text{ if } s_t=s_{t-1}, \mu_t=\mu_{t-1},\\
0, & \text{ otherwise. }
\end{cases}
\end{equation}
With this definition, the final space $W_{\mathrm{diff},T_{\mathrm{diff}}}$ contains all discrete derivatives $(\partial_z^{(\lambda)})^{a}(\partial_{\overline{z}}^{(\lambda)})^{b}\mathbf \Phi$ of the smoothed signal $\mathbf \Phi\in W_{\mathrm{smooth}}$ of degrees $s=a+b\le T_{\mathrm{diff}}.$ Each such derivative can be obtained by arranging the elementary steps \eqref{eq:chargeldiffblock} in different order, so that the derivative will actually appear in $W_{\mathrm{diff},T_{\mathrm{diff}}}$ with the coefficient $\frac{T_{\mathrm{diff}}!}{a!b!(T_{\mathrm{diff}}-a-b)!}$.  This coefficient is not important for the subsequent exposition. 

\subparagraph*{Multiplication layers.} In contrast to the smoothing and differentiation layers, the multiplication layers act strictly locally (pointwise). These layers implement products and linear combinations of signals of the preceding layers subject to conservation of global charge, based on the procedure of generation of invariant polynomials described in Section \ref{sec:chargepolyinv}. 

Denoting the inividual layers by $\mathcal L_{\mathrm{mult},t}, t=1,\ldots,T_{\mathrm{mult}},$ their action is described by the chain  
\begin{equation*}
\mathcal L_{\mathrm{mult}}:W_{\mathrm{diff}}\stackrel{\mathcal L_{\mathrm{mult,1}}}{\longrightarrow}W_{\mathrm{mult},1}
\stackrel{\mathcal L_{\mathrm{mult,2}}}{\longrightarrow}W_{\mathrm{mult},2}\ldots
\stackrel{\mathcal L_{\mathrm{mult},T_{\mathrm{mult}}}}{\longrightarrow}W_{\mathrm{mult},T_{\mathrm{mult}}}\equiv U_{\lambda,\Lambda}.
\end{equation*}
Each space $W_{\mathrm{mult},t}$ except for the final one ($W_{\mathrm{mult},T_{\mathrm{mult}}}$) is decomposed into subspaces characterized by spatial position $\gamma\in (\lambda\mathbb Z)^2$ and charge $\mu$:
\begin{equation}\label{eq:chargewmultt}
W_{\mathrm{mult},t} = \oplus_{\gamma\in \lambda Z_{\lfloor\Lambda/\lambda\rfloor}}\oplus_{\mu =-T_{\mathrm{diff}}}^{T_{\mathrm{diff}}} W_{\mathrm{mult},t,\gamma,\mu}.
\end{equation}
Each space $W_{\mathrm{mult},t,\gamma,\mu}$ is a complex $d_{\mathrm{mult}}$-dimensional space, where $d_{\mathrm{mult}}$ is a parameter of the model: 
\begin{equation*}
W_{\mathrm{mult},t,\gamma,\mu} = \mathbb C^{d_{\mathrm{mult}}}.
\end{equation*}
The final space $W_{\mathrm{mult},T_{\mathrm{mult}}}$ is real, $d_U$-dimensional, and only has the charge-0 component:
\begin{equation*}
W_{\mathrm{mult},T_{\mathrm{mult}}} = \oplus_{\gamma\in \lambda Z_{\lfloor\Lambda/\lambda\rfloor}} W_{\mathrm{mult},t,\gamma,\mu=0},\quad W_{\mathrm{mult},t,\gamma,\mu=0}=\mathbb R^{d_U},
\end{equation*}
so that $W_{\mathrm{mult},T_{\mathrm{mult}}}$ can be identified with $U_{\lambda,\Lambda}.$
The initial space $W_{\mathrm{diff}}$ can also be expanded in the form \eqref{eq:chargewmultt} by reshaping its components \eqref{eq:chargewdifftdecomp},\eqref{eq:chargewdifftsmu}:
\begin{align*}
W_{\mathrm{diff}} &= \oplus_{s=0}^{T_\mathrm{diff}}\oplus_{\mu=-s}^s W_{\mathrm{diff},T_\mathrm{diff},s,\mu}\\
&= \oplus_{s=0}^{T_\mathrm{diff}}\oplus_{\mu=-s}^s \oplus_{\gamma\in\lambda Z_{\lfloor\Lambda/\lambda\rfloor}} \mathbb C^{d_V}\\
&=\oplus_{\gamma\in\lambda Z_{\lfloor\Lambda/\lambda\rfloor}}\oplus_{\mu=-T_\mathrm{diff}}^{T_\mathrm{diff}}W_{\mathrm{mult},0,\gamma,\mu},
\end{align*}
where 
\begin{align*}
W_{\mathrm{mult},0,\gamma,\mu} = \oplus_{s=|\mu|}^{T_\mathrm{diff}} \mathbb C^{d_V}.
\end{align*}
The multiplication layers $\mathcal L_{\mathrm{mult},t}$ act separately and identically at each $\gamma\in \lambda Z_{\lfloor\Lambda/\lambda\rfloor},$ i.e., without loss of generality these layers can be thought of as maps
\begin{equation*}
\mathcal L_{\mathrm{mult},t}: \oplus_{\mu =-T_{\mathrm{diff}}}^{T_{\mathrm{diff}}} W_{\mathrm{mult},t-1,\gamma=0,\mu}\longrightarrow \oplus_{\mu =-T_{\mathrm{diff}}}^{T_{\mathrm{diff}}} W_{\mathrm{mult},t,\gamma=0,\mu}.
\end{equation*}
To define $\mathcal L_{\mathrm{mult},t}$, let us represent its input $\mathbf\Phi\in \oplus_{\mu =-T_{\mathrm{diff}}}^{T_{\mathrm{diff}}} W_{\mathrm{mult},t-1,\gamma=0,\mu}$ as
$$
\mathbf\Phi=\sum_{\mu =-T_{\mathrm{diff}}}^{T_{\mathrm{diff}}}\mathbf\Phi_\mu=
\sum_{\mu =-T_{\mathrm{diff}}}^{T_{\mathrm{diff}}}\sum_{n=1}^{d_{\mathrm{mult}}}\Phi_{\mu,n}\mathbf e_{\mu,n},
$$
where $\mathbf e_{\mu,n}$ denote the basis vectors in $W_{\mathrm{mult},t-1,\gamma=0,\mu}$. We represent the output $\mathbf\Psi\in \oplus_{\mu =-T_{\mathrm{diff}}}^{T_{\mathrm{diff}}} W_{\mathrm{mult},t,\gamma=0,\mu}$ of $\mathcal L_{\mathrm{mult},t}$ in the same way: 
$$
\mathbf\Psi=\sum_{\mu =-T_{\mathrm{diff}}}^{T_{\mathrm{diff}}}\mathbf\Psi_\mu=
\sum_{\mu =-T_{\mathrm{diff}}}^{T_{\mathrm{diff}}}\sum_{n=1}^{d_{\mathrm{mult}}}\Psi_{\mu,n}\mathbf e_{\mu,n}.
$$  
Then, based on Eq.\eqref{eq:chargeftmun}, for $t<T_{\mathrm{mult}}$ we define $\mathcal L_{\mathrm{mult},t}\mathbf \Phi=\mathbf\Psi$ by
\begin{equation}\label{eq:chargemultweightsmu}
\Psi_{\mu,n}=w^{(t)}_{0,n}\mathbf{1}_{\mu=0}+\sum_{n_1=1}^{d_{\mathrm{mult}}}
w_{1, \mu,n,n_1}^{(t)}\Phi_{\mu,n_1}+\sum_{\genfrac{}{}{0pt}{}{-T_{\mathrm{diff}}\le \mu_1,\mu_2\le T_{\mathrm{diff}}}{\mu_1+\mu_2=\mu} }\sum_{n_1=1}^{d_{\mathrm{mult}}}\sum_{n_2=1}^{d_{\mathrm{mult}}}
w_{2, \mu_1,\mu_2,n,n_1,n_2}^{(t)}\Phi_{\mu_1,n_1}\Phi_{\mu_2,n_2},\end{equation}
with some complex weights $w^{(t)}_{0,n}, w_{1, \mu,n,n_1}^{(t)}, w_{2, \mu_1,\mu_2,n,n_1,n_2}^{(t)}$. In the final layer $t=T_{\mathrm{mult}}$ the network only needs to generate a real charge-0 (invariant) vector, so in this case $\mathbf\Psi$ only has real $\mu=0$ components:
\begin{equation}\label{eq:chargemultweights0}
\Psi_{0,n}=\operatorname{Re}\Big(w^{(t)}_{0,n}+\sum_{n_1=1}^{d_{\mathrm{mult}}}
w_{1, 0,n,n_1}^{(t)}\Phi_{0,n_1}+\sum_{\genfrac{}{}{0pt}{}{-T_{\mathrm{diff}}\le \mu_1,\mu_2\le T_{\mathrm{diff}}}{\mu_1+\mu_2=0} }\sum_{n_1=1}^{d_{\mathrm{mult}}}\sum_{n_2=1}^{d_{\mathrm{mult}}}
w_{2, \mu_1,\mu_2,n,n_1,n_2}^{(t)}\Phi_{\mu_1,n_1}\Phi_{\mu_2,n_2}\Big).\end{equation}

\paragraph*{} This completes the description of the charge-conserving convnet. In the sequel, it will be convenient to consider a family of convnets having all parameters and weights in common except for the grid spacing $\lambda$. Observe that this parameter can be varied independently of all other parameters and weights ($\Lambda, d_{\mathrm{mult}}, T_{\mathrm{diff}}, T_{\mathrm{mult}}, w^{(t)}_{0,n}, w_{1, \mu,n,n_1}^{(t)}, w_{2, \mu_1,\mu_2,n,n_1,n_2}^{(t)}$). The parameter $\lambda$ affects the number of smoothing layers, and decreasing this parameter means that essentially the same convnet is applied at a higher resolution. Accordingly, we will call such a family a ``multi-resolution convnet''.

\begin{defin} A \textbf{charge-conserving convnet} is a map $\widehat f:V\to U$ given in \eqref{eq:chargefhatvp},  characterized by parameters $\lambda, \Lambda, d_{\mathrm{mult}}, T_{\mathrm{diff}}, T_{\mathrm{mult}}$ and weights $w^{(t)}_{0,n}, w_{1, \mu,n,n_1}^{(t)}, w_{2, \mu_1,\mu_2,n,n_1,n_2}^{(t)}$, and constructed as described above. A \textbf{multi-resolution charge-conserving convnet} $\widehat f_\lambda$ is obtained by arbitrarily varying the grid spacing parameter $\lambda$ in the charge-conserving convnet $\widehat f$. 
\end{defin}

We comment now why it is natural to call this model ``charge-conserving''. As already explained in Section \ref{sec:chargediff}, if the intermediate spaces labeled by specific $\mu$'s are equipped with the special representations \eqref{eq:chargermugta}, then, up to the spatial cutoff, the differentiation layers $\mathcal L_{\mathrm{diff}}$ are $\mathrm{SE}(2)$-equivariant and conserve the ``total charge'' $\mu+\eta$, where $\eta$ is the ``local charge'' (see Eq.\eqref{eq:chargeladder}). Clearly, the same can be said about the smoothing layers $\mathcal L_{\mathrm{smooth}}$ which, in fact, separately conserve the global charge $\mu$ and the local charge $\eta$. Moreover, observe that the multiplication layers $\mathcal L_{\mathrm{mult}}$, though nonlinear, are also equivariant and separately conserve the charges $\mu$ and $\eta$. Indeed, consider the transformations \eqref{eq:chargemultweightsmu},\eqref{eq:chargemultweights0}. The first term in these transformations creates an $\mathrm{SE}(2)$-invariant, $\mu=\eta=0$ signal. The second, linear term does not change $\mu$ or $\eta$ of the input signal. The third term creates products $\Psi_{\mu}= \Phi_{\mu_1}\Phi_{\mu_2}$, where $\mu=\mu_1+\mu_2$. This multiplication operation is equivariant with respect to the respective representations $R^{(\mu)}, R^{(\mu_1)}, R^{(\mu_2)}$ as defined in \eqref{eq:chargermugta}. Also, if the signals $\Phi_{\mu_1},\Phi_{\mu_2}$ have local charges $\eta_1,\eta_2$ at a particular point $\mathbf x$, then the product $\Phi_{\mu_1}\Phi_{\mu_2}$ has local charge $\eta=\eta_1+\eta_2$ at this point (see Eqs.\eqref{eq:chargevmueta},\eqref{eq:chargevmueta0}).

\subsection{The main result}\label{sec:chargeresult}

To state our main result, we define a limit point of charge-conserving convnets.
\begin{defin}\label{def:convlimpoint}
With $V=L^2(\mathbb R^\nu,\mathbb R^{d_V})$ and $U=L^2(\mathbb R^\nu,\mathbb R^{d_U})$, we say that a map $f:V\to U$ is a \textbf{limit point of charge-conserving convnets} if for any compact set $K\subset V$, any $\epsilon>0$ and $\Lambda_0>0$ there exist a multi-resolution charge-conserving convnet $\widehat f_\lambda$ with $\Lambda>\Lambda_0$ such that $\sup_{\mathbf \Phi\in K}\|\widehat f_\lambda(\mathbf \Phi)-f(\mathbf \Phi)\|\le\epsilon$ for all sufficiently small grid spacings $\lambda$.  
\end{defin}
Then our main result is the following theorem.
\begin{theorem}\label{th:charge}
Let $V=L^2(\mathbb R^\nu,\mathbb R^{d_V})$ and $U=L^2(\mathbb R^\nu,\mathbb R^{d_U})$. A map $f:V\to  U$ is a limit point of charge-conserving convnets if and only if $f$ is $\mathrm{SE}(2)$--equivariant and continuous in the norm topology. 
\end{theorem}
\begin{proof} To simplify the exposition, we will assume that $d_V=d_U=1$; generalization of all the arguments to vector-valued input and output signals is straightforward. 

We start by observing that a multi-resolution family of charge-conserving convnets has a natural scaling limit as the lattice spacing $\lambda\to 0$:
\begin{equation}\label{eq:chargef0hatlamb}
\widehat f_0(\mathbf \Phi)=\lim_{\lambda\to 0}\widehat f_\lambda(\mathbf \Phi).\end{equation}
Indeed, by \eqref{eq:chargefhatvp}, at $\lambda>0$ we can represent the convnet as the composition of maps
$$\widehat f_\lambda = \mathcal L_{\mathrm{mult}}\circ \mathcal L_{\mathrm{diff}}\circ \mathcal L_{\mathrm{smooth}}\circ P_{\lambda, \Lambda'}.$$ 
The part $\mathcal L_{\mathrm{diff}}\circ \mathcal L_{\mathrm{smooth}}\circ P_{\lambda, \Lambda'}$ of this computation implements several maps $\mathcal L_{\lambda}^{(a,b)}$ introduced in \eqref{eq:chargedefltau}. More precisely, by the definition of differentiation layers in Section \ref{sec:chargeconvnet}, the output space $W_{\mathrm{diff}}$ of the linear operator $\mathcal L_{\mathrm{diff}}\circ \mathcal L_{\mathrm{smooth}}\circ P_{\lambda, \Lambda'}$ can be decomposed into the direct sum \eqref{eq:chargewdifftdecomp} over several degrees $s$ and charges $\mu$. The respective components of $\mathcal L_{\mathrm{diff}}\circ \mathcal L_{\mathrm{smooth}}\circ P_{\lambda, \Lambda'}$ are, up to unimportant combinatoric coefficients, just the operators $\mathcal L_{\lambda}^{(a,b)}$ with $a+b=s$, $a-b=\mu$:
\begin{equation}\label{eq:chargeldifflsmooth}\mathcal L_{\mathrm{diff}}\circ \mathcal L_{\mathrm{smooth}}\circ P_{\lambda, \Lambda'}=(\ldots,c_{a,b}\mathcal L_{\lambda}^{(a,b)},\ldots),\quad c_{a,b}=\tfrac{T_{\mathrm{diff}}!}{a!b!(T_{\mathrm{diff}}-a-b)!},\end{equation}
with the caveat that the output of $\mathcal L_{\lambda}^{(a,b)}$ is spatially restricted to the bounded domain $[-\Lambda,\Lambda]^2$. By Lemma \ref{lm:chargeclt}, as $\lambda\to 0$, the operators $\mathcal L_{\lambda}^{(a,b)}$ converge to the operator $\mathcal L_{0}^{(a,b)}$ defined in Eq.\eqref{eq:chargel0ab}, so that   for any $\mathbf\Phi\in L^2(\mathbb R^\nu
)$ the signals $\mathcal L_{\lambda}^{(a,b)}\mathbf\Phi$ are bounded functions on $\mathbb R^\nu$ and converge to $\mathcal L_{0}^{(a,b)}\mathbf\Phi$ in the uniform norm $\|\cdot\|_\infty$. Let us denote the limiting linear operator by $\mathcal L_{{\mathrm{conv}}}$:
\begin{equation}\label{eq:chargelconv}\mathcal L_{{\mathrm{conv}}} = \lim_{\lambda\to 0}\mathcal L_{\mathrm{diff}}\circ \mathcal L_{\mathrm{smooth}}\circ P_{\lambda, \Lambda'}.\end{equation}
The full limiting map $\widehat f_0(\mathbf \Phi)$ is then obtained by pointwise application (separately at each point $\mathbf x\in[-\Lambda,\Lambda]^2$) of the multiplication layers $\mathcal L_{\mathrm{mult}}$ to the signals $\mathcal L_{T_{\mathrm{diff}}}\mathbf \Phi$:
\begin{equation}\label{eq:chargef0philmultdiff}\widehat f_0(\mathbf\Phi)=\mathcal L_{\mathrm{mult}}(\mathcal L_{{\mathrm{conv}}}\mathbf \Phi).\end{equation}
For any $\mathbf\Phi\in L^2(\mathbb R^\nu
)$,  this $f_0(\mathbf\Phi)$ is a well-defined bounded signal on the domain $[-\Lambda,\Lambda]^2.$ It is bounded because the multiplication layers $\mathcal L_{\mathrm{mult}}$ implement a continuous (polynomial) map, and because, as already mentioned, $\mathcal L_{{\mathrm{conv}}}\mathbf \Phi$ is a bounded signal. Since the domain $[-\Lambda,\Lambda]^2$ has a finite Lebesgue measure, we have $f_0(\mathbf\Phi)\in L^\infty([-\Lambda,\Lambda]^2
)\subset L^2([-\Lambda,\Lambda]^2
).$ By a similar argument, the convergence in \eqref{eq:chargef0hatlamb} can be understood in the  $L^\infty([-\Lambda,\Lambda]^2
)$ or $L^2([-\Lambda,\Lambda]^2
)$ sense, e.g.:
\begin{equation}\label{eq:chargef0fll2}
\|\widehat f_0(\mathbf \Phi)-\widehat f_\lambda(\mathbf \Phi)\|_{L^2([-\Lambda,\Lambda]^2)}\stackrel{\lambda\to 0}{\longrightarrow}0,\quad \mathbf \Phi\in V.
\end{equation}
Below, we will use the scaling limit $\widehat f_0$ as an intermediate approximator.

We will now prove the necessity and then the sufficiency parts of the theorem. 

\paragraph*{Necessity} \emph{(a limit point $f$ is continuous and $\mathrm{SE}(2)$--equivariant)}. 
As in the previous theorems \ref{th:convmain}, \ref{th:convpool}, continuity of $f$ follows by standard topological arguments, and we only need to prove the $\mathrm{SE}(2)$--equivariance.

Let us first prove the $\mathbb R^2$--equivariance of $f$. By the definition of a limit point, for any $\mathbf \Phi\in V$, $\mathbf x\in\mathbb R^2$, $\epsilon>0$ and $\Lambda_0>0$ there  is a multi-resolution convnet $\widehat f_\lambda$ with $\Lambda>\Lambda_0$ such that 
\begin{equation}\label{eq:chargesupphikwf}
\|\widehat f_\lambda(\mathbf \Phi)-f(\mathbf \Phi)\|\le\epsilon,\quad\|\widehat f_\lambda(R_{(\mathbf x,1)}\mathbf \Phi)-f(R_{(\mathbf x,1)}\mathbf \Phi)\|\le\epsilon\end{equation} for all sufficiently small $\lambda$. Consider the scaling limit $\widehat f_0=\lim_{\lambda\to 0}\widehat f_\lambda$ constructed above. As shown above, $\widehat f_\lambda(\mathbf\Phi)$ converges to $\widehat f_0(\mathbf\Phi)$ in the $L^2$ sense, so the inequalities \eqref{eq:chargesupphikwf} remains valid for $\widehat f_0(\mathbf\Phi)$:
\begin{equation}\label{eq:chargesupphikwf0}
\|\widehat f_0(\mathbf \Phi)-f(\mathbf \Phi)\|\le\epsilon,\quad\|\widehat f_0(R_{(\mathbf x,1)}\mathbf \Phi)-f(R_{(\mathbf x,1)}\mathbf \Phi)\|\le\epsilon.\end{equation}
The map $\widehat f_0$ is not $\mathbb R^2$--equivariant only because its output is restricted to the domain $[-\Lambda,\Lambda]^2$, since otherwise both maps $\mathcal L_{\mathrm{mult}},\mathcal L_{T_{\mathrm{diff}}}$ appearing in the superposition \eqref{eq:chargef0philmultdiff} are $\mathbb R^2$--equivariant. Therefore, for any $\mathbf y\in \mathbb R^2$, 
\begin{equation}\label{eq:chargeidf0r}
\widehat f_0(R_{(\mathbf x,1)}\mathbf \Phi)(\mathbf y)= R_{(\mathbf x,1)}\widehat f_0(\mathbf \Phi)(\mathbf y)=\widehat f_0(\mathbf \Phi)(\mathbf y-\mathbf x),\quad\text{ if } \mathbf y,\mathbf y-\mathbf x\in [-\Lambda,\Lambda]^2.
\end{equation}
Consider the set $$\Pi_{\Lambda, \mathbf x}=\{\mathbf y\in\mathbb R^2: \mathbf y,\mathbf y-\mathbf x\in [-\Lambda,\Lambda]^2\}=[-\Lambda,\Lambda]^2\cap R_{(\mathbf x,1)}([-\Lambda,\Lambda]^2).$$ The identity \eqref{eq:chargeidf0r} implies that 
\begin{equation}\label{eq:chargeppilmx}P_{\Pi_{\Lambda, \mathbf x}}\widehat f_0(R_{(\mathbf x,1)}\mathbf \Phi)=P_{\Pi_{\Lambda, \mathbf x}}R_{(\mathbf x,1)}\widehat f_0(\mathbf \Phi),\end{equation}
where $P_{\Pi_{\Lambda, \mathbf x}}$ denotes the projection to the subspace $L^2(\Pi_{\Lambda, \mathbf x}
)$ in $L^2(\mathbb R^2
)$. For a fixed $\mathbf x$, the projectors $P_{\Pi_{\Lambda, \mathbf x}}$ converge strongly to the identity as $\Lambda\to\infty$, therefore we can choose $\Lambda$ sufficiently large so that 
\begin{equation}\label{eq:chargeppilx}\|P_{\Pi_{\Lambda, \mathbf x}} f(\mathbf \Phi)-f(\mathbf \Phi)\|\le\epsilon,\quad \|P_{\Pi_{\Lambda, \mathbf x}} f(R_{(\mathbf x,1)}\mathbf \Phi)- f(R_{(\mathbf x,1)}\mathbf \Phi)\|\le\epsilon.\end{equation}
Then, assuming that the approximating convnet has a sufficiently large range $\Lambda$, we have 
\begin{align*}
\|f(R_{(\mathbf x,1)}\mathbf \Phi)-R_{(\mathbf x,1)} f(\mathbf \Phi)\|
\le{} & \| f(R_{(\mathbf x,1)}\mathbf \Phi)-P_{\Pi_{\Lambda, \mathbf x}} f(R_{\mathbf x,1)}\mathbf \Phi)\|\\
&+\|P_{\Pi_{\Lambda, \mathbf x}} f(R_{(\mathbf x,1)}\mathbf \Phi)-P_{\Pi_{\Lambda, \mathbf x}}\widehat f_0(R_{(\mathbf x,1)}\mathbf \Phi)\|\\
&+\|P_{\Pi_{\Lambda, \mathbf x}}\widehat f_0(R_{(\mathbf x,1)}\mathbf \Phi)-P_{\Pi_{\Lambda, \mathbf x}}R_{(\mathbf x,1)}\widehat f_0(\mathbf \Phi)\|\\
&+\|P_{\Pi_{\Lambda, \mathbf x}}R_{(\mathbf x,1)}\widehat f_0(\mathbf \Phi)-P_{\Pi_{\Lambda, \mathbf x}}R_{(\mathbf x,1)} f(\mathbf \Phi)\|\\
&+\|P_{\Pi_{\Lambda, \mathbf x}}R_{(\mathbf x,1)} f(\mathbf \Phi)-R_{(\mathbf x,1)}f(\mathbf \Phi)\|\\
&\le4\epsilon,\end{align*}
where we used the bounds \eqref{eq:chargesupphikwf0}, \eqref{eq:chargeppilx}, the equalities $\|P_{\Pi_{\Lambda, \mathbf x}}\|=\|R_{(\mathbf x,1)}\|=1$, and the identity \eqref{eq:chargeppilmx}. Taking the limit $\epsilon\to 0$, we obtain 
the desired $\mathbb R^2$--equivariance of $f$.

To complete the proof of $\mathrm{SE}(2)$--equivariance, we will show that for any $\theta\in \mathrm{SO}(2)$ we have
\begin{equation}\label{eq:chargertheta}
R_{(0,\theta)}\widehat f_0(\mathbf\Phi)(\mathbf x)=\widehat f_0(R_{(0,\theta)}\mathbf\Phi)(\mathbf x), \quad \mathbf x\in \Pi_{\Lambda, \theta},
\end{equation}
where
$$\Pi_{\Lambda, \theta}=[-\Lambda, \Lambda]^2\cap R_{(0,\theta)}([-\Lambda, \Lambda]^2).$$
Identity \eqref{eq:chargertheta} is an analog of identity \eqref{eq:chargeidf0r} that we used to prove the $R_{(\mathbf x,1)}$--equivariance of $f$. Once Eq.\eqref{eq:chargertheta} is established, we can prove the $R_{(0,\theta)}$--equivariance of $f$ by arguing in the same way as we did above to prove the $R_{(\mathbf x,1)}$--equivariance. After that, the $R_{(0,\theta)}$--equivariance and the $R_{(\mathbf x,1)}$--equivariance together imply the full $\mathrm{SE}(2)$--equivariance.

Note that by using the partial translation equivariance \eqref{eq:chargeidf0r} and repeating the computation from Lemma \ref{lm:chargecharequiv}, it suffices to prove the identity \eqref{eq:chargertheta} only in the special case $\mathbf x=0$:
\begin{equation}\label{eq:chargertheta0}
\widehat f_0(R_{(0,\theta)}\mathbf\Phi)(0)=\widehat f_0(\mathbf\Phi)(0).
\end{equation}
Indeed, suppose that Eq.\eqref{eq:chargertheta0} is established and $\Lambda$ is sufficiently large so that $\mathbf x, \theta^{-1}\mathbf x\in [-\Lambda,\Lambda]^2$. Then, 
\begin{align*}
\widehat f_0(R_{(0,\theta)}\mathbf\Phi)(\mathbf x) 
&=R_{(-\mathbf x,1)}\widehat f_0(R_{(0,\theta)}\mathbf \Phi)( 0)\\
&=\widehat f_0(R_{(-\mathbf x,1)}R_{(0,\theta)}\mathbf \Phi)( 0)\\
&=\widehat f_0(R_{(0,\theta)}R_{(-\theta^{-1}\mathbf x,1)}\mathbf \Phi)( 0)\\
&=\widehat f_0(R_{(-\theta^{-1}\mathbf x,1)}\mathbf \Phi)( 0)\\
&=R_{(-\theta^{-1}\mathbf x,1)}\widehat f_0(\mathbf \Phi)(0)\\
&=R_{(\mathbf x,\theta)}R_{(-\theta^{-1}\mathbf x,1)}\widehat f_0(\mathbf \Phi)(\mathcal A_{(\mathbf x,\theta)} 0)\\
&=R_{(0,\theta)}\widehat f_0(\mathbf \Phi)(\mathbf x),
\end{align*}
where we used general properties of the representaton $R$ (steps 1, 6, 7), Eq.\eqref{eq:chargertheta0} (step 4), and the partial $\mathbb R^\nu$--equivariance \eqref{eq:chargeidf0r} (steps 2 and 5, using the fact that $0,\mathbf x, \theta^{-1}\mathbf x\in [-\Lambda,\Lambda]^2$). 

To establish Eq.\eqref{eq:chargertheta0}, recall that, by Eq.\eqref {eq:chargef0philmultdiff}, the value $\widehat f_0(\mathbf\Phi)(0)$ is obtained by first evaluating  $\mathcal L_{\mathrm{conv}}(\mathbf \Phi)$ at $\mathbf x=0$ and then applying to the resulting values the map $\mathcal L_{\mathrm{mult}}$. By Eqs.\eqref{eq:chargeldifflsmooth},\eqref{eq:chargelconv} and Lemma \ref{lm:chargeclt},  we can write  $\mathcal L_{\mathrm{conv}}(\mathbf \Phi)(0)$ as a vector with components
\begin{equation}\label{eq:lconvphi0}\mathcal L_{\mathrm{conv}}(\mathbf \Phi)(0)=(\ldots,c_{a,b}\mathcal L_{0}^{(a,b)},\ldots),\end{equation}
where, by Eq.\eqref{eq:chargel0ab},
$$\mathcal L_0^{(a,b)}\mathbf\Phi(0) = \int_{\mathbb R^2}\mathbf\Phi(-\mathbf y)\Psi_{a,b}(\mathbf y)d^2\mathbf y,$$
and $\Psi_{a,b}$ is given by Eq.\eqref{eq:chargepsiab}:
$$\Psi_{a,b}=\partial_z^a\partial_{\overline{z}}^b\Big(\frac{1}{2\pi}e^{-|\mathbf x|^2/2}\Big).$$
In the language of Section \ref{sec:chargediff}, $\Psi_{a,b}$ has local charge $\eta=b-a$:
$$\Psi_{a,b}(\mathcal A_{(0,e^{-i\phi})}\mathbf x)=e^{i(a-b)\phi}\Psi_{a,b}(\mathbf x).$$
It follows that 
\begin{align*}
\mathcal L_{0}^{(a,b)}(R_{(0,e^{i\phi})}\mathbf \Phi)(0)={}&\int_{\mathbb R^2}R_{(0,e^{i\phi})}\mathbf\Phi(-\mathbf y)\Psi_{a,b}(\mathbf y)d^2\mathbf y\\
={}&\int_{\mathbb R^2}\mathbf\Phi(\mathcal A_{(0,e^{-i\phi})}(-\mathbf y))\Psi_{a,b}(\mathbf y)d^2\mathbf y\\
={}&\int_{\mathbb R^2}\mathbf\Phi(-\mathbf y)\Psi_{a,b}(\mathcal A_{(0,e^{i\phi})}\mathbf y)d^2\mathbf y\\
={}&\int_{\mathbb R^2}\mathbf\Phi(-\mathbf y)e^{i(b-a)\phi}\Psi_{a,b}(\mathbf y)d^2\mathbf y\\
={}&e^{i(b-a)\phi}\mathcal L_{0}^{(a,b)}(\mathbf \Phi)(0),\end{align*}
i.e., $\mathcal L_{0}^{(a,b)}(\mathbf \Phi)(0)$ transforms under rotations $e^{i\phi}\in\mathrm{SO}(2)$ as a character \eqref{eq:charge1drep} with $\xi=b-a$.

Now consider the map $\mathcal L_{\mathrm{mult}}.$ Since each component in the decomposition \eqref{eq:lconvphi0} transforms as a character with $\xi=b-a$, the construction of $\mathcal L_{\mathrm{mult}}$ in Section \ref{sec:chargeconvnet} (based on the procedure of generating invariant polynomials described in Section \ref{sec:chargepolyinv}) quarantees that $\mathcal L_{\mathrm{mult}}$ computes a function invariant with respect to $\mathrm{SO}(2)$, thus proving Eq.\eqref{eq:chargertheta0}: 
$$
\widehat f_0(R_{(0,\theta)}\mathbf\Phi)(0)=\mathcal L_{\mathrm{mult}}(\mathcal L_{\mathrm{conv}}(R_{(0,\theta)}\mathbf \Phi)(0))=\mathcal L_{\mathrm{mult}}(\mathcal L_{\mathrm{conv}}(\mathbf \Phi)(0))=\widehat f_0(\mathbf\Phi)(0).
$$  
This completes the proof of the necessity part.

\paragraph{Sufficiency} \emph{(a continuous $\mathrm{SE}(2)$--equivariant map $f:V\to U$ can be approximated by charge-conserving convnets).}

Given a continuous $\mathrm{SE}(2)$--equivariant $f:V\to U$, a compact set $K\subset V$ and positive numbers $\epsilon, \Lambda_0,$ we need to construct a multi-resolution charge-conserving convnet $\widehat f=(\widehat f_\lambda)$ with $\Lambda>\Lambda_0$ and the property $\sup_{\mathbf \Phi\in K}\|\widehat f_\lambda(\mathbf \Phi)-f(\mathbf \Phi)\|\le\epsilon$ for all sufficiently small $\lambda$. 
We construct the desired convnet by performing a series of reductions of this approximation problem.

\medskip
\noindent
\textbf{1. Smoothing.} For any $\epsilon_1>0$, consider the smoothed map $\widetilde{f}_{\epsilon_1}:V\to U$ defined by
\begin{equation}\label{eq:chargefeps1}
\widetilde{f}_{\epsilon_1}(\mathbf\Phi)=f(\mathbf\Phi)*g_{\epsilon_1},
\end{equation}
where $$g_{\epsilon_1}(\mathbf x)=\frac{1}{2\pi\epsilon_1}e^{-|\mathbf x|^2/(2\epsilon_1)}.$$
The map $\widetilde{f}_{\epsilon_1}$ is continuous and $\mathrm{SE}(2)$--equivariant, as a composition of two continuous and $\mathrm{SE}(2)$--equivariant  maps. 
We can choose $\epsilon_1$ small enough so that for all $\mathbf \Phi\in K$
\begin{equation}\label{eq:chargetfeps1fphi}\|\widetilde{f}_{\epsilon_1}(\mathbf\Phi)-f(\mathbf\Phi)\|\le \frac{\epsilon}{10}.\end{equation}
The problem of approximating $f$ then reduces to the problem of approximating maps $\widetilde{f}_{\epsilon_1}$  of the form \eqref{eq:chargefeps1}.

\medskip
\noindent
\textbf{2. Spatial cutoff.} We can choose $\Lambda$ sufficiently large so that for all $\mathbf\Phi\in K$ \begin{equation}\label{eq:chargepltfeps1}\|P_\Lambda \widetilde{f}_{\epsilon_1}(\mathbf\Phi)-\widetilde{f}_{\epsilon_1}(\mathbf\Phi)\|<\frac{\epsilon}{10}.\end{equation}
We can do this because $\widetilde{f}_{\epsilon_1}(K)$ is compact, as an image of a compact set under a continuous map, and because $P_\Lambda$ converge strongly to the identity as  $\Lambda\to +\infty$. Thus, we only need to approximate output signals $\widetilde{f}_{\epsilon_1}(\mathbf\Phi)$ on the domain $[-\Lambda,\Lambda]^2$.

\medskip
\noindent
\textbf{3. Output localization.} Define the map $\widetilde f_{\epsilon_1,\mathrm{loc}}: V\to \mathbb R$ by
\begin{equation}\label{eq:chargetildefepsloc}\widetilde f_{\epsilon_1,\mathrm{loc}}(\mathbf\Phi)=\widetilde f_{\epsilon_1}(\mathbf\Phi)(0)=\langle g_{\epsilon_1}, f(\mathbf\Phi) \rangle_{L^2(\mathbb R^\nu)}.\end{equation}
Since both $g_{\epsilon_1}, f(\mathbf\Phi)\in L^2(\mathbb R^\nu)$, the map $\widetilde f_{\epsilon_1,\mathrm{loc}}$ is well-defined, and it is continuous since $f$ is continuous.

By translation equivariance of $f$ and hence $\widetilde f_{\epsilon_1}$, the map $\widetilde f_{\epsilon_1}$ can be recovered from $\widetilde f_{\epsilon_1,\mathrm{loc}}$ by 
\begin{equation}\label{eq:chargerecwfe1}\widetilde f_{\epsilon_1}(\mathbf\Phi)(\mathbf x)=\widetilde f_{\epsilon_1}(R_{(-\mathbf x,1)}\mathbf\Phi)(0)=\widetilde f_{\epsilon_1,\mathrm{loc}}(R_{(-\mathbf x,1)}\mathbf\Phi).\end{equation}
By the $\mathrm{SO}(2)$-equivariance of $\widetilde f_{\epsilon_1},$ the map $\widetilde f_{\epsilon_1,\mathrm{loc}}$ is $\mathrm{SO}(2)$-invariant.

\medskip
\noindent
\textbf{4. Nested finite-dimensional $\mathrm{SO}(2)$-modules $V_\zeta$.} For any nonnegative integer $a,b$ consider again the signal $\Psi_{a,b}$ introduced in Eq.\eqref{eq:chargepsiab}.
For any $\zeta=1,2,\ldots,$ consider the subspace $V_\zeta\subset V$ spanned by the vectors $\operatorname{Re}(\Psi_{a,b})
$ and $\operatorname{Im}(\Psi_{a,b})
$ with $a+b\le \zeta.$  
These vectors form a total system in $V$ if $a,b$ take arbitrary nonnegative integer values. Accordingly, if we denote by $P_{V_\zeta}$ the orthogonal projection to $V_\zeta$ in $V$, then the operators $P_{V_\zeta}$ converge strongly to the identity as $\zeta\to\infty.$ 

The subspace $V_\zeta$ is a real finite-dimensional $\mathrm{SO}(2)$--module. As discussed in Subsection \ref{sec:chargepolyinv}, it is convenient to think of such modules as consisting of complex conjugate irreducible representations under constraint \eqref{eq:chargewnconj}. The complex extension of the real module $V_\zeta$ is spanned by signals $\{\Psi_{a,b}
\}_{a+b\le\zeta
}$, so that $\Psi_{a,b}
$ and $\Psi_{b,a}
$ form a complex conjugate pair for $a\ne b$ (if $a=b$, then $\Psi_{a,b}
$ is real). The natural representation \eqref{eq:chargergta} of $\mathrm{SO}(2)$ 
 transforms the signal $\Psi_{a,b}
$  as a character \eqref{eq:charge1drep} with $\xi=a-b$ (in the language of Section \ref{sec:chargediff}, $\Psi_{a,b}
$ has local charge $\eta=b-a$ w.r.t. $\mathbf x =0$):
\begin{equation}\label{eq:charger0eiphipsiab}R_{(0,e^{i\phi})}\Psi_{a,b}(\mathbf x)=\Psi_{a,b}(\mathcal A_{(0,e^{-i\phi})}\mathbf x)=e^{i(a-b)\phi}\Psi_{a,b}(\mathbf x).\end{equation}
The action of $\mathrm{SO}(2)$ on the real signals $\operatorname{Re}(\Psi_{a,b})
$ and $\operatorname{Im}(\Psi_{a,b})
$ can be related to its action on $\Psi_{a,b}
$ and $\Psi_{b,a}
$ as in Eqs.\eqref{eq:chargereiphixy},\eqref{eq:chargereiphizz}. 

\medskip
\noindent
\textbf{5. Restriction to  $V_\zeta$.}
Let $\widetilde f_{\epsilon_1,\mathrm{loc},\zeta}:V_\zeta\to\mathbb R$ be the restriction of the map $\widetilde f_{\epsilon_1,\mathrm{loc}}$ defined in Eq.\eqref{eq:chargetildefepsloc} to the subspace $V_\zeta$:
\begin{equation}\label{eq:chargeftildeepsloczeta}\widetilde f_{\epsilon_1,\mathrm{loc},\zeta}=\widetilde f_{\epsilon_1,\mathrm{loc}}|_{V_\zeta}.\end{equation}
Consider the map $\widetilde f_{\epsilon_1,\zeta}:V\to U$ defined by projecting to $V_\zeta$ and translating the map $\widetilde f_{\epsilon_1,\mathrm{loc},\zeta}$ to points $\mathbf x\in[-\Lambda,\Lambda]^2$ like in the reconstruction formula \eqref{eq:chargerecwfe1}:
\begin{equation}\label{eq:chargewtfeps1zeta}\widetilde f_{\epsilon_1,\zeta}(\mathbf\Phi)(\mathbf x)=
\begin{cases}\widetilde f_{\epsilon_1,\mathrm{loc},\zeta}(P_{V_\zeta}R_{(-\mathbf x,1)}\mathbf\Phi),&\mathbf x\in[-\Lambda,\Lambda]^2\\
0, &\text{otherwise}.\end{cases}\end{equation}
We claim that if $\zeta$ is sufficiently large then for all $\mathbf \Phi\in K$
\begin{equation}\label{eq:chargefeps1z0}\|\widetilde f_{\epsilon_1,\zeta}(\mathbf\Phi)-P_\Lambda \widetilde{f}_{\epsilon_1}(\mathbf\Phi)\|<\frac{\epsilon}{10}.\end{equation}
Indeed,
\begin{equation}\label{eq:chargefeps1z}\|\widetilde f_{\epsilon_1,\zeta}(\mathbf\Phi)-P_\Lambda \widetilde{f}_{\epsilon_1}(\mathbf\Phi)\|\le 2\Lambda \sup_{\mathbf\Phi_1\in K_1}|\widetilde f_{\epsilon_1,\mathrm{loc}}(P_{V_\zeta}\mathbf\Phi_1)-\widetilde f_{\epsilon_1,\mathrm{loc}}(\mathbf\Phi_1)|,\end{equation}
where \begin{equation}\label{eq:chargek1}K_1=\{R_{(-\mathbf x,1)}\mathbf\Phi)|(\mathbf x, \mathbf\Phi)\in[-\Lambda,\Lambda]^2\times K\}\subset V.\end{equation}
The set $K_1$ is compact, by compactness of $K$ and strong continuity of $R$.  Then, by compactness of $K_1$, strong convergence $P_{V_\zeta}\mathbf\Phi_1\stackrel{\zeta\to \infty}{\longrightarrow} \mathbf\Phi_1$ and continuity of $\widetilde f_{\epsilon_1,\mathrm{loc}}$, the r.h.s. of \eqref{eq:chargefeps1z} becomes arbitrarily small as $\zeta\to\infty$. 

It follows from \eqref{eq:chargefeps1z0} that the problem of approximating $f$ reduces to approximating the map $\widetilde f_{\epsilon_1,\zeta}$ for a fixed finite $\zeta$.

\medskip
\noindent
\textbf{6. Polynomial approximation.} The map $\widetilde f_{\epsilon_1,\mathrm{loc},\zeta}:V_\zeta\to\mathbb R$ defined in \eqref{eq:chargeftildeepsloczeta} is a continuous $\mathrm{SO}(2)$-invariant map on the $\mathrm{SO}(2)$-module $V_\zeta$. By Lemma \ref{lm:chargeso2poly}, such a map can be approximated by invariant polynomials. Let $K_1\subset V$ be the compact set defined in Eq.\eqref{eq:chargek1}. Note that $P_{V_\zeta}K_1$ is then a compact subset of $V_\zeta$. Let $\widehat{f}_{\mathrm{loc}}:V_\zeta\to \mathbb R$ be an $\mathrm{SO}(2)$-invariant polynomial such that  for all $\mathbf\Phi_2\in P_{V_\zeta}K_1$
\begin{equation}\label{eq:chargehatfloc}|\widehat{f}_{\mathrm{loc}}(\mathbf\Phi_2)-\widetilde f_{\epsilon_1,\mathrm{loc},\zeta}(\mathbf\Phi_2)|\le \frac{\epsilon}{10\cdot 2\Lambda}.\end{equation}
Consider now the map $\widehat{f}_0:V\to U$ defined by  
\begin{equation}\label{eq:chargefhat0}
\widehat{f}_0(\mathbf\Phi)(\mathbf x)=\begin{cases} \widehat{f}_{\mathrm{loc}}(P_{V_\zeta}R_{(-\mathbf x,1)}\mathbf\Phi),& \mathbf x\in[-\Lambda,\Lambda]^2,\\
0,&\text{otherwise}. \end{cases} 
\end{equation}
Using Eqs.\eqref{eq:chargewtfeps1zeta} and \eqref{eq:chargehatfloc}, we have for all $\mathbf x\in[-\Lambda,\Lambda]^2$ and $\mathbf\Phi_2\in P_{V_\zeta}K_1$ 
$$|\widehat{f}_{0}(\mathbf\Phi_2)(\mathbf x)-\widetilde f_{\epsilon_1,\zeta}(\mathbf\Phi_2)(\mathbf x)|\le \frac{\epsilon}{10\cdot 2\Lambda}$$
and hence for all $\mathbf\Phi\in K$ \begin{equation}\label{eq:chargefhat0fteps1}\|\widehat{f}_{0}(\mathbf\Phi)-\widetilde f_{\epsilon_1,\zeta}(\mathbf\Phi)\|<\frac{\epsilon}{10}.\end{equation}

\medskip
\noindent
\textbf{7. Identification of convnet with $\lambda=0$.} We show now that the map $\widehat{f}_0$ given in \eqref{eq:chargefhat0} can be written as the scaling limit ($\lambda\to 0$) of a multi-resolution charge-conserving convnet. 

First note that the projector $P_{V_\zeta}$ can be written as
$$P_{V_\zeta}\mathbf\Phi=\sum_{a,b:a+b\le\zeta}\langle\Psi'_{a,b},\mathbf\Phi\rangle\Psi_{a,b},$$
where $\Psi'_{a,b}$ is the basis in $V_\zeta$ dual to the basis $\Psi_{a,b}.$ Let $V_{\zeta,\xi}$ denote the isotypic component in $V_{\zeta}$ spanned by vectors $\Psi_{a,b}$ with $a-b=\xi$. By Eq.\eqref{eq:charger0eiphipsiab}, this notation is consistent with the notation of Section \ref{sec:chargepolyinv} where the number $\xi$ is used to specify the characters \eqref{eq:charge1drep}. By unitarity of the representation $R$, different isotypic components are mutually orthogonal, so $\Psi'_{a,b}\in V_{\zeta,a-b}$ and we can expand
$$\Psi'_{a,b}=\sum_{\genfrac{}{}{0pt}{}{0\le a',b'\le \zeta}{a'-b'=a-b}}c_{a,b,a',b'}\Psi_{a',b'}$$ with some coefficients $c_{a,b,a',b'}$. Then we can write
\begin{align}\label{eq:chargepvzetar}
P_{V_\zeta}R_{(-\mathbf x,1)}\mathbf\Phi=& \sum_{a,b:a+b\le\zeta}\langle\Psi'_{a,b},R_{(-\mathbf x,1)}\mathbf\Phi\rangle\Psi_{a,b}\nonumber\\
=&\sum_{\xi=-\zeta}^{\zeta}\sum_{\genfrac{}{}{0pt}{}{a,b:a+b\le \zeta}{a-b=\xi}}\sum_{\genfrac{}{}{0pt}{}{a',b':a'+b'\le \zeta}{a'-b'=\xi}}\overline{c_{a,b,a',b'}}\langle\Psi_{a',b'},R_{(-\mathbf x,1)}\mathbf\Phi\rangle\Psi_{a,b}\nonumber\\
=&\sum_{\xi=-\zeta}^{\zeta}\sum_{\genfrac{}{}{0pt}{}{a,b:a+b\le \zeta}{a-b=\xi}}\sum_{\genfrac{}{}{0pt}{}{a',b':a'+b'\le \zeta}{a'-b'=\xi}}\overline{c_{a,b,a',b'}}\Big(\int_{\mathbb R^2}\mathbf\Phi(\mathbf x+\mathbf y)\overline{\Psi_{a',b'}(\mathbf y)}d^2\mathbf y\Big)\Psi_{a,b}\nonumber\\
=&\sum_{\xi=-\zeta}^{\zeta}\sum_{\genfrac{}{}{0pt}{}{a,b:a+b\le \zeta}{a-b=\xi}}\sum_{\genfrac{}{}{0pt}{}{a',b':a'+b'\le \zeta}{a'-b'=\xi}}\overline{c_{a,b,a',b'}}(-1)^{a'+b'}\Big(\int_{\mathbb R^2}\mathbf\Phi(\mathbf x-\mathbf y){\Psi_{b',a'}(\mathbf y)}d^2\mathbf y\Big)\Psi_{a,b}\nonumber\\
=& \sum_{\xi=-\zeta}^{\zeta}\sum_{\genfrac{}{}{0pt}{}{a,b:a+b\le \zeta}{a-b=\xi}}\sum_{\genfrac{}{}{0pt}{}{a',b':a'+b'\le \zeta}{a'-b'=\xi}}\overline{c_{a,b,a',b'}}(-1)^{a'+b'}(\mathcal L_{0}^{(b',a')}\mathbf\Phi(\mathbf x))\Psi_{a,b},
\end{align}
where in the last step we used definition \eqref{eq:chargel0ab} of $\mathcal L_{0}^{(a,b)}.$

We can now interpret the map $\widehat{f}_0$ given by \eqref{eq:chargefhat0} as the $\lambda\to 0$ limit of a convnet of Section \ref{sec:chargeconvnet} in the following way.

First, by the above expansion, the part $P_{V_\zeta}R_{(-\mathbf x,1)}$ of the map $\widehat{f}_0$ computes various convolutions $\mathcal L_{0}^{(b',a')}\mathbf\Phi$ with $a'+b'\le \zeta, a'-b'=\xi$ --- this corresponds to the $\lambda\to 0$ limit of smoothing and differentiation layers of Section \ref{sec:chargeconvnet} with $T_{\mathrm{diff}}=\zeta$. The global charge parameter $\mu$ appearing in the decomposition \eqref{eq:chargewdifftdecomp} of the target spaces $W_{\mathrm{diff},t}$ of differentiation layers corresponds to $-\xi(=b'-a')$ in the above formula, while the degree $s$ corresponds to $a'+b'$. The vectors $\Psi_{a,b}$ with $a-b=\xi$ over which we expand in \eqref{eq:chargepvzetar} serve as a particular basis in the $\mu=-\xi$ component of $W_{\mathrm{diff}, T_{\mathrm{diff}}}$. 

Now, the invariant polynomial $\widehat{f}_{\mathrm{loc}}$ appearing in \eqref{eq:chargefhat0} can be expressed as a polynomial in the variables associated with the isotypic components $V_{\zeta,\xi}$. These components are spanned by the vectors $\Psi_{a,b}$ with $a-b=\xi$. By Eq.\eqref{eq:chargepvzetar}, $\widehat{f}_{\mathrm{loc}}(P_{V_\zeta}R_{(-\mathbf x,1)}\mathbf\Phi)$ can then be viewed as an invariant polynomial in the variables $\mathcal L_{0}^{(b',a')}\mathbf\Phi(\mathbf x)$ that correspond to the isotypic components $V_{\zeta,\xi}$ with $\xi=a'-b'$.
 As shown in Section \ref{sec:chargepolyinv}, this invariant polynomial can then  be generated by the layerwise multiplication procedure \eqref{eq:chargeftmun} starting from the initial variables $\mathcal L_{0}^{(b',a')}\mathbf\Phi(\mathbf x)$. This procedure is reproduced in the definition \eqref{eq:chargemultweightsmu},\eqref{eq:chargemultweights0} of convnet multiplication layers. (The charge-conservation constraints are expressed in Eqs.\eqref{eq:chargemultweightsmu},\eqref{eq:chargemultweights0} in terms of $\mu$ rather than $\xi$, but $\mu=-\xi$, and the constraints are invariant with respect to changing the sign of all $\mu$'s.) Thus, if the number $T_{\mathrm{mult}}$ of multiplication layers and the dimensions $d_{\mathrm{mult}}$ of these layers are sufficiently large, then one can arrange the weights in these layers so as to exactly give the map $\mathbf\Phi\mapsto\widehat{f}_{\mathrm{loc}}(P_{V_\zeta}R_{(-\mathbf x,1)}\mathbf\Phi)$.

\medskip
\noindent
\textbf{8. Approximation by convnets with $\lambda>0$.} It remains to show that the scaling limit $\widehat{f}_{0}$ is approximated by the $\lambda>0$ convnets $\widehat{f}_{\lambda}$
in the sense that if $\lambda$ is sufficiently small then for all $\mathbf\Phi\in K$ \begin{equation}\label{eq:chargef0flam}\|\widehat{f}_{0}(\mathbf\Phi)-\widehat{f}_{\lambda}(\mathbf\Phi)\|<\frac{\epsilon}{10}.\end{equation}
We have already shown earlier in Eq.\eqref{eq:chargef0fll2} that for any $\mathbf\Phi\in V$ the signals $\widehat{f}_{\lambda}(\mathbf\Phi)$ converge to $\widehat{f}_{\lambda}(\mathbf\Phi)$ in the $L^2([-\Lambda,\Lambda]^2)$ sense. In fact, Lemma \ref{lm:chargeclt} implies that this convergence is uniform on any compact set  $K\subset V$, which proves Eq.\eqref{eq:chargef0flam}.

\medskip
\noindent
Summarizing all the above steps, we have constructed a multi-resolution charge-conserving convnet $\widehat{f}_{\lambda}$ such that, by the inequalities \eqref{eq:chargetfeps1fphi},\eqref{eq:chargepltfeps1},\eqref{eq:chargefeps1z0},\eqref{eq:chargefhat0fteps1} and \eqref{eq:chargef0flam}, we have $\sup_{\mathbf \Phi\in K}\|\widehat f_\lambda(\mathbf \Phi)-f(\mathbf \Phi)\|\le\epsilon$ for all sufficiently small $\lambda$. This completes the proof of the sufficiency part.
\end{proof}

\section{Discussion}\label{sec:discus}
We summarize and discuss the obtained results, and indicate potential directions of further research.

In Section \ref{sec:compact} we considered approximation of maps defined on finite-dimensional spaces and described universal and exactly invariant/equivariant extensions of the usual shallow neural network (Propositions \ref{th:invar}, \ref{th:equivar}). These extensions are obtained by adding to the network a special polynomial layer. This construction can be seen as an alternative to the symmetrization of the network (similarly to how constructing symmetric polynomials as functions of elementary symmetric polynomials is an alternative to symmetrizing non-symmetric polynomials). A drawback (inherited from the theory of invariant polynomials) of this construction  is that it requires us to know appropriate sets of generating polynomial invariants/equivariants, which is difficult in practice. This difficulty can be ameliorated using polarization if the modules in question are decomposed into multiple copies of a few basic modules  (Proposition \ref{th:constr}, \ref{th:constr_equiv}), but this approach still may be too complicated in general for practical applications. 

Nevertheless, in the case of the symmetric group $S_N$ we have derived an explicit complete $S_N$-invariant modification of the usual shallow neural network  (Theorem \ref{th:sn}). While complete and exactly $S_N$-invariant, this modification does not involve symmetrization over $S_N$. With its relatively small computational complexity, this modification thus presents a viable alternative to the symmetrization-based approach.

One can expect that further progress in the design of invariant/equivariant models may be achieved by using more advanced general constructions from the representation and invariant theories. In particular, in Section \ref{sec:compact} we have not considered \emph{products} of representations, but later in Section \ref{sec:charge} we essentially use them in the abelian $\mathrm{SO}(2)$ setting when defining multiplication layers in ``charge-conserving convnet'' . 

In Section \ref{sec:translations} we considered approximations of maps defined on the space $V=L^2(\mathbb R^\nu,\mathbb R^{d_V})$ of $d_V$-component signals on $\mathbb R^\nu$. The crucial feature of this setting is the infinite-dimensionality of the space $V$, which requires us to reconsider the notion of approximation. Inspired by classical finite-dimensional results \cite{pinkus1999approximation}, our approach in Section \ref{sec:translations} was to assume that a map $f$ is defined on the whole $L^2(\mathbb R^\nu,\mathbb R^{d_V})$ as a map $f:L^2(\mathbb R^\nu,\mathbb R^{d_V})\to L^2(\mathbb R^\nu,\mathbb R^{d_U})$ or $f:L^2(\mathbb R^\nu,\mathbb R^{d_V})\to \mathbb R$, and consider its approximation by finite models $\widehat f$ in a weak sense of comparison on compact subsets of $V$ (see Definitions \ref{def:basicconvnet} and \ref{def:convpool}). This approach has allowed us to prove reasonable universal approximation properties of standard convnets. Specifically, in Theorem \ref{th:convmain} we prove that a map $f:L^2(\mathbb R^\nu,\mathbb R^{d_V})\to L^2(\mathbb R^\nu,\mathbb R^{d_U})$ can be approximated by convnets without pooling if and only if $f$ is norm-continuous and $\mathbb R^\nu$-equivariant. In Theorem \ref{th:convpool} we prove that a map $f:L^2(\mathbb R^\nu,\mathbb R^{d_V})\to \mathbb R$ can be approximated by convnets with downsampling if and only if $f$ is norm-continuous.

In applications involving convnets (e.g., image recognition or segmentation), the approximated maps $f$ are considered only on small subsets of the full space $V$. Compact (or, more generally, precompact) subsets have properties that seem to make them a reasonable general abstraction for such subsets. In particular, a subset $K\subset V$ is precompact if, for example, it results from a continuous generative process involving finitely many bounded parameters; or if $K$ is a finite union of precompact subsets; or if for any $\epsilon>0$ the set $K$ can be covered by finitely many $\epsilon$-balls. From this perspective, it seems reasonable to consider restrictions of maps $f$ to compact sets, as we did in our weak notion of approximation in Section \ref{sec:translations}. At the same time, it would be interesting to refine the notion of model convergence by considering the structure of the sets $K$ in more detail and relate it quantitatively to the approximation accuracy (in partucular, paving the way to computing approximation rates). 

In Section  \ref{sec:charge} we consider the task of constructing finite universal approximators for maps $f:L^2(\mathbb R^2,\mathbb R^{d_V})\to L^2(\mathbb R^2,\mathbb R^{d_U})$ equivariant with respect to the group $\mathrm{SE}(2)$ of two-dimensional rigid planar motions. We introduce a particular convnet-like model -- ``charge-conserving convnet'' -- solving this task. We extend the topological framework of Section \ref{sec:translations} to rigorously formulate the properties of equivariance and completeness to be proved. Our main result, Theorem \ref{th:charge}, shows that a map $f:L^2(\mathbb R^2,\mathbb R^{d_V})\to L^2(\mathbb R^2,\mathbb R^{d_U})$ can be approximated in the small-scale limit by finite charge-conserving convnets if and only if $f$ is norm-continuous and $\mathrm{SE}(2)$-equivariant. 

The construction of this convnet is based on splitting the feature space into isotypic components characterized by a particular representation of the group $\mathrm{SO}(2)$ of proper 2D rotations. The information flow in the model is constrained by what can be interpreted as ``charge conservation'' (hence the name of the model). The model is essentially polynomial, only including elementary arithmetic operations ($+,-,*$) arranged so as to satisfy these constraints but otherwise achieve full expressivity. 

While in Sections \ref{sec:translations}, \ref{sec:charge} we have constructed intrinsically $\mathbb R^\nu$- and $\mathrm{SO}(2)$-equivariant and complete approximators for maps $f:L^2(\mathbb R^\nu,\mathbb R^{d_V})\to L^2(\mathbb R^2,\mathbb R^{d_U})$, we have not been able to similarly construct intrinsically $\mathbb R^\nu$-invariant approximators for maps $f:L^2(\mathbb R^\nu,\mathbb R^{d_V})\to \mathbb R$. As noted in Section \ref{sec:translations} and confirmed by Theorem \ref{th:convpool}, if we simply include pooling in the convnet, it completely destroys the $\mathbb R^\nu$-invariance in our continuum limit. It would be interesting to further explore this issue.

The convnets considered in Section \ref{sec:translations} have a rather conventional structure as sequences of linear convolutional layers equipped with a nonlinear activation function \citep{Goodfellow-et-al-2016}. In contrast, the charge-conserving convnets of Section \ref{sec:charge} have a special and somewhat artificial structure (three groups of layers of which the first two are linear and commuting; no arbitrary nonlinearities). This structure was essential for our proof of the main Theorem \ref{th:charge}, since these assumptions on the model allowed us to prove that the model is both $\mathrm{SE}(2)$-equivariant and complete. It would be interesting to extend this theorem to more general approximation models.

\bibliographystyle{plain}
\bibliography{main}

\appendix

\section{Proof of Lemma \ref{lm:chargeclt}}\label{sec:clt}

The proof is a slight modification of the standard proof of Central Limit Theorem via Fourier transform (the CLT can be directly used to prove the lemma in the case $a=b=0$ when $\mathcal L_{\lambda}^{(a,b)}$ only includes diffusion factors). 

To simplify notation, assume without loss of generality that $d_V=1$ (in the general case the proof is essentially identical). We will use the appropriately discretized version of the Fourier transform (i.e., the Fourier series expansion). Given a discretized signal $ \Phi:(\lambda\mathbb Z)^2\to \mathbb C$, we define $\mathcal F_\lambda \Phi$ as a function on $[-\frac{\pi}{\lambda},\frac{\pi}{\lambda}]^2$ by 
$$\mathcal F_\lambda \Phi(\mathbf p)=\frac{\lambda^2}{2\pi}\sum_{{\gamma}\in(\lambda\mathbb Z)^2} \Phi({\gamma})e^{-i\mathbf p\cdot {\gamma}}.$$
Then, $\mathcal F_\lambda: L^2((\lambda\mathbb Z)^2,\mathbb C)\to L^2([-\frac{\pi}{\lambda},\frac{\pi}{\lambda}]^2,\mathbb C)$ is a unitary isomorphism, assuming that the scalar product in the input space is defined by $\langle \Phi,\Psi\rangle=\lambda^2\sum_{\gamma\in(\lambda\mathbb Z)^2}\overline{\Phi(\gamma)}\Psi(\gamma)$ and in the output space by $\langle \Phi,\Psi\rangle=\int_{[-\frac{\pi}{\lambda},\frac{\pi}{\lambda}]^2} \overline{\Phi(\mathbf p)}\Psi(\mathbf p) d^2\mathbf p$. Let $P_\lambda$ be the discretization projector \eqref{eq:ptau}. It is easy to check that $\mathcal F_\lambda P_\lambda$ strongly converges to the standard Fourier transform as $\lambda\to 0:$
$$\lim_{\lambda\to 0}\mathcal F_\lambda P_\lambda \Phi = \mathcal F_0 \Phi, \quad \Phi\in L^2(\mathbb R^2,\mathbb C),$$
where $$\mathcal F_0 \Phi(\mathbf p)=\frac{1}{2\pi}\int_{\mathbb R^2} \Phi({\gamma})e^{-i\mathbf p\cdot {\gamma}}d^2{\gamma}$$
and where we naturally embed $L^2([-\frac{\pi}{\lambda},\frac{\pi}{\lambda}]^2,\mathbb C)\subset L^2(\mathbb R^2,\mathbb C)$. Conversely, let $P_\lambda'$ denote the orthogonal projection onto the subspace $L^2([-\frac{\pi}{\lambda},\frac{\pi}{\lambda}]^2,\mathbb C)$ in $L^2(\mathbb R^2,\mathbb C):$
\begin{equation}\label{eq:chargep'}
P_\lambda':\Phi\mapsto \Phi|_{[-\frac{\pi}{\lambda},\frac{\pi}{\lambda}]^2}.
\end{equation} Then 
\begin{equation}\label{eq:chargefinvlim}
\lim_{\lambda\to 0}\mathcal F_\lambda^{-1} P'_\lambda \Phi = \mathcal F_0^{-1} \Phi, \quad \Phi\in L^2(\mathbb R^2).
\end{equation}
Fourier transform gives us the spectral representation of the discrete differential operators \eqref{eq:dzlam},\eqref{eq:dzovlam},\eqref{eq:laplacelam}  as operators of multiplication by function: 
\begin{align*}
\mathcal F_\lambda \partial_{z}^{(\lambda)} \Phi &=\Psi_{\partial_{z}^{(\lambda)}}\cdot\mathcal F_\lambda \Phi,\\
\mathcal F_\lambda \partial_{\overline{z}}^{(\lambda)} \Phi &=\Psi_{\partial_{\overline{z}}^{(\lambda)}}\cdot\mathcal F_\lambda  \Phi,\\
\mathcal F_\lambda \Delta^{(\lambda)} \Phi &=\Psi_{\Delta^{(\lambda)}}\cdot\mathcal F_\lambda  \Phi,
\end{align*}
where, denoting $\mathbf p=(p_x,p_y)$,
\begin{align*}\Psi_{\partial_{z}^{(\lambda)}}(p_x,p_y)&=\frac{i}{2\lambda}(\sin\lambda p_x-i\sin\lambda p_y),\\
\Psi_{\partial_{\overline{z}}^{(\lambda)}}(p_x,p_y)&=\frac{i}{2\lambda}(\sin\lambda p_x+i\sin\lambda p_y),\\
\Psi_{\Delta^{(\lambda)}}(p_x,p_y) &=-\frac{4}{\lambda^2}\Big(\sin^2\frac{\lambda p_x}{2}+\sin^2\frac{\lambda p_y}{2}\Big).
\end{align*}
The operator $\mathcal L_\lambda^{(a,b)}$ defined in \eqref{eq:chargedefltau} can then be written as  
\begin{equation*}
\mathcal F_\lambda \mathcal L_\lambda^{(a,b)} \Phi =\Psi_{\mathcal L_\lambda^{(a,b)}}\cdot\mathcal F_\lambda P_\lambda\Phi,
\end{equation*}
where the function $\Psi_{\mathcal L_\lambda^{(a,b)}}$ is given by
\begin{equation*}\Psi_{\mathcal L_\lambda^{(a,b)}}=(\Psi_{\partial_{z}^{(\lambda)}})^a (\Psi_{\partial_{\overline{z}}^{(\lambda)}})^b(1+\tfrac{\lambda^2}{8}\Psi_{\Delta^{(\lambda)}})^{\lceil 4/\lambda^2\rceil}.
\end{equation*}
We can then write $\mathcal L_\lambda^{(a,b)}\Phi$ as a convolution of $P_\lambda \Phi$ with the kernel $$\Psi_{a,b}^{(\lambda)}=\frac{1}{2\pi}\mathcal F_\lambda^{-1}\Psi_{\mathcal L_\lambda^{(a,b)}}$$ on the grid $(\lambda\mathbb Z)^2:$ 
\begin{equation}\label{eq:chargellabconv}\mathcal L_\lambda^{(a,b)}\Phi(\gamma)=\lambda^2\sum_{\theta\in(\lambda\mathbb Z)^2}P_\lambda\Phi(\gamma-\theta)\Psi_{a,b}^{(\lambda)}(\theta),\quad \gamma\in (\lambda\mathbb Z)^2.
\end{equation}


Now consider the operator $\mathcal L_0^{(a,b)}$ defined in \eqref{eq:chargel0ab}. At each $\mathbf x\in \mathbb R^2$, the value $\mathcal L_0^{(a,b)} \Phi(\mathbf x)$ can be written as a scalar product:
\begin{equation}\label{eq:cltl0ab}
\mathcal L_0^{(a,b)} \Phi(\mathbf x)=\int_{\mathbb R^2}\Phi(\mathbf x-\mathbf y)\Psi_{a,b}(\mathbf y)d^2\mathbf y=\langle R_{-\mathbf x}\widetilde{\Phi},\Psi_{a,b}\rangle_{L^2(\mathbb R^2)},\end{equation}
where  $\widetilde{\Phi}(\mathbf x)=\overline{\Phi}(-\mathbf x)$, $\Psi_{a,b}$ is defined by \eqref{eq:chargepsiab}, and $R_{\mathbf x}$ is our standard representation of the group $\mathbb R^2$, $R_{\mathbf x}\Phi(\mathbf y)=\Phi(\mathbf y-\mathbf x)$. For $\lambda>0$, we can write $\mathcal L_\lambda^{(a,b)}  \Phi(\mathbf x)$ in a similar form. Indeed, using \eqref{eq:chargellabconv} and naturally extending the discretized signal $\Psi_{a,b}^{(\lambda)}$ to the whole $\mathbb R^2$, we have
\begin{equation*}
\mathcal L_\lambda^{(a,b)}\Phi(\gamma)=\int_{\mathbb R^2}\Phi(\gamma-\mathbf y)\Psi_{a,b}^{(\lambda)}(\mathbf y)d^2\mathbf y=\langle R_{-\gamma}\widetilde{\Phi},\Psi_{a,b}^{(\lambda)}\rangle_{L^2(\mathbb R^2)}.
\end{equation*}
Then, for any $\mathbf x\in \mathbb R^2$ we can write 
\begin{equation}\label{eq:chargellabx}\mathcal L_\lambda^{(a,b)}\Phi(\mathbf x)=\langle R_{-\mathbf x+\delta \mathbf x}\widetilde{\Phi},\Psi_{a,b}^{(\lambda)}\rangle_{L^2(\mathbb R^2)},\end{equation}
where $-\mathbf x+\delta\mathbf x$ is the point of the grid $(\lambda\mathbb Z)^2$ nearest to $-\mathbf x$.

Now consider the formulas \eqref{eq:cltl0ab},\eqref{eq:chargellabx} and observe that, by Cauchy-Schwarz inequality and since $R$ is norm-preserving, to prove statement 1) of the lemma we only need to show that the functions $\Psi_{a,b},\Psi_{a,b}^{(\lambda)}$ have uniformly bounded $L^2$-norms. For $\lambda>0$ we have
\begin{align}\|\Psi_{a,b}^{(\lambda)}\|^2_{L^2(\mathbb R^2)}&=\Big\|\frac{1}{2\pi}\mathcal F_\lambda^{-1}\Psi_{\mathcal L_\lambda^{(a,b)}}\Big\|^2_{L^2(\mathbb R^2)}\nonumber\\
&=\frac{1}{4\pi^2}\|\Psi_{\mathcal L_\lambda^{(a,b)}}\|^2_{L^2(\mathbb R^2)}\nonumber\\
&=\frac{1}{4\pi^2}\big\|(\Psi_{\partial_{z}^{(\lambda)}})^a (\Psi_{\partial_{\overline{z}}^{(\lambda)}})^b(1+\tfrac{\lambda^2}{8}\Psi_{\Delta^{(\lambda)}})^{\lceil 4/\lambda^2\rceil}\big\|_{L^2(\mathbb R^2)}^2\nonumber\\
&\le  \frac{1}{4\pi^2}\int_{-\pi/\lambda}^{\pi/\lambda}\int_{-\pi/\lambda}^{\pi/\lambda}
\big(\tfrac{|p_x|+|p_y|}{2}\big)^{2(a+b)}\exp\big(-\lceil 4/\lambda^2\rceil(\sin^2\tfrac{\lambda p_x}{2}+\sin^2\tfrac{\lambda p_y}{2})\big)
dp_xdp_y\nonumber\\
&\le \frac{1}{4\pi^2}\int_{-\infty}^{\infty}\int_{-\infty}^{\infty}
\big(\tfrac{|p_x|+|p_y|}{2}\big)^{2(a+b)}\exp(-\tfrac{4}{\pi^2}(p_x^2+p_y^2))dp_xdp_y\label{eq:chargepsiabbound}\\
&<\infty,\nonumber
\end{align}
where we used the inequalities
\begin{align*}&|\sin t| \le |t|,\\
&|1+t|\le e^t,\quad t>-1,\\
&|\sin t|\ge \tfrac{2|t|}{\pi},\quad t\in [-\tfrac{\pi}{2},\tfrac{\pi}{2}].
\end{align*}
Expression \eqref{eq:chargepsiabbound} provides a finite bound, uniform in $\lambda$, for the squared norms $\|\Psi_{a,b}^{(\lambda)}\|^2$. This bound also holds for $\|\Psi_{a,b}\|^2$.

Next, observe that to establish the strong convergence in statement 2) of the lemma, it suffices to show that 
\begin{equation}\label{eq:chargepsitpsi0}
\lim_{\lambda\to 0}\|\Psi_{a,b}^{(\lambda)}-\Psi_{a,b}\|_{L^2(\mathbb R^2)}=0.\end{equation}
Indeed, by \eqref{eq:cltl0ab},\eqref{eq:chargellabx}, we would then have 
\begin{align*}\|\mathcal L_\lambda^{(a,b)} \Phi-\mathcal L_0^{(a,b)} \Phi \|_\infty
&=\sup_{\mathbf x\in\mathbb R^2}|\langle R_{-\mathbf x+\delta\mathbf x}\widetilde{ \Phi},\Psi_{a,b}^{(\lambda)}\rangle-\langle R_{-\mathbf x}\widetilde{ \Phi},\Psi_{a,b}\rangle|\\
&=\sup_{\mathbf x\in\mathbb R^2}|\langle R_{-\mathbf x}(R_{\delta\mathbf x}-1)\widetilde{ \Phi},\Psi_{a,b}^{(\lambda)}\rangle+\langle R_{-\mathbf x}\widetilde{ \Phi},\Psi_{a,b}^{(\lambda)}-\Psi_{a,b}\rangle|\\
&\le \sup_{\|\delta\mathbf x\|\le\lambda }\|R_{\delta\mathbf x}\widetilde{ \Phi}-\widetilde{ \Phi}\|_2\sup_{\lambda}\|\Psi_{a,b}^{(\lambda)}\|_2+\|\widetilde{ \Phi}\|_2\|\Psi_{a,b}^{(\lambda)}-\Psi_{a,b}\|_2\\
&\stackrel{\lambda\to 0}{\longrightarrow}0
\end{align*}
thanks to the unitarity of $R$, convergence $\lim_{\delta\mathbf x\to 0}\|R_{\delta\mathbf x}\widetilde{ \Phi}-\widetilde{ \Phi}\|_2=0,$ uniform boundedness of $\|\Psi_{a,b}^{(\lambda)}\|_2$ and convergence \eqref{eq:chargepsitpsi0}.

To establish \eqref{eq:chargepsitpsi0}, we write $$\Psi_{a,b}^{(\lambda)}-\Psi_{a,b}=\frac{1}{2\pi}(\mathcal F_\lambda^{-1}\Psi_{\mathcal L_\lambda^{(a,b)}}-\mathcal F_0^{-1}\Psi_{\mathcal L_0^{(a,b)}}),$$
where $\Psi_{\mathcal L_0^{(a,b)}}=2\pi \mathcal F_\lambda \Psi_{a,b}.$ By definition \eqref{eq:chargepsiab} of $\Psi_{a,b}$ and standard properties of Fourier transform, the explicit form of the function $\Psi_{\mathcal L_0^{(a,b)}}$ is
$$\Psi_{\mathcal L_0^{(a,b)}}(p_x,p_y)=\big(\tfrac{i(p_x-ip_y)}{2}\big)^a\big(\tfrac{i(p_x+ip_y)}{2}\big)^b \exp\big(-\tfrac{p_x^2+p_y^2}{2}\big).$$
  
Observe that the function $\Psi_{\mathcal L_0^{(a,b)}}$ is the pointwise limit of the functions $\Psi_{\mathcal L_\lambda^{(a,b)}}$ as $\lambda\to 0$. The functions $|\Psi_{\mathcal L_\lambda^{(a,b)}}|^2$ are bounded uniformly in $\lambda$ by the integrable function appearing in the integral \eqref{eq:chargepsiabbound}. Therefore we can use the dominated convergence theorem and conclude that 
\begin{equation}\label{eq:chargelimpsilt}
\lim_{\lambda\to 0}\big\|\Psi_{\mathcal L_\lambda^{(a,b)}}-P'_\lambda\Psi_{\mathcal L_0^{(a,b)}}\big\|_2=0,
\end{equation}
where $P_\lambda'$ is the cut-off projector \eqref{eq:chargep'}.  
We then have 
\begin{align*}
\|\Psi_{a,b}^{(\lambda)}-\Psi_{a,b}\|_2={}&\frac{1}{2\pi}
\big\|\mathcal F_\lambda^{-1}\Psi_{\mathcal L_\lambda^{(a,b)}}-\mathcal F_0^{-1}\Psi_{\mathcal L_0^{(a,b)}}\big\|_2\\
&\le \frac{1}{2\pi}\big\|\mathcal F_\lambda^{-1}(\Psi_{\mathcal L_\lambda^{(a,b)}}-P'_\lambda\Psi_{\mathcal L_0^{(a,b)}})\big\|_2+ 
\frac{1}{2\pi}\big\|(\mathcal F_\lambda^{-1}P'_\lambda-\mathcal F_0^{-1})\Psi_{\mathcal L_0^{(a,b)}}\big\|_2\\
&\stackrel{\lambda\to 0}{\longrightarrow}0
\end{align*}
by \eqref{eq:chargelimpsilt} and \eqref{eq:chargefinvlim}. We have thus proved \eqref{eq:chargepsitpsi0}.

It remains to show that the convergence $\mathcal L_\lambda^{(a,b)} \Phi\to \mathcal L_0^{(a,b)} \Phi$ is uniform on compact sets $K\subset V$. This follows by a version of continuity argument. For any $\epsilon>0$, we can choose finitely many $\Phi_n,n=1,\ldots,N,$ such that for any $\Phi\in K$ there is some $\Phi_n$ for which $\|\Phi-\Phi_n\|<\epsilon.$ Then $\|\mathcal L_\lambda^{(a,b)} \Phi- \mathcal L_0^{(a,b)} \Phi\|\le \|\mathcal L_\lambda^{(a,b)} \Phi_n- \mathcal L_0^{(a,b)} \Phi_n\|+2\sup_{\lambda\ge 0} \|\mathcal L_\lambda^{(a,b)}\|\epsilon$. Since $\sup_{\lambda\ge 0} \|\mathcal L_\lambda^{(a,b)}\|<\infty$ by statement 1) of the lemma, the desired uniform convergence for  $\Phi\in K$ follows from the convergence for $\Phi_n,n=1,\ldots,N$. 



\end{document}